\pdfoutput=1

\documentclass[twoside,11pt]{article}

\usepackage{savesym}
\usepackage[ruled,vlined,linesnumbered,noresetcount]{algorithm2e}
\usepackage{amsmath}
\usepackage{amssymb}
\usepackage{amsthm}
\usepackage{array}
\usepackage{bm}
\usepackage{booktabs}       
\usepackage{etoc}
\usepackage{graphicx}
\usepackage[utf8]{inputenc} 
\usepackage{multirow}
\usepackage{mathabx}
\usepackage{microtype}      
\usepackage{nicefrac}       
\usepackage{url}            
\usepackage{xcolor}
\usepackage{bbm}
\usepackage{graphicx}
\usepackage{enumitem}
\usepackage{subcaption}
\usepackage[normalem]{ulem}

\usepackage{jmlr2e}
\usepackage{cleveref}

\usepackage{lastpage}

\makeatletter
\newcommand{\distas}[1]{\mathbin{\overset{#1}{\kern\z@\sim}}}%
\newsavebox{\mybox}\newsavebox{\mysim}
\newcommand{\distras}[1]{%
	\savebox{\mybox}{\hbox{\kern3pt$\scriptstyle#1$\kern3pt}}%
	\savebox{\mysim}{\hbox{$\sim$}}%
	\mathbin{\overset{#1}{\kern\z@\resizebox{\wd\mybox}{\ht\mysim}{$\sim$}}}%
}
\makeatother

\savesymbol{corollary}
\savesymbol{definition}
\savesymbol{proposition}
\savesymbol{assumption}
\savesymbol{lemma}
\savesymbol{remark}

\newtheorem*{theorem*}{Theorem}

\newtheorem*{corollary*}{Corollary}
\newtheorem{definition}{Definition}
\newtheorem*{definition*}{Definition}
\newtheorem{proposition}{Proposition}
\newtheorem*{proposition*}{Proposition}
\newtheorem{assumption}{Assumption}
\newtheorem*{assumption*}{Assumption}
\newtheorem{lemma}{Lemma}
\newtheorem*{lemma*}{Lemma}
\newtheorem{remark}{Remark}
\newtheorem*{remark*}{Remark}
\newtheorem*{example*}{Example}

\jmlrheading{22}{2021}{1-\pageref{LastPage}}{11/20; Revised 5/21}{6/21}{20-1300}{Takuo Matsubara, Chris J. Oates and Fran\c{c}ois-Xavier Briol}
\ShortHeadings{The Ridgelet Prior}{Matsubara, Oates, and Briol}

\firstpageno{1}

\title{The Ridgelet Prior: A Covariance Function Approach to Prior Specification for Bayesian Neural Networks}

\author{\name Takuo Matsubara \email tmatsubara@turing.ac.uk \\
	\addr School of Mathematics, Statistics \& Physics\\
	Newcastle University\\
	Newcastle upon Tyne, NE1 7RU, United Kingdom
	\AND
	\name Chris J.\ Oates \email chris.oates@ncl.ac.uk \\
	\addr School of Mathematics, Statistics \& Physics\\
	Newcastle University\\
	Newcastle upon Tyne, NE1 7RU, United Kingdom
	\AND
	\name Fran\c{c}ois-Xavier Briol \email f.briol@ucl.ac.uk \\
	\addr Department of Statistical Science\\
	University College London\\
	London, WC1E 6BT, United Kingdom
}

\editor{Pierre Alquier}

\begin{document}

\maketitle

\begin{abstract}
Bayesian neural networks attempt to combine the strong predictive performance of neural networks with formal quantification of uncertainty associated with the predictive output in the Bayesian framework. 
However, it remains unclear how to endow the parameters of the network with a prior distribution that is meaningful when lifted into the output space of the network. 
A possible solution is proposed that enables the user to posit an appropriate Gaussian process covariance function for the task at hand.
Our approach constructs a prior distribution for the parameters of the network, called a \textit{ridgelet prior}, that approximates the posited Gaussian process in the output space of the network. 
In contrast to existing work on the connection between neural networks and Gaussian processes, our analysis is non-asymptotic, with finite sample-size error bounds provided. 
This establishes the universality property that a Bayesian neural network can approximate any Gaussian process whose covariance function is sufficiently regular.
Our experimental assessment is limited to a proof-of-concept, where we demonstrate that the ridgelet prior can out-perform an unstructured prior on regression problems for which a suitable Gaussian process prior can be provided.
\end{abstract}

\begin{keywords}
Bayesian neural networks, Gaussian processes, prior selection, ridgelet transform, statistical learning theory
\end{keywords}


\section{Introduction}
\label{sec:introduction}

Neural networks are beginning to be adopted in a range of sensitive application areas such as healthcare \citep{Topol2019}, social care \citep{Serrano2019}, and the justice system \citep{Tortora2020}, where the accuracy and reliability of their predictive output demands careful assessment. 
This problem lends itself naturally to the Bayesian paradigm and there has been a resurgence in interest in Bayesian neural networks (BNNs), originally introduced and studied in \citep{Buntine1991, Mackay1992, Neal1995}. 
BNNs use the language of probability to express uncertainty regarding the ``true'' value of the parameters in the neural network, initially by assigning a \textit{prior} distribution over the space of possible parameter configurations and then updating this distribution on the basis of a training dataset. 
The resulting \textit{posterior} distribution over the parameter space implies an associated predictive distribution for the output of the neural network, assigning probabilities to each of the possible values that could be taken by the output of the network.
This predictive distribution carries the formal semantics of the Bayesian framework and can be used to describe epistemic uncertainty associated with the phenomena being modelled. 

Attached to any probabilistic quantification of uncertainty are \textit{semantics}, which describe how probabilities should be interpreted (e.g. are these probabilities epistemic or aleatoric; whose belief is being quantified; what assumptions are premised?).
As for any Bayesian model, the semantics of the posterior predictive distribution are largely inherited from the semantics of the prior distribution, which is typically a representation of a user's subjective belief about the unknown ``true'' values of parameters in the model. 
This represents a challenge for BNNs, as a user cannot easily specify their prior belief at the level of the parameters of the network in general settings where the influence of each parameter on the network's output can be difficult to understand.  
Furthermore, the total number of parameters can go from a few dozens to several million or more, rendering careful selection of priors for each parameter impractical. 
This has lead some researchers to propose \textit{ad hoc} choices for the prior distribution, which will be reviewed in \Cref{sec:background} \citep[see also][]{Nalisnick2018}. 
Such \textit{ad hoc} choices of prior appear to severely limit interpretability of the semantics of the BNN. 
It has also been reported that such priors can have negative consequences for the predictive performance of BNNs \citep{Yao2019}.

The development of interpretable prior distributions for BNNs is an active area of research that, if adequately solved, has the potential to substantially advance methodology for neural networks. Potential benefits include:
\begin{itemize}
	\item \textbf{Fewer Data Required:} BNNs are ``data hungry'' models; their large number of parameters means that a large number of data are required for the posterior to concentrate on a suitable configuration of parameter values.
	The inclusion of domain knowledge in the prior distribution could be helpful in reducing the effective degrees of freedom in the parameter space, mitigating the requirement for a large training dataset.
	\item \textbf{Faster Computation:} The use of an \textit{ad hoc} prior distribution can lead to a posterior distribution that is highly multi-modal \citep{Pourzanjani2017}, creating challenges for computation (e.g. using variational inference or Markov chain Monte Carlo; MCMC). 
	The inclusion of domain knowledge could be expected to counteract (to some extent) the multi-modality issue by breaking some of the symmetries present in the parametrisation of the network.
	\item \textbf{Lower Generalisation Error:} An important issue with BNNs is that their out-of-sample performance can be poor when an \textit{ad hoc} prior is used. 
	These issues have lead several authors to question the usefulness of the BNNs; see \cite{Mitros2019} and \cite{Wenzel2020}.
	Model predictions are strongly driven by the prior and we therefore expect inclusion of domain knowledge to be an important factor in improving the generalisation performance of BNNs.
\end{itemize}
In this paper we do \textit{not} claim to provide a solution to the problem of prior selection that enjoys all the benefits just discussed.
Such an endeavour would require very extensive (and application-specific) empirical investigation, which is not our focus in this work.
Rather, this paper proposes and studies a novel approach to prior specification that operates at the level of the output of the neural network and, in doing so, provides a route for expert knowledge on the phenomenon being modelled to be probabilistically encoded.
The construction that we present is stylised to admit a detailed theoretical analysis and therefore the empirical results in this paper are limited to a proof-of-concept.
In subsequent work we will discuss generalisations of the construction that may be more amenable to practical applications, for example by reducing the number of hidden units that may be required.

Our analysis can be viewed in the context of a recent line of research which focuses on the predictive distribution as a function of the prior distribution on network parameters \citep{Flam-Shepherd2017,Hafner2019,Pearce2019,Sun2019}. 
These papers propose to reduce the problem of prior selection for BNNs to the somewhat easier problem of prior selection for Gaussian processes (GPs). 
The approach studied in these papers, and also adopted in the present paper, can be summarised as follows: (i) Elicit a GP model that encodes domain knowledge for the problem at hand, (ii) Select a prior for the parameters of the BNN such that the output of the BNN in some sense ``closely approximates'' the GP. 
This high-level approach is appealing since it provides a direct connection between the established literature on covariance modelling for GPs \citep{Duvenaud2014b,Rasmussen2006,Stein1999} and the literature on uncertainty quantification using a BNN. 
For instance, existing covariance models can be used to encode \textit{a priori} assumptions of amplitude, smoothness, periodicity and so on as required. 
Moreover, the number of parameters required to elicit a GP (i.e. the parameters of the mean and covariance functions) is typically much smaller than the number of parameters in a BNN.

Existing work on this topic falls into two categories. 
In the first, the prior is selected in order to minimise a variational objective between the BNN and the target GP \citep{Flam-Shepherd2017,Hafner2019,Sun2019}. 
Although some of the more recent approaches have demonstrated promising empirical results, all lack theoretical guarantees. 
In addition, these approaches often constrain the user to use a particular algorithm for posterior approximation (such as variational inference), or require the need to see some of the training data in order to construct the prior model. 
The second approach consists of carefully adapting the architecture of the BNN to ensure convergence to the GP via a central limit theorem argument \citep{Pearce2019}. 
This approach is particularly efficient, but requires deriving a new BNN architecture for every GP covariance function and in this sense may be considered impractical.

In this paper we propose the \emph{ridgelet prior}, a novel method to construct interpretable prior distributions for BNNs. 
It follows the previous two-stage approach, but remedies several issues with existing approaches:
\begin{itemize}
	\item \textbf{Universal Approximation:} The ridgelet prior can be used for a BNN to approximate \emph{any} GP of interest (provided generic regularity conditions are satisfied) in the prior predictive without the need to modify the architecture of the network. 
	\item \textbf{Finite Sample-Size Error Bounds:} Approximation error bounds are obtained which are valid for a finite-dimensional parametrisation of the network, as opposed to relying on asymptotic results such as a central limit theorem. 
	\item \textbf{Compatibility:} The prior can be used within most existing algorithms for posterior approximation, such as variational inference or MCMC. 
	\item \textbf{No Optimisation Required:} The ridgelet prior does not require access to any part of the dataset in its construction and is straight-forward to implement (e.g. it does not require any numerical optimisation routine).
\end{itemize}

\noindent To construct the ridgelet prior, we build on existing analysis of the \textit{ridgelet transform} \citep{Candes1998, Murata1996b, Sonoda2017b}, which was used to study the consistency of (non-Bayesian) neural networks. 
In particular, we derive a novel result for numerical approximation using a finite-bandwidth version of the ridgelet transform, presented in Theorem \ref{thm: recon_rate}, that may be of independent interest.
The ridgelet prior is defined for neural networks with $L \geq 1$ hidden layers but our theoretical analysis focuses on the ``shallow'' case of $L = 1$ hidden layer, which is nevertheless a sufficiently rich setting for our consistency results to hold.

The remainder of the paper is structured as follows: 
\Cref{sec:background} reviews common prior choices for BNNs and known connections between BNNs and GPs. 
\Cref{sec:methodology} presents the ridgelet prior in full detail. 
Our theoretical guarantees for the ridgelet prior are outlined in \Cref{sec:theory}.
A proof-of-concept empirical assessment is contained \Cref{sec:experiments}.
\Cref{sec:conclusion} summarises the main conclusions of this work.
Code to reproduce all results in this paper can be downloaded from: \url{https://github.com/takuomatsubara/BNN-RidgeletPrior}.


\section{Background}
\label{sec:background}

To begin we briefly introduce notation for GPs and BNNs, discussing the issue of prior specification for these models.

\subsection{Prior Specification and Covariance Functions}

This paper focusses on the problem of approximating a deterministic function $\mathrm{f}:\mathbb{R}^d \rightarrow \mathbb{R}$ using a BNN. 
This problem is fundamental and underpins algorithms for regression and classification.
The Bayesian approach is to model $\mathrm{f}$ as a stochastic process (also called ``random function'') $f : \mathbb{R}^d \times \Theta \rightarrow \mathbb{R}$, where $\Theta$ is a measurable parameter space on which a \emph{prior} probability distribution is elicited, denoted $\mathbb{P}$.
The set $\Theta$ may be either finite or infinite-dimensional.
In either case, $\theta \mapsto f(\cdot,\theta)$ is a random variable taking values in the vector space of real-valued functions on $\mathbb{R}^d$. 
The combination of a dataset of size $n$ and Bayes' rule are used to constrain, in a statistical sense, this distribution on $\Theta$, to produce a posterior $\mathbb{P}_n$ that is absolutely continuous with respect to $\mathbb{P}$.
If the model is well-specified, then there exists an element $\theta^\dagger \in \Theta$ such that $f(\cdot,\theta^\dagger) = \mathrm{f}(\cdot)$ and, if the Bayesian procedure is consistent, $\mathbb{P}_n$ will converge (in an appropriate sense) to a point mass on $\theta^\dagger$ in the $n \rightarrow \infty$ limit.

In practice, Bayesian inference requires that a suitable prior distribution $\mathbb{P}$ is elicited. Stochastic processes are intuitively described by their moments, and these can be used by a domain expert to elicit $\mathbb{P}$.
The first two moments are given by the mean function $m: \mathbb{R}^d \rightarrow \mathbb{R}$ and the covariance function $k: \mathbb{R}^d \times \mathbb{R}^d \rightarrow \mathbb{R}$, given pointwise by
\begin{equation*}
m(\bm{x}) := \int_\Theta f(\bm{x},\theta) \mathrm{d}\mathbb{P}(\theta), \qquad
k(\bm{x},\bm{x}') := \int_\Theta (f(\bm{x},\theta)-m(\bm{x}))(f(\bm{x}',\theta)-m(\bm{x}')) \mathrm{d}\mathbb{P}(\theta) .
\end{equation*} 
for all $\bm{x},\bm{x}' \in \mathbb{R}^d$. GPs and BNNs are examples of stochastic processes that can be used.
In the case of a GP, the first two moments completely characterise $\mathbb{P}$.
Indeed, under the conditions of Mercer's theorem \citep[see e.g. Section 4.5 of][]{Steinwart2008},
\begin{equation*}
f(\bm{x},\theta) = m(\bm{x}) + \sum_{i=1}^{\text{dim}(\Theta)} \theta_i \varphi_i(\bm{x}), \qquad \theta_i \stackrel{\text{i.i.d.}}{\sim} \mathcal{N}(0,1)
\end{equation*}
where the functions $\varphi_i:\mathbb{R}^d \rightarrow \mathbb{R}$ are obtained from the Mercer decomposition of the covariance function $k(\bm{x},\bm{x}') = \sum_i \varphi_i(\bm{x}) \varphi_i(\bm{x}')$ and $\text{dim}(\Theta)$ denotes the dimension of $\Theta$.
The shorthand notation $\mathcal{GP}(m,k)$ is often used to denote this GP.
The \emph{kernel trick} enables explicit computation with the $\theta_i$ and $\varphi_i$ to be avoided, so that the user can specify the mean and covariance functions and, in doing so, $\mathbb{P}$ is implicitly defined.
There is a well-established literature on covariance modelling \citep{Duvenaud2014b,Rasmussen2006,Stein1999} for GPs.
For BNNs, however, there is no analogue of the kernel trick and it is unclear how to construct a prior $\mathbb{P}$ for the parameters $\theta$ of the BNN that is in agreement with moments that have been expert-elicited.

Fix a function $\phi: \mathbb{R} \to \mathbb{R}$, which we will call \emph{activation function}.
In this paper a BNN with $L \geq 1$ \emph{hidden layers} is understood to be a stochastic process with functional form
\begin{eqnarray}\label{eq:twolayer_neural_net}	
f(\bm{x},\theta) = \sum_{j=1}^{N_L} w_{1,j}^L \phi(z_j^L(\bm{x})), \quad z_i^{l}(\bm{x}) := b_{i}^{l-1} + \sum_{j=1}^{N_{l-1}} w_{i,j}^{l-1} \phi(z_j^{l-1}(\bm{x})), \quad l = 2,\dots,L
\end{eqnarray}
where $N_l := \text{dim}(\bm{z}^l(\bm{x}))$ is the number of nodes in the $l$th layer and the edge case is the \emph{input layer} 
$$
z_i^{1}(\bm{x}) := b_{i}^0 + \sum_{j=1}^d w_{i,j}^0 x_j,
$$
The parameters $\theta$ of the BNN consists of the \emph{weights} $w_{i,j}^l \in \mathbb{R}$ and the \emph{biases} $b_{i}^l \in \mathbb{R}$ of each layer $l = 0,\dots,L$, where the $L$th layer's bias is excluded in our definition. 
Common examples of activation functions include the rectified linear unit (ReLU) $\phi(x)=\max(0,x)$, logistic $\phi(x)=1/(1+\exp(-x))$, hyperbolic tangent $\phi(x)=\tanh(x)$ and the Gaussian $\phi(x)=\exp(-x^2)$. 
In all cases the complexity of the mapping $\theta \mapsto f(\cdot,\theta)$ in \eqref{eq:twolayer_neural_net} makes prior specification challenging, since it is difficult to ensure that a distribution on the parameters $\theta$ will be meaningful when lifted to the output space of the neural network.

\begin{figure}[t!] 
	\includegraphics[width=0.325\textwidth]{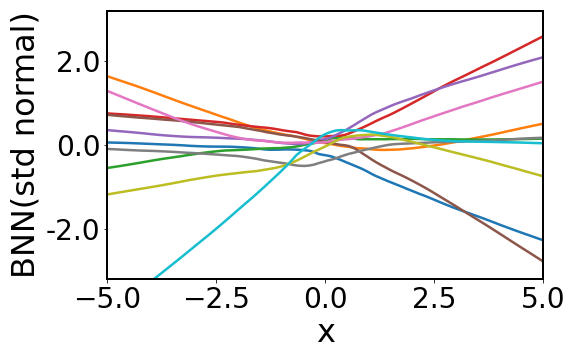}
	\hfill
	\includegraphics[width=0.325\textwidth]{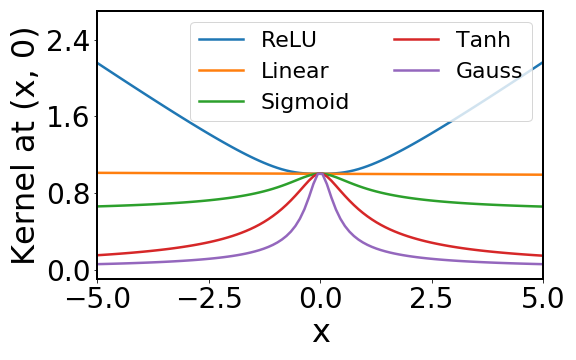}
	\hfill
	\includegraphics[width=0.325\textwidth]{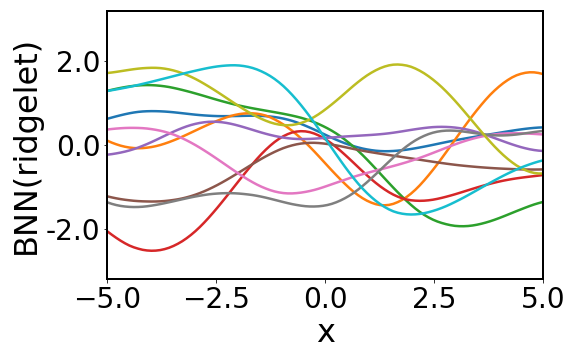}
	\hfill
	\caption{\emph{Lifting the parameter prior distribution of a Bayesian neural network (BNN) to the output space of the network.} 
		\emph{Left:} Realisations from a BNN, with the ReLU activation function and independent standard Gaussian distributions on the weights and bias parameters. 
		\emph{Middle:} The covariance function of a BNN, as the activation function is varied over ReLU, linear, sigmoid, hyperbolic tangent and Gaussian.
		\emph{Right:} Realisations from a BNN endowed with a ridgelet prior, which is constructed to approximate a GP with covariance $k(\bm{x},\bm{y}) = \sigma^2\exp(-\frac{1}{2 l^2} \|\bm{x}-\bm{y}\|_2^2)$ with $ \sigma = 1.0 $, $ l = 1.75 $. 
		[In all cases one hidden layer was used.]
	}
	\label{fig:BNN_priors}
\end{figure}

\subsection{The Covariance of Bayesian Neural Networks}

Here we discuss existing choices for the prior $\mathbb{P}$ on $\theta$ in a BNN, which are motivated by the covariance structure that they induce on the output space of the neural network. 
This is a rapidly evolving field and a full review requires a paper in itself; we provide a succinct summary and refer the reader to the survey in \cite{Nalisnick2018}.

Several deep connections between BNNs and GPs have been exposed \citep{Rasmussen2006,Stein1999,Kristiadi2020,Khan2019,Adlam2020}. 
The first detailed connection between BNNs and GPs was made by \citet{Neal1995}. 
Let $[\bm{w}_i^l]_{j} := w_{i,j}^l$.
In the case of a shallow BNN with $L = 1$, assume that each of the weights $\bm{w}_1^1$, $\bm{w}_i^0$ and biases $b_i^0$ are \textit{a priori} independent, each with mean $0$ and with finite second moments $\sigma_{\bm{w}^1}^2, \sigma_{\bm{w}^0}^2, \sigma_{b}^2$, respectively, where $ \sigma_{\bm{w}^1}^2 = \sigma^2/{N_1} $ for some fixed $\sigma>0$. To improve presentation, let $f(\bm{x}) := f(\bm{x} , \theta)$, so that $\theta$ is implicit, and let $\mathbb{E}$ denote expectation with respect to $\theta \sim \mathbb{P}$. 
A well-known result from \citet{Neal1995} is that, according to the central limit theorem, the BNN converges asymptotically (as $N_1 \rightarrow \infty$) to a zero-mean GP with covariance function
\begin{align}
k^1(\bm{x}, \bm{x}') := \mathbb{E}[f(\bm{x}) f(\bm{x}')] = \sigma^2 \mathbb{E} \left[ \phi\left(\bm{w}_{1}^0 \cdot \bm{x} + b_{1}^0 \right) \phi\left(\bm{w}_{1}^0 \cdot \bm{x}' + b_{1}^0 \right) \right] + \sigma_{b}^2, \label{eq: C_BNN}
\end{align}
Analytical forms of the GP covariance were obtained for several activation functions $\phi$, such as the ReLU and Gaussian error functions in, \cite{Lee2018,Williams1996,Yang2019a}. Furthermore, similar results were obtained more recently for neural networks with multiple hidden layers in \cite{Lee2018,Matthews2018,Novak2019,Garriga-Alonso2019}. 
Placing independent priors with the same second moments $\sigma^2/N_{l-1}$ and $\sigma^2_b$ on the weights and biases of the $l^{\text{th}}$ layer, and taking $N_1 \rightarrow \infty, N_2 \rightarrow \infty,\ldots$ in succession, it can be shown that the $l^{\text{th}}$ layer of this BNN convergences to a zero mean GP with covariance:
\begin{align}
k^l(\bm{x},\bm{x}') = \sigma^2 \mathbb{E}_{z^{l-1}_i \sim \mathcal{GP}(0,k^{l-1})}[ \phi(z_i^{l-1}(\bm{x})) \phi(z_i^{l-1}(\bm{x}')) ] +\sigma^2_b.
\end{align}
Of course, the discussion of this section is informal only and we refer the reader to the original references for full and precise detail.

The identification of limiting forms of covariance function allows us to investigate whether such priors are suitable for performing uncertainty quantification in real-world tasks.
Unfortunately, the answer is often ``no''. 
One reason, which has been demonstrated empirically by multiple authors \citep{Hafner2019,Yang2019,Yao2019}, is that BNNs can have poor out-of-sample performance. 
Typically one would want a covariance function to have a notion of locality, so that $k(\bm{x},\bm{x}')$ decays sufficiently rapidly as $\bm{x}$ and $\bm{x}'$ become distant from each other.
This ensures that when predictions are made for a location $\bm{x}$ that is far from the training dataset, the predictive variance is appropriately increased.
However, as exemplified in \Cref{fig:BNN_priors}, the covariance structure of a BNN need not be local.
Even the use of a light-tailed Gaussian activation function still has a possibility of leading to a covariance model that is non-local.
These existing studies \citep{Hafner2019,Yang2019,Yao2019} illustrate the difficulties of inducing a meaningful prior on the output space of the neural network when operating at the level of the parameters $\theta$ of the network.


\section{Methods}
\label{sec:methodology}

In this section the ridgelet approach to prior specification is presented.
The approach relies on the classical ridgelet transform, which is briefly introduced in \Cref{subset: RT intro}.
Then, in \Cref{subsec: prior construction}, we describe how a ridgelet prior is constructed.

\subsection{The Ridgelet Transform} \label{subset: RT intro}

The \textit{ridgelet transform} \citep{Candes1998, Murata1996b, Sonoda2017b} was developed in the context of harmonic analysis in the 1990s \citep{Barron1993b, Jones1992b, Leshno1993b, Murata1996b, Kurkova2001} and has received recent interest as a tool for the theoretical analysis of neural networks \citep{Sonoda2017b,Bach2018b,Sonoda2018b, Ongie2019b}.
In this section we provide a brief and informal description of the ridgelet transform, deferring all mathematical details until \Cref{sec:theory}.
To this end, let $\widehat{f}$ denote the Fourier transform of a function $f$, and let $\bar{z}$ denote the complex conjugate of $z \in \mathbb{C}$.
Given an activation function $\phi : \mathbb{R} \rightarrow \mathbb{R}$, suppose we have a corresponding function $\psi : \mathbb{R} \rightarrow \mathbb{R}$ such that the relationship 
\begin{align*}
(2 \pi)^{\frac{d}{2}} \int_{\mathbb{R}} | \xi |^{-d} \overline{\widehat{\psi}(\xi)} \widehat{\phi}(\xi) \mathrm{d} \xi = 1
\end{align*}
is satisfied.
Such a function $\psi$ is available in closed form for many of the activation functions $\phi$ that are commonly used in neural networks; examples can be found in \Cref{ex: ridgelet_func} in \Cref{subsec: ridgelet}.
Then, under regularity conditions detailed in \Cref{sec:theory}, the ridgelet transform of a function $f : \mathbb{R}^d \rightarrow \mathbb{R}$ is defined as
\begin{eqnarray}
R[f](\bm{w}, b) := \int_{\mathbb{R}^{d}} \psi(\bm{w} \cdot \bm{x} + b) f(\bm{x}) \mathrm{d} \bm{x} \label{eq: classical ridgelet}
\end{eqnarray}
for $ \bm{w} \in \mathbb{R}^d $, and $ b \in \mathbb{R} $, and the dual ridgelet transform of a function $\tau : \mathbb{R}^d \times \mathbb{R} \rightarrow \mathbb{R}$ is defined as
\begin{eqnarray}
R^{\ast}[\tau](\bm{x}) := \int_{\mathbb{R}^{d+1}} \phi(\bm{w} \cdot \bm{x} + b) \tau(\bm{w}, b) \mathrm{d} \bm{w} \mathrm{d} b \label{eq: classical inv trans}
\end{eqnarray}
for $ \bm{x} \in \mathbb{R}^d$.
There are two main properties of the ridgelet transform that we exploit in this work. 
First, a discretisation of \eqref{eq: classical inv trans} using a cubature method gives rise to an expression closely related to one layer of a neural network; c.f. \eqref{eq:twolayer_neural_net}.
Second, under regularity conditions, the dual ridgelet transform works as the pseudo-inverse of the ridgelet transform, meaning that $(R^*  R)[ f] = f$ whenever the left hand side is defined.
Next we explain how these two properties of the ridgelet transform will be used.

\subsection{The Ridgelet Prior} \label{subsec: prior construction}

In this section our proposed ridgelet prior is presented.
Our starting point is a Gaussian stochastic process $\mathcal{GP}(m,k)$ and we aim to construct a probability distribution for the parameters $\theta$ of a neural network $f(\bm{x},\theta)$ in \eqref{eq:twolayer_neural_net} such that the stochastic process $\theta \mapsto f(\cdot,\theta)$ closely approximates the GP, in a sense yet to be defined.

The construction proceeds in three elementary steps.
The first step makes use of the property that discretisation of $R^*$ in \eqref{eq: classical inv trans} using a cubature method (i.e. a linear combination of function values) gives rise to a neural network.
To see this, let us abstractly denote by $\tilde{R}$ and $\tilde{R}^*$ approximations of $R$ and $R^*$ obtained using a cubature method with $D$ nodes:
\begin{align}
R[f](\bm{w}, b) \approx \tilde{R}[f](\bm{w},b) & := \sum_{j=1}^D u_j  \psi(\bm{w} \cdot \bm{x}_j + b) f(\bm{x}_j) \label{eq: discretise R} \\
R^{\ast}[\tau](\bm{x}) \approx \tilde{R}^* [\tau ](\bm{x}) & := \sum_{i=1}^{N_1} v_i  \phi(\bm{w}_i^0 \cdot \bm{x} + b_i^0) \tau(\bm{w}_i, b_i) \label{eq: discretise Rs}
\end{align}
where $ ( \bm{x}_j, u_j )_{j=1}^{D} \subset \mathbb{R}^{d} \times \mathbb{R} $ and $ ( ( \bm{w}_i^0, b_i^0 ) , v_i )_{i=1}^{N_1} \subset \mathbb{R}^{d+1} \times \mathbb{R} $ are the cubature nodes and weights employed respectively in \eqref{eq: discretise R} and \eqref{eq: discretise Rs}.
The specification of suitable cubature nodes and weights will be addressed in \Cref{sec:theory}, but for now we assume that they have been specified.
It is clear that \eqref{eq: discretise Rs} closely resembles one layer of a neural network; c.f. \eqref{eq:twolayer_neural_net}.

The second step makes use of the fact that $(R^* R)[f] = f$, which suggests that we may approximate a function $f$ using the discretised ridgelet transform and its dual
\begin{align}
(\tilde{R}^* \tilde{R})[f](\bm{x}) & = \sum_{i=1}^{N_1} v_i \phi(\bm{w}_i^0 \cdot \bm{x} + b_i^0) \left[ \sum_{j=1}^D u_j \psi(\bm{w}_i^0 \cdot \bm{x}_j + b_i^0) f(\bm{x}_j) \right] \nonumber \\
& = \sum_{i=1}^{N_1} \underbrace{ \sum_{j=1}^D \left[ v_i u_j \psi(\bm{w}_i^0 \cdot \bm{x}_j + b_i^0 ) f(\bm{x}_j) \right]  }_{=: w_{1,i}^1} \phi(\bm{w}_i^0 \cdot \bm{x} + b_i^0) \label{eq: sub in discrete}
\end{align}
where the coefficients $\bm{w}_1^1 = (w_{1,1}^1,\dots,w_{1,N_1}^1)^\top$ depend explicitly on the function $f$ being approximated.
Thus $\tilde{R}^* \tilde{R}$ is a linear operator that returns a neural network approximation to each function $f$ provided as input.

The third and final step is to compute the pushforward of $\mathcal{GP}(m,k)$ through the linear operator $\tilde{R}^* \tilde{R}$, in order to obtain a probability distribution over the coefficients $\bm{w}_1^1$ of the neural network in \eqref{eq: sub in discrete}.
Let $[ \bm{f} ]_i := f(\bm{x}_i)$, $[ \bm{m} ]_i := m(\bm{x}_i)$, $[\bm{K}]_{i,j} := k(\bm{x}_i,\bm{x}_j)$ and $[\bm{\Psi}^0]_{i,j} := v_i u_j \psi(\bm{w}_i^0 \cdot \bm{x}_j + b_i^0)$ so that $\bm{f}, \bm{m} \in \mathbb{R}^{D}$, $\bm{K} \in \mathbb{R}^{D \times D}$ and $\bm{\Psi}^0 \in \mathbb{R}^{N_1 \times D}$.
If $f \sim \mathcal{GP}(m,k)$ then it follows immediately that $\bm{w}_1^1 = \bm{\Psi}^0 \bm{f}$ is a Gaussian random vector with
\begin{align*}
\mathbb{E}[\bm{w}_1^1] & = \bm{\Psi}^0 \bm{m} \\
\mathbb{E}[(\bm{w}_1^1 - \mathbb{E}[\bm{w}_1^1]) ( \bm{w}_1^1 - \mathbb{E}[\bm{w}_1^1])^\top ] & = \bm{\Psi}^0 \bm{K} (\bm{\Psi}^0)^\top .
\end{align*}
To arrive at a prior for a general neural network of the form \eqref{eq:twolayer_neural_net} we apply this construction recursively, starting from the input layer and working up towards the output layer.
The dimension of the cubature rule $((\bm{w}_i^{l-1},b_i^{l-1}),v_i^{l-1})_{i=1}^{N_l}$ used at level $l-1$ is required to equal $N_l$ so that our discretised ridgelet transform inherits the same network architecture as in \eqref{eq:twolayer_neural_net}.
Our notation is generalised to $[\bm{\Psi}^{l-1}]_{i,j} := v_i u_j \psi(\bm{w}_{i}^{l-1} \cdot \bm{\phi}^{l-1}(\bm{x}_j) + b_{i}^{l-1})$ in order to indicate that this cubature rule with $N_l$ elements was used to construct the matrix $\bm{\Psi}^{l-1}$, where $[ \bm{\phi}^{l-1}(\bm{x}_j) ]_i := \phi(z_i^{l-1}(\bm{x}_j))$ and $\bm{\phi}^{0}(\bm{x}_j) = \bm{x}_j$.
Let $N_{L+1}$ be the output dimension of the neural network in \eqref{eq:twolayer_neural_net}, that is $N_{L+1} = 1$.
Our ridgelet prior can now be formally defined:

\begin{definition}[Ridgelet Prior] \label{def: BNN_prior}
	Consider the neural network in \eqref{eq:twolayer_neural_net}.
	Given a mean function $m: \mathbb{R}^d \to \mathbb{R}$ and a covariance function $k : \mathbb{R}^d \times \mathbb{R}^d \rightarrow \mathbb{R}$, a prior distribution is called a \emph{ridgelet prior} if the weights $\bm{w}_i^l$ at level $l$ depend on the weights and biases at level $l-1$ according to
	\begin{align*}
	\bm{w}_i^l  | \{ (\bm{w}_r^{l-1},b_r^{l-1}) : r = 1,\dots,N_{l} \} & \stackrel{\text{\normalfont i.i.d.}}{\sim} \mathcal{N}(\bm{\Psi}^{l-1} \bm{m}, \bm{\Psi}^{l-1} \bm{K} (\bm{\Psi}^{l-1})^\top)  
	\end{align*}
	where $i = 1,\dots,N_{l+1}$ and $l = 1,\dots,L$.
	To complete the prior specification, the bias parameters $b_i^{l-1}$ at all layers and the weights $\bm{w}_i^0$ at the input layer are required to be independent and distributed as $b_i^{l-1} \sim \mathcal{N}(0, \sigma_{b}^2)$ and $\bm{w}_i^0 \sim \mathcal{N}(0, \sigma_{\bm{w}}^2 \bm{I}_{d \times d})$.
\end{definition}

\begin{figure}[t!]
	\includegraphics[height=90pt]{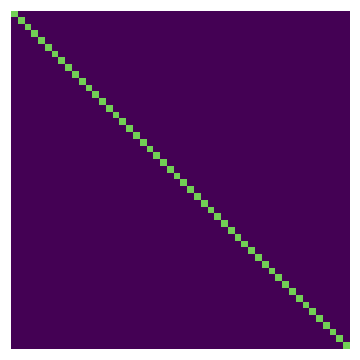}
	\hfill
	\includegraphics[height=90pt]{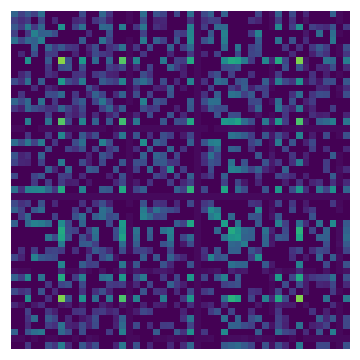}
	\includegraphics[height=90pt]{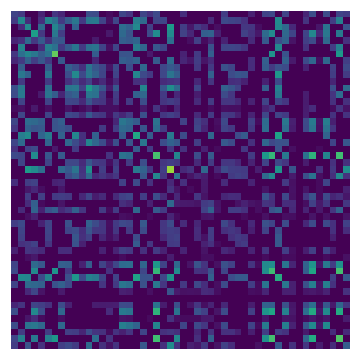}
	\raisebox{-5pt}{\includegraphics[height=97pt]{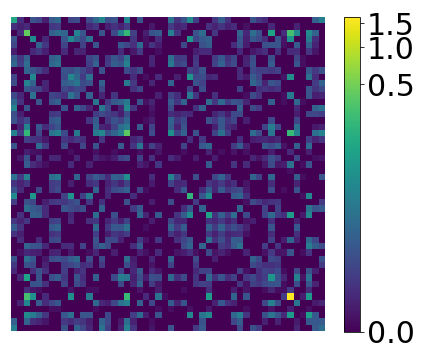}}
	
	\caption{\emph{Example of the covariance matrix $\bm{\Psi}^{l-1} \bm{K} (\bm{\Psi}^{l-1})^\top$ of the ridgelet prior for $l = 1$ and $N_{1} = 50$.} 
		\emph{Left:} The covariance matrix of independent standard Gaussian distribution prior on $\mathbb{R}^{N_{1}}$ for comparison.
		\emph{Right:} The covariance matrix $\bm{\Psi}^{0} \bm{K} (\bm{\Psi}^{0})^\top$ computed from 3 independent realisations of $\{ \bm{w}_i^0, b_i^0 \}_{i=1}^{50}$ and from $\bm{K}$ of a Gaussian covariance function.}
	\label{fig:exp_prior_covmat}
\end{figure}

Several initial remarks are in order:

\begin{remark}
	The dependence of the distribution for the weights $\bm{w}_i^l$ on the previous layer's weights $\bm{w}_{r}^{l-1}$ and the biases $b_{r}^{l-1}$, $r = 1,\dots,N_{l}$, is an important feature of our construction and seems essential if we aim to approximate a user-specified covariance model.
	This dependence is illustrated in \Cref{fig:exp_prior_covmat}.
\end{remark}

\begin{remark}
	Let $[\bm{\phi}^0(\bm{x})]_i := \phi(\bm{w}_i^0 \cdot \bm{x} + b_i^0)$.
	In the comparison with \eqref{eq: C_BNN}, the covariance of a BNN with one hidden layer and endowed with our ridgelet prior takes the form
	\begin{eqnarray}
	\mathbb{E}[(f(\bm{x}) - m(\bm{x})) (f(\bm{x}') - m(\bm{x}'))] = \mathbb{E} \left[ \bm{\phi}^0(\bm{x})^\top \bm{\Psi}^0 \bm{K} (\bm{\Psi}^0)^\top \bm{\phi}^0(\bm{x}') \right] \label{eq: one layer ridgelet cov}
	\end{eqnarray}
	where the expectation on the right hand side is with respect to the random weights $\bm{w}_i^0$ and biases $b_i^0$.
	In \Cref{sec:theory} it is shown that, in an appropriate limit, the expression in \eqref{eq: one layer ridgelet cov} converges to $k(\bm{x},\bm{x}')$, the covariance model that we aimed to approximate at the outset.
\end{remark}

\begin{remark}
	To limit scope we consider a standard feed-forward network architecture, but it is possible in principle to extend the ridgelet prior to other types of network.
	For convolutional neural networks (CNN), we note that convolutional layers can be seen as feed-forward layers with a sparse weight matrix. 
	In this case one could consider taking the ridgelet prior as a starting point and then condition on specific weights being exactly zero, in order to produce identical sparsity structure to a CNN \citep[see for example Section 2.1 of][]{Garriga-Alonso2019}.
	For residual neural networks (ResNets), the predictor vector $\bm{\phi}^{l}$ can be augmented to include any skip connections that may be present.
	However, our analysis in \Cref{sec:theory} focuses on the simple case of a single hidden layer, for which both CNNs and ResNets reduce to a standard feed-forward neural network.
\end{remark}

This completes our definition of the ridgelet prior.
An illustration is provided in \Cref{fig: exp_prior}, the full details of which are reserved for \Cref{sec: exp_01}.
It can be seen that as $N$, the number of hidden units in the BNN, is increased the samples from the BNN begin to resemble, in a statistical sense, samples from the target GP.
In the case of multiple hidden layers, a larger number $N$ of hidden units appear to be required to achieve a similar degree of approximation to the GP.
Next we present our theoretical analysis, which considers only the case of one hidden layer, and is the principal contribution of the paper.

\begin{figure}[t!]
	
	\subcaptionbox{BNN ($\scriptsize \begin{array}{l} L = 1 \\ N=100 \end{array}$)}{\includegraphics[width=0.2425\textwidth]{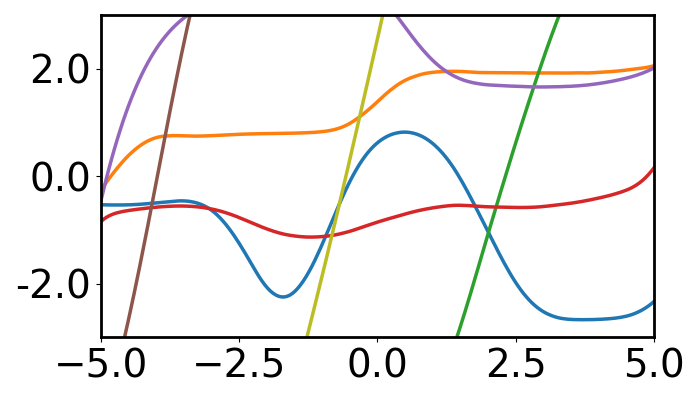}}
	\subcaptionbox{BNN ($\scriptsize \begin{array}{l} L = 1 \\ N=1,000 \end{array}$)}{\includegraphics[width=0.2425\textwidth]{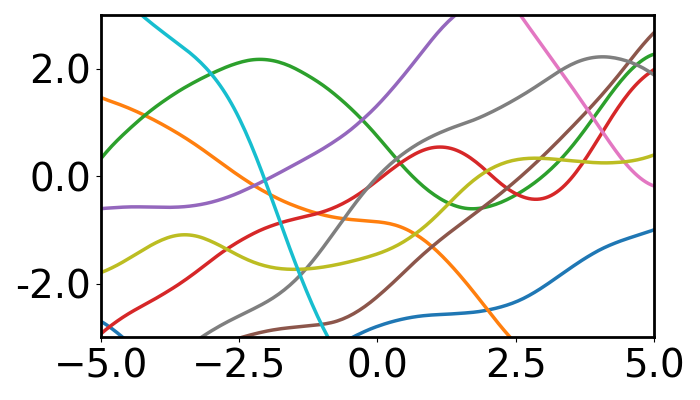}}
	\subcaptionbox{BNN ($\scriptsize \begin{array}{l} L = 1 \\ N=3,000 \end{array}$)}{\includegraphics[width=0.2425\textwidth]{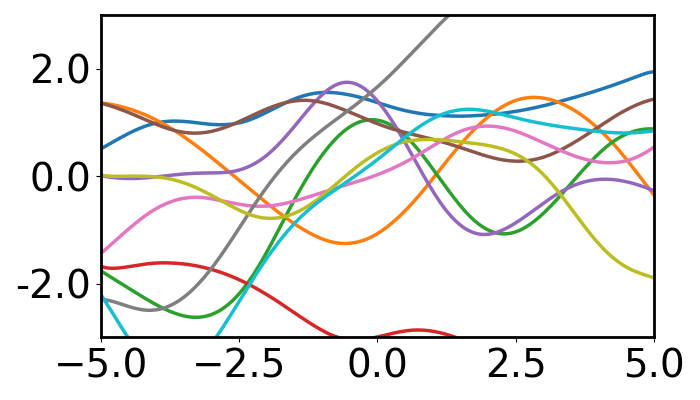}}
	
	\subcaptionbox{BNN ($\scriptsize \begin{array}{l} L = 3 \\ N=5,000 \end{array}$)}{\includegraphics[width=0.2425\textwidth]{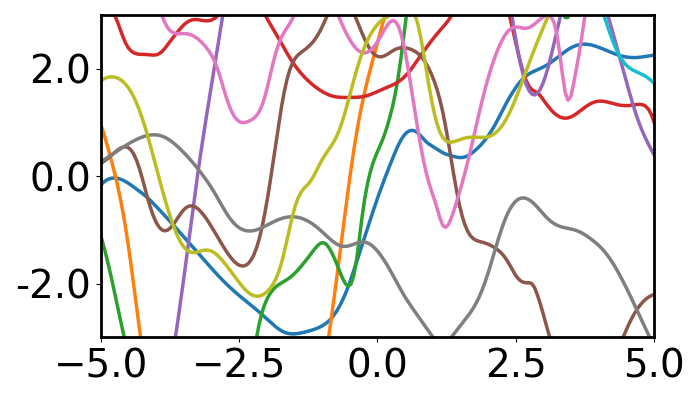}}
	\subcaptionbox{BNN ($\scriptsize \begin{array}{l} L = 3 \\ N=10,000 \end{array}$)}{\includegraphics[width=0.2425\textwidth]{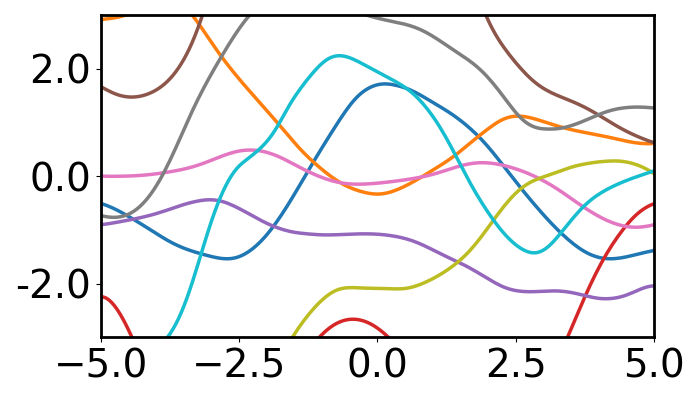}}
	\subcaptionbox{BNN ($\scriptsize \begin{array}{l} L = 3 \\ N=50,000 \end{array}$)}{\includegraphics[width=0.2425\textwidth]{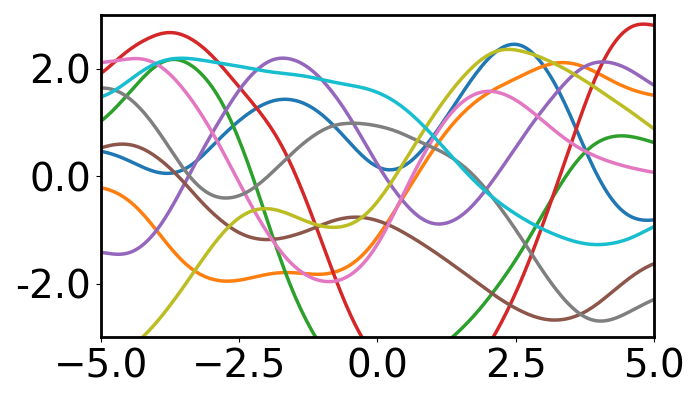}}
	\subcaptionbox{GP}{\includegraphics[width=0.2425\textwidth]{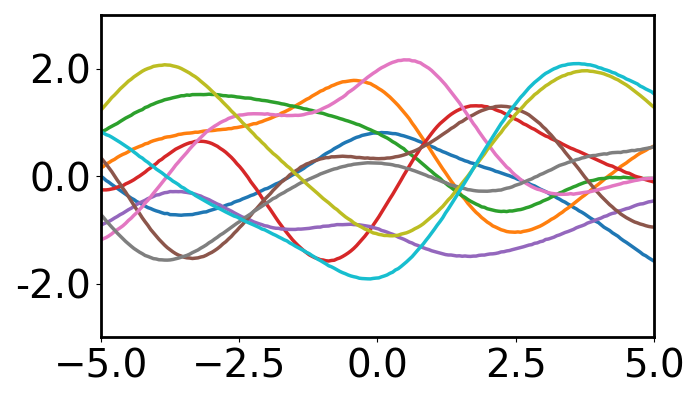}}
	
	\caption{Sample paths from a Bayesian neural network (BNN) equipped with a ridgelet prior. 
		Here we examine the effect of increasing the number $N$ of hidden units in a BNN with $L=1$ hidden layer from (a) 100, to (b) 1,000 and (c) 3,000, and in a BNN with $L=3$ hidden layers from (d) 5,000, to (e) 10,000 and (f) 50,000. 
		In (g) sample paths are shown from the target Gaussian process (GP) prior.
		Full details of this example are reserved for \Cref{sec: exp_01}.
	}
	\label{fig: exp_prior}
\end{figure}


\section{Theoretical Assessment}
\label{sec:theory}

This section presents our theoretical analysis of the ridgelet prior in the setting of a single hidden layer ($L=1$).
The main result is \Cref{thm: recon_rate}, which establishes an explicit error bound for approximation of a GP prior using the ridgelet prior in a BNN.
The bound is non-asymptotic and valid for networks with finite numbers of units, which is an important point of distinction from earlier work \citep{Neal1995, Williams1996}.
First, in \Cref{sec: def_not}, we introduce definitions and notation required for the ridgelet transform to be rigorously studied.
In \Cref{subsec: ridgelet} we introduce a finite-bandwidth ridgelet transform, with the error between the finite-bandwidth and original ridgelet transforms being analysed in \Cref{sec: proof_recon_mod}.
\Cref{subsec: theorems etc} applies these finite-bandwidth results to perform a theoretical analysis of the ridgelet prior for a BNN.

\vspace{5pt}
\noindent
{\bf Notation:}
The following notation will be used.
For an open set $\mathcal{X} \subseteq \mathbb{R}^d$ and a Borel measure $\mu$ on $\mathcal{X}$, let $L^p(\mathcal{X}, \mu)$ denote the set of functions $f : \mathcal{X} \rightarrow \mathbb{R}$ such that $\| f \|_{L^p(\mathcal{X}, \mu)} < \infty$ where
$$
\| f \|_{L^p(\mathcal{X}, \mu)} := \left\{ \begin{array}{ll} \left( \int_{\mathcal{X}} | f(x) |^p \mathrm{d} \mu(x) \right)^{1/p} & 1 \leq p < \infty \\
\text{ess} \sup_{\bm{x} \in \mathcal{X}} | f(\bm{x}) | & p = \infty. \end{array} \right.
$$ 
If $\mu$ is the Lebesgue measure on $\mathcal{X}$, we use the shorthand $L^p(\mathcal{X})$ for $L^p(\mathcal{X}, \mu)$ and furthermore we let $L^p_{\text{loc}}(\mathcal{X})$ denote the set of function $f : \mathcal{X} \rightarrow \mathbb{R}$ such that $\| f \|_{L^p(K)}$ exists and is finite on all compact $K \subseteq \mathcal{X}$.
Let $C(\mathcal{X})$ denote the set of all continuous functions $f : \mathcal{X} \rightarrow \mathbb{R}$ and let $C^{1}(\mathcal{X})$ denote the set of all functions $f : \mathcal{X} \rightarrow \mathbb{R}$ whose first order partial derivatives $\partial_{x_i}f(\bm{x})$ exist and are continuous on $\mathcal{X}$.
Denote by $C^{1 \times 1}(\mathcal{X} \times \mathcal{X})$ the set of all functions $h : \mathcal{X} \times \mathcal{X} \rightarrow \mathbb{R}$ whose mixed first order derivatives  $\partial_{x_i} \partial_{y_j} h(\bm{x},\bm{y})$ exist and are continuous on $\mathcal{X} \times \mathcal{X}$.
For multinomials and higher-order derivatives, we employ multi-index notation $\bm{x}^{\bm{\alpha}} := x_1^{\alpha_1} \dots x_d^{\alpha_d}$, $\partial^{\bm{\alpha}} h(\bm{x}) := \partial_{x_1}^{\alpha_1} \dots \partial_{x_d}^{\alpha_d} h(\bm{x})$ and $\partial^{\bm{\alpha} , \bm{\beta}}h(\bm{x},\bm{y}) := \partial_{x_1}^{\alpha_1} \dots \partial_{x_d}^{\alpha_d} \partial_{y_1}^{\beta_1} \dots \partial_{y_d}^{\beta_d} h(\bm{x},\bm{y})$.
A bivariate function $k : \mathcal{X} \times \mathcal{X} \rightarrow \mathbb{R}$ is called \emph{positive definite} if $\sum_{i,j} a_i a_j k(\bm{x}_i, \bm{x}_j) > 0$ for all $\bm{0} \neq (a_1,\dots,a_n) \in \mathbb{R}^n$ and all distinct $\{\bm{x}_i\}_{i=1}^n \subset \mathcal{X}$.

\subsection{Regularity of the Activation Function} \label{sec: def_not}

In this section we outline our regularity assumptions on the activation function $\phi$.
To do this we first recall the classical Fourier transform and its generalisations, all for real-valued functions on $\mathbb{R}^d$. 

\vspace{5pt}
\noindent
{\bf Fourier transform on $L^1(\mathbb{R}^d)$ and $L^2(\mathbb{R}^d)$:} 
The Fourier transform of $f \in L^1(\mathbb{R}^d)$ is defined by $ \widehat{f}(\bm{\xi}) := (2 \pi)^{-\frac{d}{2}} \int_{\mathbb{R}^d} f(\bm{x}) \exp(- i \bm{\xi} \cdot \bm{x}) \mathrm{d} \bm{x} $ for each $\bm{\xi} \in \mathbb{R}^d$.
The Fourier transform of $f \in L^2(\mathbb{R}^d)$ is formally defined as a limit of a sequence $(\hat{f}_n)_{n \in \mathbb{N}}$ where the $f_n \in L^1(\mathbb{R}^d) \cap L^2(\mathbb{R}^d)$ and $f_n \rightarrow f$ in $L^2(\mathbb{R}^d)$ \citep[see e.g.][p.113-114]{Grafakos2000b}.

\vspace{5pt}
The image of the Fourier transform on $L^1(\mathbb{R}^d)$ is not contained in $L^1(\mathbb{R}^d)$.
However, there exists a subset of $L^1(\mathbb{R}^d)$ on which the Fourier transform defines an automorphism.
It is convenient to work on this subset, which consists of so-called \emph{Schwartz functions}, defined next.

\vspace{5pt}
\noindent
{\bf Schwartz functions:} 
An infinitely differentiable function $f$ is called a \textit{Schwartz function} if for any pair of multi-indices $ \bm{\alpha}, \bm{\beta} \in \mathbb{N}_{0}^{d} $ we have that $C_{\bm{\alpha}, \bm{\beta}} := \sup_{\bm{x} \in \mathbb{R}^d} \left| \bm{x}^{\bm{\alpha}} \partial^{\bm{\beta}} f(\bm{x}) \right| < \infty $. 
The set of Schwartz functions is denoted by $ \mathcal{S}(\mathbb{R}^d) $ \cite[p.105]{Grafakos2000b}.
Note that the Fourier transform on $\mathcal{S}(\mathbb{R}^d)$ is well-defined as $ \mathcal{S}(\mathbb{R}^d) \subset L^1(\mathbb{R}^d)$.
Moreover, the Fourier transform is a homeomorphism from $ \mathcal{S}(\mathbb{R}^d) $ onto itself \cite[p.113]{Grafakos2000b}.

\vspace{5pt}
The most commonly encountered activation functions $\phi$ are not elements of $L^1(\mathbb{R})$ and we therefore also require the notion of a generalised Fourier transform:

\vspace{5pt}
\noindent
{\bf Generalised Fourier transform:} 
Let $ \mathcal{S}_m(\mathbb{R}^d) $ be the vector space of functions $ f \in \mathcal{S}(\mathbb{R}^d) $ that satisfy $ f(\bm{x}) = \mathcal{O}(\| \bm{x} \|^m) $ for $ \bm{x} \to 0 $.
For $ k \in \mathbb{N}_{0} $, let $ t: \mathbb{R}^d \to \mathbb{R} $ be any continuous function of at most polynomial growth of order $ k $, meaning that $\sup_{\bm{x} \in \mathbb{R}^d} |t(\bm{x})| / (1 + \|\bm{x}\|^k) < \infty$.
A measurable function $ \widehat{t} \in L^2_{\text{loc}}(\mathbb{R}^d \setminus \{ \bm{0} \}) $ is called a \emph{generalised Fourier transform} of $ t $ if there exists an integer $ m \in \mathbb{N}_{0} $ such that $ \int_{\mathbb{R}^{d}} \widehat{t}(\bm{w}) f(\bm{w}) \mathrm{d} \bm{w} = \int_{\mathbb{R}^{d}} t(\bm{x}) \widehat{f}(\bm{x}) \mathrm{d} \bm{x} $ for all $ f \in \mathcal{S}_{2m}(\mathbb{R}^d) $ \cite[p.103]{Wendland2005}.
The set of all continuous functions of at most polynomial growth of order $k$ that admit a generalized Fourier transform will be denoted $ C_{k}^{*}(\mathbb{R}^d) $.

The generalized Fourier transform can be computed for activation functions $\phi$ that are typically used in a neural network, for which classical Fourier transforms are not well-defined:

\begin{example}[Generalised Fourier transform of the ReLU function]
	Let $ \phi(x) := \max(0, x) $, then $\phi \in C_1^{*}(\mathbb{R})$.
	Although $\phi$ is not an element of $L^1(\mathbb{R}^d)$, it has polynomial growth and admits a generalised Fourier transform $ \widehat{\phi}(w) = - (\sqrt{2 \pi} w^2)^{-1} $.
\end{example}

\begin{example}[Generalised Fourier transform of the $\tanh$ function]
	Let $ \phi(x) := \tanh(x) $, then $\phi \in C_0^{*}(\mathbb{R})$.
	Likewise, $\phi$ admits a generalised Fourier transform $ \widehat{\phi}(w) = - i \sqrt{\frac{\pi}{2}} \mathrm{csch}\left( \frac{\pi}{2} w \right) $ where $\mathrm{csch}$ is the hyperbolic cosecant.
\end{example}

In order to present our theoretical results, we will make the following assumptions on the activation function $\phi$ that defines the ridgelet transform:

\begin{assumption}[Activation function] \label{assm: ridgelet pair}
	The activation function $\phi : \mathbb{R} \rightarrow \mathbb{R}$ satisfies
	\begin{enumerate}
		\item $ \phi \in C_{0}^{*}(\mathbb{R}),$
	\end{enumerate}
	and there exists a function $\psi : \mathbb{R} \rightarrow \mathbb{R}$ such that
	\begin{equation*}
	\text{2. } \psi \in \mathcal{S}(\mathbb{R}), \qquad  \text{3. }(2 \pi)^{\frac{d}{2}} \int_{\mathbb{R}} | \xi |^{-d} \overline{\widehat{\psi}(\xi)} \widehat{\phi}(\xi) \mathrm{d} \xi = 1, \qquad \text{4. }\int_{\mathbb{R}} | \xi |^{-d-2} \big| \overline{\widehat{\psi}(\xi)} \widehat{\phi}(\xi) \big| \mathrm{d} \xi < \infty.
	\end{equation*}
\end{assumption}

The boundedness assumption $\phi \in C_0^*(\mathbb{R})$ rules out some commonly used activation functions, such as ReLU, but enables stronger convergence results to be obtained.
However, our analysis is also able to handle $\phi \in C_1^*(\mathbb{R})$.
For presentational purposes we present results under the assumption $\phi \in C_0^*(\mathbb{R})$ in the main text and, in \Cref{sec: proof_lemma_unbounded} we present theoretical results for the unbounded setting $\phi \in C_1^*(\mathbb{R})$.
Parts 2 and 3 of \Cref{assm: ridgelet pair} are the standard assumptions for the ridgelet transform \citep{Candes1998, Murata1996b}; several examples of $(\phi, \psi)$ pairs are shown in \Cref{ex: ridgelet_func}.

\begin{table}[t!]
	\centering
	\footnotesize{
		\begin{tabular}{|l|l|}
			\hline
			\textbf{Activation $\phi$} & \textbf{Associated Function $\psi$} \\
			\hline
			$ \tanh(z) $ & $ \frac{\mathrm{d}^{d-r+2}}{\mathrm{d} z^{d-r+2}} \left[ \exp\left( - \frac{z^2}{2}  \right) \sin\left( \frac{\pi z}{2} \right) \right] \times 2^{-\frac{d+r}{2}} \pi^{-\frac{d-r+2}{2}} \exp\left( \frac{\pi^2}{8} \right) $ \\
			$ ( 1 + \exp(- z) )^{-1} $ & $ \frac{\mathrm{d}^{d-r+2}}{\mathrm{d} z^{d-r+2}} \left[ \exp\left( - \frac{z^2}{2}  \right) \sin\left( \pi z \right) \right] \times 2^{-\frac{d+r}{2}} \pi^{-\frac{d-r+2}{2}} \exp\left( \frac{\pi^2}{2} \right) $ \\
			$ \exp\left(- \frac{z^2}{2} \right) $ & $ \frac{\mathrm{d}^{d+r}}{\mathrm{d} z^{d+r}} \exp\left(- \frac{z^2}{2} \right) \times 2^{-\frac{d-2r}{2}} \pi^{-\frac{d+r}{2}}$ \\
			$ \max(0, z) $ & $ \frac{\mathrm{d}^{d+r+2}}{\mathrm{d} z^{d+r+2}} \exp\left( - \frac{z^2}{2} \right) \times \left( - 2^{-\frac{d}{2}} \pi^{-\frac{d-2r+1}{2}} \right) $ \\
			\hline
		\end{tabular}
		\caption{Examples of functions $\phi$, $\psi$ that satisfy the regularity assumptions used in this work where $r = (d \mod 2)$.}
		\label{ex: ridgelet_func}
	}
\end{table}

Next we turn our attention to proving novel and general results about the ridgelet transform, under \Cref{assm: ridgelet pair}.

\subsection{A Finite-Bandwidth Ridgelet Transform} \label{subsec: ridgelet}

Our aim in this section is to approximate the dual ridgelet transform in \eqref{eq: classical inv trans} with a \emph{finite-bandwidth} transform, meaning that the Lebesgue measure  $\lambda(\bm{w},b)$ is approximated using a finite measure $\lambda_\sigma(\bm{w},b)$. 
This will in turn enable cubature methods to be applied to discretise \eqref{eq: classical inv trans}.
In anticipation of our analysis of the ridgelet prior in \Cref{def: BNN_prior}, the finite measure we consider is
\begin{align*}
\lambda_\sigma(\bm{w}, b) := Z \times \frac{1}{( 2 \pi \sigma_{\bm{w}}^2)^{\frac{d}{2}}} \exp\left( - \frac{\| \bm{w} \|^2}{2 \sigma_{\bm{w}}^2} \right) \times \frac{1}{( 2 \pi \sigma_{b}^2)^{\frac{1}{2}}} \exp\left( - \frac{b^2}{2 \sigma_{b}^2} \right) .
\end{align*}
where it will be convenient to set $Z := (2 \pi)^{\frac{1}{2}} \sigma_{\bm{w}}^{d} \sigma_b$ indicating the total measure assigned by $\lambda_{\sigma}$ to $\mathbb{R}^{d+1}$.
It is possible to generalise our analysis beyond this Gaussian case, and other choices for $\lambda_\sigma$ are analysed in \Cref{sec: proof_recon_mod}.

\begin{definition}[Finite-bandwidth ridgelet transform] \label{def: ridgelet_modified}
	In the setting of \Cref{assm: ridgelet pair}, the ridgelet transform $R$ and the \emph{finite-bandwidth} approximation $R_\sigma^\ast$ of its dual $R^*$ are defined pointwise as
	\begin{eqnarray}
	R[f](\bm{w}, b) & := & \int_{\mathbb{R}^{d}} f(\bm{x}) \psi(\bm{w} \cdot \bm{x} + b) \mathrm{d} \bm{x},  \label{eq: int1} \\
	R_\sigma^{\ast}[\tau](\bm{x}) & := & \int_{\mathbb{R}^{d+1}} \tau(\bm{w}, b) \phi(\bm{w} \cdot \bm{x} + b) \mathrm{d} \lambda_\sigma(\bm{w}, b) \label{eq: int2}
	\end{eqnarray}
	for all $ f \in L^1(\mathbb{R}^d) $, $ \tau \in L^1(\mathbb{R}^{d+1}, \lambda_\sigma) $, $ \bm{x} \in \mathbb{R}^d $, $ \bm{w} \in \mathbb{R}^d $, and $ b \in \mathbb{R}$. 
\end{definition}

\noindent The integral transforms in \eqref{eq: int1} and \eqref{eq: int2} are indeed well-defined:
For any $f \in L^1(\mathbb{R}^d)$, the boundedness of $\psi \in \mathcal{S}(\mathbb{R})$ guarantees $R[f] \in L^\infty(\mathbb{R}^{d+1})$.
Similarly for any $ \tau \in L^1(\mathbb{R}^{d+1}, \lambda_\sigma) $, the boundedness of $\phi \in C_0^{*}(\mathbb{R})$ guarantees $R_\sigma^{\ast}[\tau] \in L^\infty(\mathbb{R}^{d})$.
Recall that a discussion on relaxation of the boundedness assumption $\phi \in C_0^*(\mathbb{R})$ is reserved for \Cref{sec: proof_lemma_unbounded}.

The classical ridgelet transform in the sense of \cite{Murata1996b} corresponds to the limit $\sigma_{\bm{w}}, \sigma_{b} \rightarrow \infty$, where the measure $\lambda_\sigma$ becomes flat and $R_\sigma^*$ coincides with $R^*$.
An original contribution of our work, which may be of more general interest, is to study approximation properties of the reconstruction operator $R_\sigma^\ast R$ when a finite measure $ \lambda_\sigma $ is used.
Intuitively, $(R_\sigma^\ast R)$ ought to converge to an identity operator in the limit $\sigma_{\bm{w}}, \sigma_{b} \rightarrow \infty$; an important contribution of this work is to provide an explicit, non-asymptotic approximation error bound in \Cref{sec: proof_recon_mod}.
This analysis, which connects existing theory of the ridgelet transform to practical applications, may therefore be useful in the aforementioned applications of harmonic analysis, as well as in related areas, such as image analysis/denoising, where the ridgelet transform is sometimes employed as an alternative to the wavelet transform \citep{Do2003, AlZu'bi2011}.

\subsection{Convergence of the Finite-Bandwidth Ridgelet Prior} \label{subsec: theorems etc}

In this section, we build on \Cref{thm: recon_mod} to obtain theoretical guarantees for our BNN prior. We present \Cref{thm: recon_rate}, which establishes an explicit error bound for approximation of a GP prior using the ridgelet prior in a BNN.

Recall from \Cref{sec:methodology} that the ridgelet prior was derived from discretisation of the classical ridgelet transform in \eqref{eq: classical ridgelet} and \eqref{eq: classical inv trans}.
The starting point of our analysis is to consider how the finite-bandwidth ridgelet transform in \eqref{eq: int1} and \eqref{eq: int2} can be discretised.
To this end, we view the discretisation of $R_\sigma^\ast R$ as an operator $ I_{\sigma, D, N}: C(\mathbb{R}^d) \to C(\mathbb{R}^d) $ by 
\begin{eqnarray}
I_{\sigma, D, N} [f] (\bm{x}) := \sum_{i=1}^{N} v_i \left( \sum_{j=1}^{D} u_j  f(\bm{x}_j) \psi(\bm{w}_i \cdot \bm{x}_j + b_i) \right) \phi(\bm{w}_i \cdot \bm{x} + b_i) \label{eq: I_sdn}
\end{eqnarray}
where $ \{ \bm{x}_j, u_j \}_{j=1}^{D} $ and $ \{ (\bm{w}_i, b_i), v_i \}_{i=1}^{N} $ are the cubature nodes and weights used, respectively, to discretise the integrals \eqref{eq: int1} and \eqref{eq: int2}.
Informally, $I_{\sigma,D,N}$ ought to converge in some sense to the identity operator in an appropriate limit involving $\sigma,D,N \rightarrow \infty$; this will be made precise in the sequel.
The weights $\bm{w}_i$ and biases $b_i$ can be identified with $\bm{w}_i^0$ and $b_i^0$ in the ridgelet prior, as can $N_1$ be identified with $N$.
Conditional on $\{\bm{w}_i, b_i\}_{i=1}^{N}$, a draw from the BNN $f(\bm{x},\theta) = \sum_{i=1}^{N_1} w_{1,i}^1 \phi (\bm{w}_i^0 \cdot \bm{x} + b_i^0)$ equipped with the ridgelet prior is equal in distribution to $I_{\sigma,D,N}[f]$ when $f$ is drawn from $\mathcal{GP}(m,k)$.

Our analysis exploits properties of the cubature methods defined in \Cref{asmp: disc_mod} and \Cref{asmp: disc_mod_2} below.
To rigorously introduce these cubature methods, let $\mathcal{X}$ be a bounded subset of $\mathbb{R}^d$; without loss of generality we assume $\mathcal{X}$ is contained in the interior of the hyper-cube $[-S,S]^d$ for some constant $S > 0$.
It follows that there exists an infinitely differentiable function (a ``mollifier'') $\mathbbm{1}(\cdot)$ with the properties that $\mathbbm{1}(\bm{x}) \in [0,1]$, $\mathbbm{1}(\bm{x}) = 1$ if $\bm{x} \in \mathcal{X}$ and $\mathbbm{1}(\bm{x}) = 0$ if $\bm{x} \notin [-S,S]^d$.
Indeed, letting $a := 1 - \inf\{\|\bm{x} - \bm{y}\| : \bm{x} \in \mathcal{X}, \bm{y} \notin [-S,S]^d \}$ and 
\begin{eqnarray}
g(t) := \begin{cases}
e^{-\frac{1}{t}} & \text{for } t > 0 \\
0 & \text{for } t \le 0 
\end{cases} \qquad \text{ and } \qquad
h(t) := \frac{g(t)}{g(t) + g(1-t)} \nonumber
\end{eqnarray}
then the function 
\begin{eqnarray}
\mathbbm{1}(\bm{x}) := \prod_{i=1}^d \left[ 1 - h\left(\frac{x_i^2 - a^2}{1 - a^2}\right) \right] \label{eq: mollifier}
\end{eqnarray}
satisfies the desired properties with $\mathbbm{1}(\cdot)$ \cite[p.141-143]{Tu2010}. 
The function $\mathbbm{1}(\cdot)$ is used in our theoretical analysis to restrict, via multiplication, the support of function $f$ from $\mathbb{R}^d$ to $[-S,S]^d$ without changing $f$ on $\mathcal{X}$ and without global smoothness on $\mathbb{R}^d$ being lost; see \Cref{fig: morifier_eg}. 

\begin{figure}[t!]
	\centering
	\subcaptionbox{function before multiplying $\mathbbm{1}$}{\includegraphics[width=0.325\textwidth]{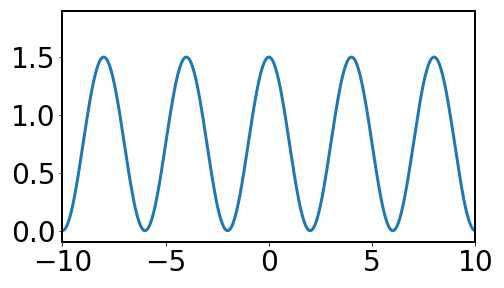}}
	\hfill
	\subcaptionbox{the mollifier $\mathbbm{1}$}{\includegraphics[width=0.325\textwidth]{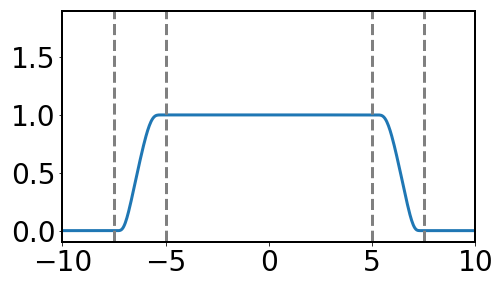}}
	\hfill
	\subcaptionbox{function after multiplying $\mathbbm{1}$}{\includegraphics[width=0.325\textwidth]{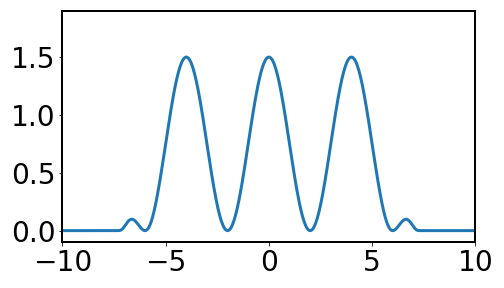}}
	\caption{Illustrating the role of a mollifier:
		(a) A function $f(\cdot)$ defined on $\mathbb{R}$, to be mollified. 
		(b) Here $\mathbbm{1}$ is mollifier that is equal to $1$ on $\mathcal{X} = (-5,5)$ and equal to $0$ outside $[-S,S]$ with $S = 7.5$.  
		(c) The product $f \cdot \mathbbm{1}$ is equal to $f$ on $\mathcal{X}$ and has a compact support on $[-S,S]$.}
	\label{fig: morifier_eg}
\end{figure}

\begin{assumption}[Discretisation of $R$] \label{asmp: disc_mod}
	The cubature nodes $ \{ \bm{x}_j \}_{j=1}^{D} $ are a regular grid on $[-S,S]^d$, corresponding to a Cartesian product of left endpoint rules \citep[(2.1.2) of][]{Davis1984}, and the cubature weights are $u_j := (2 S)^d \mathbbm{1}(\bm{x}_j) / D$ for all $ j = 1, ..., D $.
\end{assumption}

\noindent The use of a Cartesian grid cubature method is not crucial to our analysis and other methods could be considered.
However, we found that the relatively weak assumptions required to ensure convergence of a cubature method based on a Cartesian grid were convenient to check. 
One could consider adapting the cubature method to exploit additional smoothness that may be present in the integrand, but we did not pursue that in this work.

\begin{assumption}[Discretisation of $R_\sigma^*$] \label{asmp: disc_mod_2}
	The cubature nodes $ \{ \bm{w}_i, b_i \}_{i=1}^{N} $ are independently sampled from $\mathcal{N}(0, \sigma_{\bm{w}}^2 \bm{I}_{d \times d}) \times \mathcal{N}(0, \sigma_{b}^2)$ and $v_i := Z / N$ for all $ i = 1, \dots , N $.
\end{assumption}

\noindent The use of a Monte Carlo cubature scheme in \Cref{asmp: disc_mod_2} ensures an exact identification between the ridgelet prior, where the $(\bm{w}_i,b_i)$ are \textit{a priori} independent, and a cubature method.
From the approximation theoretic viewpoint, $I_{\sigma,D,N}$ is now a random operator and we must account for this randomness in our analysis; high probability bounds are used to this effect in \Cref{thm: recon_rate}. 

Now we present out main result, which is a finite-sample-size error bound for the ridgelet prior as an approximation of a GP.
The approximation error that we study is defined over $\mathcal{X} \subset [-S,S]^d$ and to this end we introduce the notation 
\begin{align*}
M_1^{*}(m) := \max_{|\bm{\alpha}| \leq 1} \sup_{\bm{x} \in [-S,S]^d} |\partial^{\bm{\alpha}} m(\bm{x})| , \qquad M_1^{*}(k) := \max_{|\bm{\alpha}| , |\bm{\beta}| \leq 1} \sup_{\bm{x}, \bm{y} \in [-S,S]^d} |\partial_{\bm{x}}^{\bm{\alpha}} \partial_{\bm{y}}^{\bm{\beta}} k(\bm{x}, \bm{y})| .    
\end{align*}
In what follows we introduce a random variable $f$ to represent the GP, and we assume that $f$ is independent of the random variables $\{\bm{w}_i,b_i\}_{i=1}^N$; all random variables are understood to be defined on a common underlying probability space whose expectation operator is denoted $\mathbb{E}$.

\begin{theorem} \label{thm: recon_rate}
	Consider a stochastic process $ f \sim \mathcal{GP}(m, k) $ with mean function $ m \in C^1(\mathbb{R}^d) $ and symmetric positive definite covariance function $ k \in C^{1 \times 1}(\mathbb{R}^d \times \mathbb{R}^d) $. 
	Let Assumptions \ref{assm: ridgelet pair}, \ref{asmp: disc_mod} and \ref{asmp: disc_mod_2} hold.
	Further, assume $ \phi $ is $ L_{\phi} $-Lipschitz continuous.
	With probability at least $ 1 - \delta $, 
	\begin{align*}
	& \sup_{\bm{x} \in \mathcal{X}} \sqrt{ \mathbb{E} \left[ \left( f(\bm{x}) - I_{\sigma,D,N} f (\bm{x}) \right)^2 | \{\bm{w}_i,b_i\}_{i=1}^N \right] } \\
	& \le C \left( M_1^{*}(m) + \sqrt{M_1^{*}(k)} \right) \left\{ \frac{1}{\sigma_{\bm{w}}} + \frac{\sigma_{\bm{w}}^{d} (\sigma_{\bm{w}} + 1)}{\sigma_{b}} + \frac{\sigma_b \sigma_{\bm{w}}^{d} (\sigma_{\bm{w}} + 1)}{D^{\frac{1}{d}}} + \frac{\sigma_b \sigma_{\bm{w}}^{d} (\sigma_{\bm{w}} + \sqrt{\log \delta^{-1}})}{\sqrt{N}} \right\} 
	\end{align*}
	where $ C $ is a constant independent of $m, k,  \sigma_{\bm{w}}, \sigma_b, D, N$ and $\delta$.
\end{theorem}
\begin{proof}
	The proof is established in \Cref{sec: proof_disc_rate}.
\end{proof}

\begin{remark}
	\Cref{thm: recon_rate} is analogous to universal approximation theorems for (non-Bayesian) neural networks \citep{Cybenko1989, Leshno1993b}, since it implies that shallow BNNs are capable of approximating essentially any GP whose mean and covariance functions are continuously differentiable.
\end{remark}

This result establishes that the root-mean-square error between the GP and its approximating BNN will vanish uniformly over the bounded set $\mathcal{X}$ as the bandwidth parameters $\sigma_{\bm{w}}$, $\sigma_b$ and the sizes $N$, $D$ of the cubature rules are increased at appropriate relative rates, which can be read off from the bound in \Cref{thm: recon_rate}.
In particular, the bandwidths $\sigma_{\bm{w}}$ and $\sigma_b$ should be increased in such a manner that $\sigma_{\bm{w}}^{d} (\sigma_{\bm{w}} + 1) / \sigma_b \rightarrow 0$ and the numbers $D$ and $N$ of cubature nodes should be increased in such a manner that the final two terms in the bound of \Cref{thm: recon_rate} asymptotically vanish.
In such a limit, the ridgelet prior provides a consistent approximation of the GP prior over the bounded set $\mathcal{X}$.
The result holds with high probability ($1 - \delta$) with respect to randomness in the sampling of the weights and biases $\{\bm{w}_i,b_i\}_{i=1}^N$. 
From this concentration inequality, and in an appropriate limit of $\sigma_{\bm{w}}$, $\sigma_b$, $N$ and $D$, the almost sure convergence of the root-mean-square error also follows (from Borel-Cantelli).

This finite-sample-size error bound can be contrasted with earlier work that provided only asymptotic convergence results \citep{Neal1995, Lee2018, Matthews2018}.

\begin{remark}
	The rates of convergence serve as a proof-of-concept and are the usual rates that one would expect in the absence of anisotropy assumptions, with the curse of dimension appearing in the term $D^{-1/d}$.
	In other words, we are gated by the rate at which we fill the volume $[-S,S]^d$ with cubature nodes $\{\bm{x}_j\}_{j=1}^D$.
	However, it may be the case that real applications posess additional regularity, such as a low ``effective dimension'', that was not part of the assumptions used to obtain our bound.
	We therefore defer further discussion of the approximation error to the empirical assessment in \Cref{sec:experiments}.
\end{remark}

\begin{remark}
	Related result can also be obtained when the boundedness assumption on the activation function is relaxed from $\phi \in C_0^*(\mathbb{R})$ to $\phi \in C_1^*(\mathbb{R})$.
	Details are reserved for \Cref{sec: proof_lemma_unbounded}.
\end{remark}

This completes our analysis of the ridgelet prior for BNNs and our attention turns, in the next section, to the empirical performance of the ridgelet prior.


\section{Implementation and Computation}
\label{sec:computation}

In this section, practical aspects of the ridgelet prior are discussed.
Since our main contributions are theoretical, this discussion is limited to a factual description of the complexity of the ridgelet prior (\Cref{subsec: prior comput}), a description of how computation was performed for the experiments reported in \Cref{sec:experiments} (\Cref{sec: mcmc_bnn}), and a brief discussion of the potential for further computational improvement (\Cref{subsec: variational spec}).

\subsection{Complexity of the Ridgelet Prior}
\label{subsec: prior comput}

To evaluate the ridgelet prior density, consider the $l$th layer and its associated Gaussian factor $\mathcal{N}(0, \bm{\Sigma}^l)$.
To evaluate the prior density three computations are required: matrix multiplication to compute $\bm{\Sigma}^l = \bm{\Psi}^{l-1} \bm{K} (\bm{\Psi}^{l-1})^\top$, inversion of $\bm{\Sigma}^l$ and computation of the determinant of $\bm{\Sigma}^l$.
The total computational complexity is therefore $O(N^l D^2 + (N^l)^2 D + (N^l)^3)$, which is cubic in $N^l$.
However, to sample from $\mathcal{N}(0, \bm{\Sigma}^l)$ the computation complexity can be reduced when $D \ll N^l$, since we can obtain a matrix square root of the form $( \bm{\Sigma}^l )^{1/2} := \bm{\Psi}^{l-1} \bm{L}$ via a Cholesky decomposition $\bm{K} = \bm{L} \bm{L}^\top$.
The complexity is then $O(D^3 + N^l D)$, which is linear in $N^l$.

\subsection{Posterior Computation} \label{sec: mcmc_bnn}

For the subsequent experiments reported in \Cref{sec: exp_02}, the posterior distribution over the parameters of the BNN is analytically intractable and must therefore be approximated.
In this paper, we exploited conditional linearity to obtain accurate approximations to the posterior using a Monte Carlo method, which will now be described.
All our experiments concern regression models of the form
\begin{align}
y^{(i)} = f(\bm{x}^{(i)}, \bm{\theta})+ \epsilon^{(i)}, \qquad f(\bm{x}, \bm{\theta}) = \sum_{i=1}^N w_i^{1} \phi(\bm{w}_i^{0} \cdot \bm{x} + b_i^{0}), \qquad \epsilon^{(i)} \sim \mathcal{N}(0, \sigma_{\epsilon}^2). \label{eq: linear reg model}
\end{align}
Letting $\bm{\theta}^{0} := ( (\bm{w}_1^{0}, b_1^{0}), \dots, (\bm{w}_N^{0}, b_N^{0}) )$ and $\bm{w}^{1} := (w_1^{1}, \dots, w_N^{1})$, the ridgelet prior takes the form
\begin{align*}
\bm{w}_i^{0} \sim \mathcal{N}(0,\sigma_{\bm{w}}^2 I_{d \times d}), \qquad b_i^{0} \sim \mathcal{N}(0,\sigma_{b}^2), \qquad \bm{w}^{1} \mid \bm{\theta}^{0} \sim \mathcal{N}(\bm{0}, \bm{\Sigma}) ,
\end{align*}
where $\bm{\Sigma} := \bm{\Psi}_0 \bm{K} \bm{\Psi}_0^\top$.
To simplify this discussion, we assume the mean function $m$ of the target GP is $m(\bm{x}) = 0$; if not then one can replace the responses $y^{(i)}$ with $y^{(i)} - m(\bm{x}^{(i)})$ in the sequel.
Let $[\bm{y}]_i = y^{(i)}$ and $[\bm{\Phi}_0]_{i,j} := \phi(\bm{w}_j^{0} \cdot \bm{x}^{(i)} + b_j^{0})$.
Our main observation is that, conditional on $\bm{\theta}^{0}$, the model \eqref{eq: linear reg model} is a linear regression whose coefficients $\bm{w}^{1}$ can be analytically marginalised to obtain a Gaussian marginal (log-)likelihood
\begin{align}
\log p( \bm{y} \mid \bm{\theta}^{0}) = C - \frac{1}{2} \log\det \bm{\Sigma}_{*} - \frac{1}{2} \bm{y}^{\top} \bm{\Sigma}_{*}^{-1} \bm{y} \nonumber
\end{align}
where $\bm{\Sigma}_{*} := \bm{\Phi}_0 \bm{\Sigma} \bm{\Phi}_0^{\top} + \sigma_{\epsilon}^2 I$ and $C$ is a constant with respect to $\bm{\theta}^{0}$.
This enables the use of MCMC to sample directly from the marginal posterior distribution $p( \bm{\theta}^{0} \mid \bm{y} )$, and for this purpose we employed the \emph{elliptical slice sampler} of \cite{Murray2010}.
If the number of data $\bm{y}$ is $M$, the computational complexity for marginal likelihood evaluation is $\mathcal{O}( N D^2 + N^2 D + N M^2 + N^2 M + M^3)$, which is quadratic with respect to $N$.
Furthermore, we recognise that the posterior distribution of $\bm{w}^{1}$ conditional on $\bm{\theta}^{0}$ are Gaussian, meaning that given each sample of $\bm{\theta}^{0}$ and data $\bm{y}$ we can simulate a corresponding sample:
\begin{align}
\bm{w}^{1} \mid \bm{\theta}^{0}, \bm{y} \sim \mathcal{N}(\bm{m}_{**}, \bm{\Sigma}_{**}) \nonumber
\end{align}
where $\bm{m}_{**} := \sigma_\epsilon^{-2} \bm{\Sigma}_{**} \bm{\Phi}_0^\top \bm{y}$ and $\bm{\Sigma}_{**}^{-1} := \bm{\Sigma}^{-1} + \sigma_\epsilon^{-2} \bm{\Phi}_0^\top \bm{\Phi}_0$.
Standard techniques were applied to regularise computation in cases where matrices were poorly conditioned; see \Cref{sec: experiment_apdx}.
This approach produces samples that are asymptotically drawn from the joint posterior $( \bm{\theta}^0, \bm{w}^1 ) | \bm{y}$.
The complexity of this conditional sampling step is $\mathcal{O}( N^3 + N^2 M + N M )$, which is cubic with respect to $N$.
This two-stage procedure was observed to work effectively in the experiments that we performed.
However, this approach does not extend to deeper BNNs and therefore more general-purpose MCMC or variational inference techniques are likely to be required.
The potential for variational techniques to be employed with the ridgelet prior is discussed next.

\subsection{Variational Inference}
\label{subsec: variational spec}

Although the algorithm in \Cref{sec: mcmc_bnn} was observed to perform well for the experiments reported in \Cref{sec:experiments}, the limited applicability to shallow networks suggests that an alternative approach to posterior approximation is needed.
As we explain in this section, variational inference may be well-suited to this task.
In variational inference, a posterior distribution $p(\bm{\theta} \mid \bm{y})$ is approximated by a distribution $q$ selected from a candidate set $\mathcal{Q}$ via maximisation of an evidence lower bound (ELBO):
\begin{align*}
\max_{q \in \mathcal{Q}}\ \mathbb{E}_{\bm{\theta} \sim q}[ \log p(\bm{y} \mid \bm{\theta}) ] - \operatorname{KL}( q \| \pi )
\end{align*}
where $\pi$ is a prior and $p(\bm{y} \mid \bm{\theta})$ is the likelihood \citep{Sun2019}.
Note that $\bm{\theta}$ denotes the set of
all parameters of the form $\bm{w}_i^l$, $b_i^l$ involved in the BNN.
The term $\operatorname{KL}( q \| \pi )$ will not available in closed form in general because of the complicated \textit{a priori} dependency among the parameters in $\bm{\theta}$.
However, closer inspection of the ridgelet prior reveals that each layer-wise component of $\bm{\theta}$ is conditionally Gaussian given the values of the remaining components; this conditional Gaussian form can be used to circumvent intractability of the Kullback--Leibler term in the ELBO.
This suggests a natural approach to optimisation, which cycles through layer-wise subsets of $\bm{\theta}$, maximising the ELBO with respect to each layer-wise component.
Although the design and assessment of a variational inference procedure is not a focus of this paper, the possibility to pursue an alternating optimisation procedure demonstrates that there may be room to pursue improved computational methodology for the ridgelet prior if needed.


\section{Empirical Assessment}
\label{sec:experiments}

In this section we briefly report empirical results that are intended as a proof-of-concept.
Our aims in this section are twofold:
First, in \Cref{sec: exp_01} we seek to illustrate the theoretical analysis of \Cref{sec:theory} and to explore how well a BNN with a ridgelet prior approximates its intended GP target.
Second, in \Cref{sec: exp_02} we aim to establish whether use of the ridgelet prior for a BNN in a Bayesian inferential context confers some of the same characteristics as when the target GP prior is used, in terms of the inferences that are obtained.

Throughout this section we focus on BNNs with a single hidden layer ($L=1$).
The settings for the ridgelet prior were identical to the hyperbolic tangent activation pair $(\phi, \psi)$ in \Cref{ex: ridgelet_func} for all experiments reported in the main text; sensitivity to these choices examined in \Cref{sec: comp_exp}.
Our principal interest is in how many hidden units, $N$, are required in order for the ridgelet prior to provide an adequate approximation of the associated GP, since the size of the network determines the computational complexity of performing inference with the ridgelet prior; see \Cref{sec:computation}.
To this end, we fix $D$, $\sigma_{\bm{w}}$ and $\sigma_b$ and investigate the influence of $N$.
The fixed values of the bandwidths $\sigma_{\bm{w}}$ and $\sigma_{b}$ were selected by analysing the approximation quality, as described in \Cref{sec: comp_exp}.
For completeness, \Cref{tbl: quad_list} in \Cref{app: experimental settings} reports the values of $D$, $\sigma_{\bm{w}}$ and $\sigma_b$ that were used for each experiment.

\subsection{Approximation Quality} \label{sec: exp_01}

The aim of this section is to explore how well a BNN with a ridgelet prior approximates its intended GP target.
For this experiment we consider dimension $d = 1$ and the so-called \emph{squared exponential} covariance function $k(x, x') := l^2 \exp( - \frac{1}{2 s^2} | x - x' |^2  )$, where $ l = 1.0$, $s = 1.5$ were chosen to ensure that samples from the GP demonstrate non-trivial behaviour on our domain of interest $\mathcal{X} = (-5,5)$.

\Cref{fig: exp_prior}, introduced earlier in the paper, presents sample paths from a BNN equipped with the ridgelet prior as the number $N$ of hidden units is varied from 100 to 3000.
In the case $N=3000$ the sample paths are almost indistinguishable from those of the target GP.
On the other hand, $N = 3000$ may be a rather large number of units to require, depending on the applied context. 
\Cref{fig: exp_prior_error} explores the approximation quality of the discretised ridgelet transform construction, reporting  the \emph{maximum root-mean-square error} (MRMSE) 
\begin{eqnarray*}
	\sup_{\bm{x} \in \mathcal{X}} \sqrt{ \mathbb{E}[ ( f(\bm{x}) - I_{\sigma, D, N} f (\bm{x}) )^2 \mid \{ \bm{w}_i, b_i \}_{i=1}^{N} ] }.
\end{eqnarray*}
It was observed that the approximation error between the BNN and the GP decays at a slow rate, consistent with our theoretical error bound $O(N^{-\frac{1}{2}})$.
To explore how well the second moments of the GP are being approximated, in \Cref{fig: exp_prior_cov} we compared the BNN covariance function $\mathbb{E}[f(x) f(0)]$ as in \eqref{eq: one layer ridgelet cov} against the covariance function $k(x, 0)$ of the target GP.
As $N$ is varied we observe convergence of the BNN covariance function to the GP covariance function. 
For accurate approximation, a large number of hidden units appears to be required.
Additional results, detailing the effect of varying $D$, $\sigma_{\bm{w}}$, $\sigma_{b}$, activation function $\phi$, and GP covariance $k$ are provided in \Cref{sec: comp_exp}.

The MRMSE is a strong assessment that acknowledges the paired nature of $f$ and $I_{\sigma,D,N}f$, in contrast to a more standard comparison of the distributions produced by the BNN and the GP.
However, as an alternative criteria to MRMSE, we also computed the maximum mean discrepancy (MMD) between the BNN and the GP; these details and results are reserved for \Cref{app: MMD_error}.

\begin{figure}[t!]
	\centering
	\subcaptionbox{approximation error\label{fig: exp_prior_error}}{\includegraphics[height=100pt]{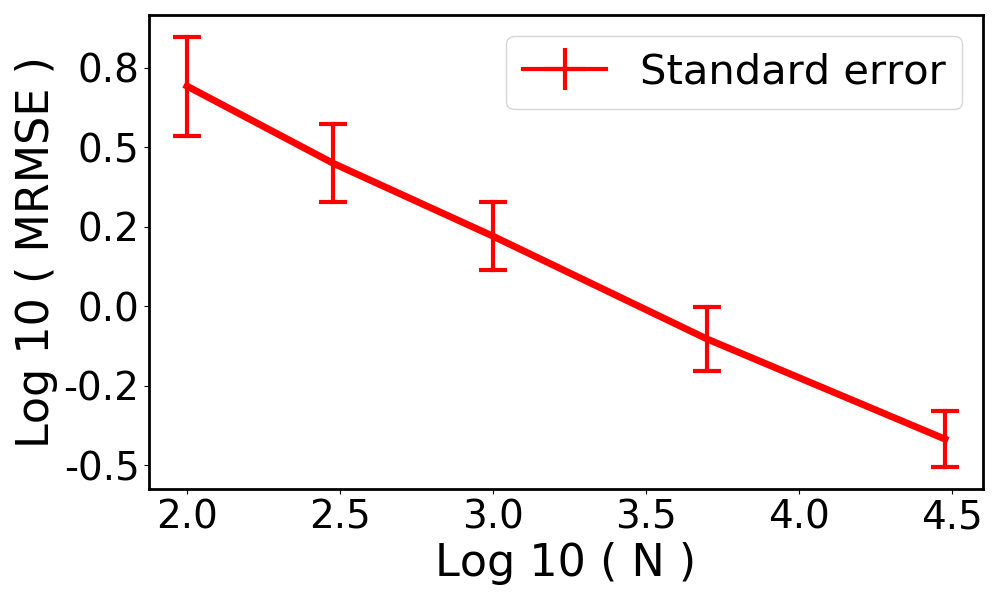}}
	\hfill
	\subcaptionbox{BNN covariance\label{fig: exp_prior_cov}}{\includegraphics[height=100pt]{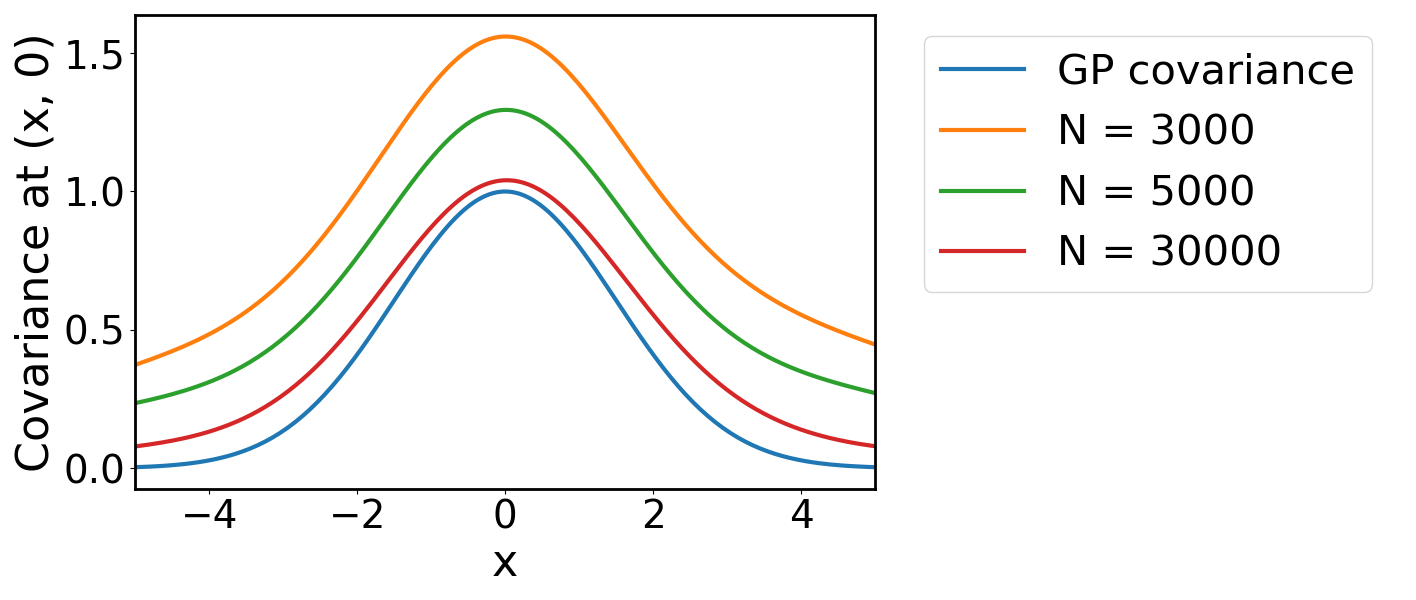}}
	\caption{Approximation quality of the ridgelet prior:
		(a) MRMSE for the BNN approximation of the GP was estimated as the number $N$ of hidden units was varied with standard errors of 100 independent experiments displayed.
		(b) The covariance associated to the BNN, as $N$ is varied, together with the covariance of the GP being approximated.
	}
	\label{fig:exp_01}
\end{figure}

\subsection{Inference and Prediction Using the Ridgelet Prior} \label{sec: exp_02}

In this section we compare the performance of the ridgelet prior to that of its target GP in an inferential context.
To this end, we identified tasks where non-trivial prior information is available and can be encoded into a covariance model; the ridgelet prior is then used to approximate this covariance model using a BNN.
Three tasks were considered:
(i) prediction of atmospheric CO$_2$ concentration using the well-known Mauna Loa dataset; (ii) prediction of airline passenger numbers using a historical dataset; (iii) a simple in-painting task that is closer in spirit to applications where BNNs may be used.
These tasks are toy in their nature and we do not attempt an empirical investigation into the practical value of the ridgelet prior; this will be reserved for a sequel.

\paragraph{Prediction of Atmospheric CO$_2$:}

In this first task we extracted 43 consecutive months of atmospheric CO$_2$ concentration data recorded at the Mauna Loa observatory in Hawaii \citep{Keeling2004} and considered the task of predicting CO$_2$ concentration up to 23 months ahead.
The time period under discussion was re-scaled onto interval $\mathcal{X} = (-5,5)$ and the response variable, representing CO$_2$ concentration, was standardised.
This task is well-suited to regression models, such as GPs, that are able to account for non-trivial prior information.
In particular, one may seek to encode (1) a linear trend toward increasing concentration of atmospheric CO$_2$ and (2) a periodic seasonal trend.
In the GP framework, this can be achieved using a mean function $m$ and covariance function $k$ of the form
\begin{eqnarray}
m(x) := a x, \qquad k(x, x') := l^2 \exp\left( - \frac{2}{s^2} \sin\left( \frac{\pi}{p^2} | x - x' | \right)^2 \right) \label{eq: exp_2_u_k}
\end{eqnarray}
where $a = 0.06$, $l = 1.0$, $s = 0.75$, $p = 1.8$; see \cite[][Section 5.4.3]{Rasmussen2006}.
The aim of this experiment is to explore whether the posterior predictive distribution obtained using the ridgelet prior for a BNN is similar to that which would be obtained using this GP prior.
To this end, we employed a simple Gaussian likelihood
\begin{align*}
y_i = f(x_i, \theta) + \epsilon_i, \qquad \epsilon_i \stackrel{\text{i.i.d.}}{\sim} \mathcal{N}(0, \sigma_{\epsilon}^2) 
\end{align*}
where $y_i$ is the standardised CO$_2$ concentration in month $i$, $x_i$ is the standardised time corresponding to month $i$, $x \mapsto f(x,\theta)$ is the regression model and $\sigma_{\epsilon} = 0.065$ was fixed.
The posterior GP is available in closed form, while we used MCMC to sample from the posterior distribution over the parameters $\theta$ when a neural network is used, as described in \Cref{sec: mcmc_bnn}.

\begin{figure}[t!]
	\centering
	\subcaptionbox{i.i.d. prior}{\includegraphics[width=0.325\textwidth]{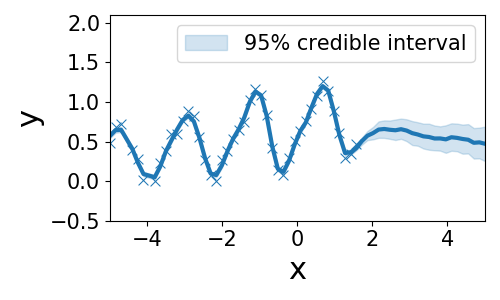}}
	\hfill
	\subcaptionbox{the ridgelet prior}{\includegraphics[width=0.325\textwidth]{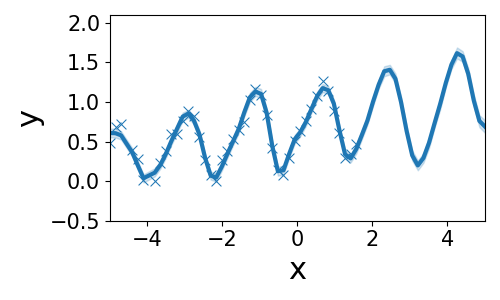}}
	\hfill
	\subcaptionbox{GP}{\includegraphics[width=0.325\textwidth]{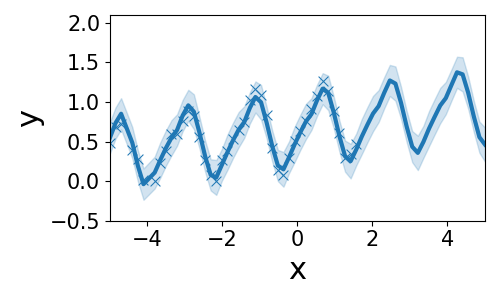}}
	\caption{Prediction of Atmospheric CO$_2$: Posterior predictive distributions were obtained using (a) a Bayesian neural network (BNN) with an i.i.d. prior, (b) a BNN with the ridgelet prior, and (c) a Gaussian process (GP) prior.
		The ridgelet prior in (b) was designed to approximate the GP prior used in (c).
		The solid blue curve is the posterior mean and the shaded regions represent the pointwise 95\% credible interval.}
	\label{fig: exp_co2}
\end{figure}

Here we present results for BNN with $N = 500$ hidden units.
Our benchmark is the so-called \textit{i.i.d. prior} that takes all parameters to be \textit{a priori} independent with $w_{i,j}^{0} \sim \mathcal{N}(0, \sigma_{w^{0}} )$, $b_i^{0} \sim \mathcal{N}(0, \sigma_{b^{0}})$, and $w_{1,i}^1 \sim \mathcal{N}(0, \sigma_{w^1})$.
Here $\sigma_{w^{0}}=3$, $\sigma_{b^{0}}=12$ and $\sigma_{w^1} = 0.1/\sqrt{N}$.
This is to be contrasted with the ridgelet prior, with values of $\sigma_{\bm{w}} = 3$ and $\sigma_b = 12$ chosen to ensure a fair comparison with the i.i.d. prior.
\Cref{fig: exp_co2} displays the posterior predictive distributions obtained using (a) the i.i.d. prior, (b) the ridgelet prior, and (c) the original GP.
It can be seen that (b) and (c) are more similar than (a) and (c).
These results suggest that the ridgelet prior is able to encode non-trivial information on the seasonality of CO$_2$ concentration into the prior distribution for the parameters of the BNN.

\paragraph{Prediction of Airline Passenger Numbers:}

Next we considered a slightly more challenging example that involves a more intricate periodic trend.
The dataset here is a subset of the airline passenger dataset studied in \cite{Pearce2019}.
The response $y_i$ represents the (standardised) monthly total number of international airline passengers in the United States and the input $x_i$ represents the (standardised) time corresponding to month $i$.
The experimental set-up was identical to the CO$_2$ experiment, but the prior in \eqref{eq: exp_2_u_k} was modified to reflect the more intricate periodic trend by taking $a = 0.2$, $l = 1.0$, $s = 0.75$, $p = 1.75$ and the measurement noise was fixed to $\sigma_{\epsilon} = 0.145$.

In a similar manner to the CO$_2$ experiment, we implemented the i.i.d. prior and the analogous ridgelet prior, in the latter case based on $D = 200$ quadrature points.
Results in \Cref{fig: exp_airline} support the conclusion that the ridgelet prior is able to capture this more complicated seasonal trend. In comparison with \cite{Pearce2019} we required more hidden units to perform this regression task.
However, we used a standard activation function, where the method of \cite{Pearce2019} would need to develop a new activation function for each covariance model of interest.

\begin{figure}[t!]
	\centering
	\subcaptionbox{i.i.d. prior}{\includegraphics[width=0.325\textwidth]{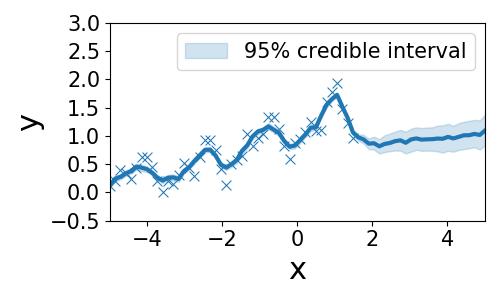}}
	\hfill
	\subcaptionbox{the ridgelet prior}{\includegraphics[width=0.325\textwidth]{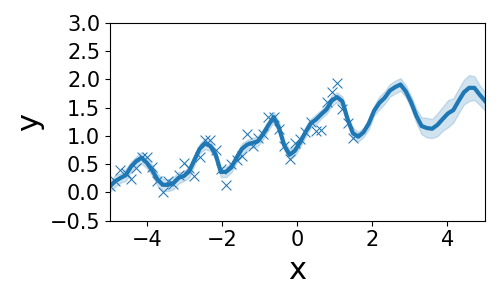}}
	\hfill
	\subcaptionbox{GP}{\includegraphics[width=0.325\textwidth]{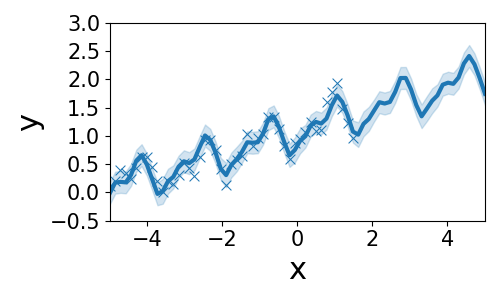}}
	\caption{Prediction of Airline Passenger Numbers: Posterior predictive distributions were obtained using (a) a Bayesian neural network (BNN) with an i.i.d. prior, (b) a BNN with the ridgelet prior, and (c) a Gaussian process (GP) prior.
		The ridgelet prior in (b) was designed to approximate the GP prior used in (c).
		The solid blue curve is the posterior mean and the shaded regions represent the pointwise 95\% credible interval.}
	\label{fig: exp_airline}
\end{figure}

\paragraph{In-Painting Task:}

Our final example is a so-called \textit{in-painting} task, where we are required to infer a missing part of an image from the remaining part.
Our aim is not to assess the suitability of the ridgelet prior for such tasks -- to do so would require extensive and challenging empirical investigation, beyond the scope of the present paper -- but rather to validate the ridgelet prior as a proof-of-concept.

The image that we consider is shown in \Cref{fig: exp_inprint_a} and we censor the central part, described by the red square in \Cref{fig: exp_inprint_b}.
Each pixel corresponds to a real value $y_i$ and the location of the pixel is denoted $\bm{x}_i \in (-5,5)^2$.
To the remaining part we add a small amount of i.i.d. noise, $\epsilon_i \sim \mathcal{N}(0, \sigma_{\epsilon}^2)$ to each pixel $i$, in \Cref{fig: exp_inprint_b} (the addition of noise here ensures conditions for ergodicity of MCMC are satisfied; otherwise the posterior is supported on a submanifold of the parameter space and more sophisticated sampling procedures would be required).
The task is then to infer this missing central region using the remaining part of the image as a training dataset.

\begin{figure}[t!]
	\centering
	\subcaptionbox{original image\label{fig: exp_inprint_a}}{\includegraphics[width=0.20\textwidth]{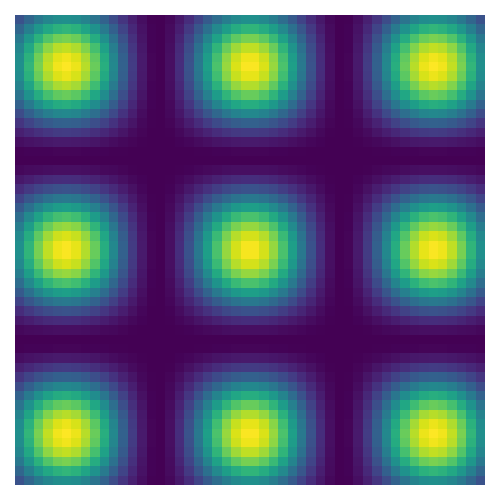}}
	\hfill
	\subcaptionbox{training dataset\label{fig: exp_inprint_b}}{\includegraphics[width=0.20\textwidth]{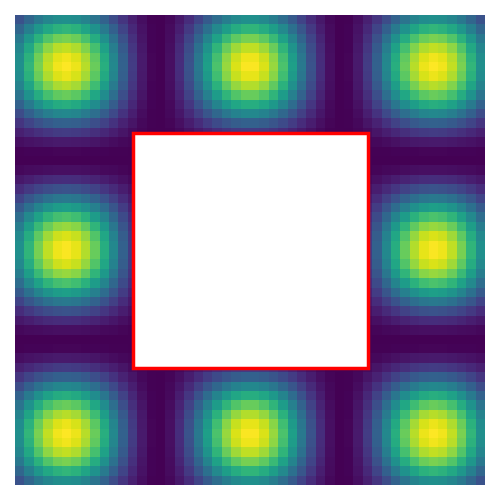}}
	\hfill
	\subcaptionbox{i.i.d. prior\label{fig: exp_inprint_c}}{\includegraphics[width=0.20\textwidth]{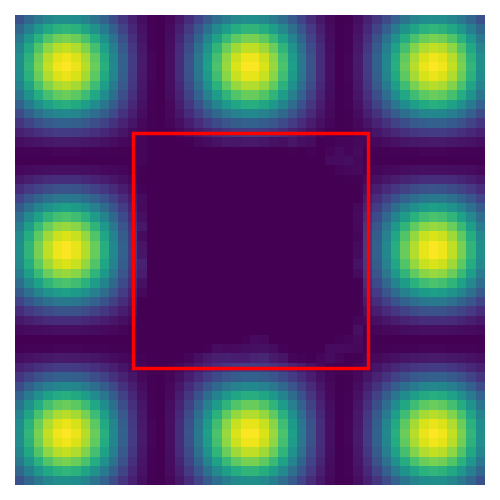}}
	\hfill
	\subcaptionbox{the ridgelet prior\label{fig: exp_inprint_d}}{\includegraphics[width=0.20\textwidth]{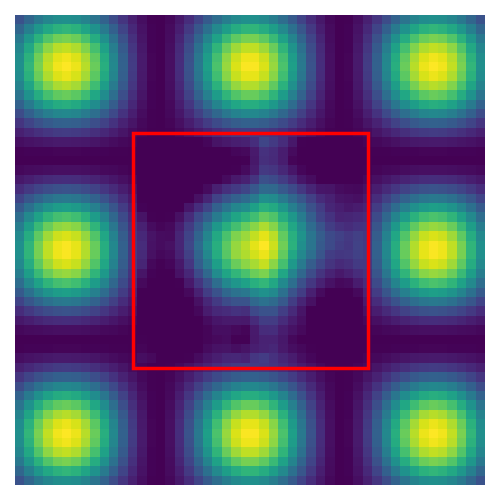}}
	\caption{In-Painting Task: 
		The central part of the original image in (a) was censored to produce (b) and the in-painting task is to infer the missing part of the image using the remaining part as a training dataset.
		Posterior predictive distributions were obtained using (c) a Bayesian neural network (BNN) with an i.i.d. prior, and (d) a BNN with the ridgelet prior.
	}
\end{figure}

For the statistical regression model we considered a GP whose mean function $m$ is zero and covariance function $k$ is
\begin{eqnarray}
k(\bm{x}, \bm{x}') & := & \textstyle l^2 \exp\left( - \frac{\| \bm{x} - \bm{x}' \|_2^2}{2 s^2} \right) +  \label{eq: exp_3_u_k} \\
& & \textstyle \Big( \cos\Big( \frac{\pi x_1}{2} \Big) + 1 \Big) \Big( \cos\Big( \frac{\pi x_2}{2} \Big) + 1 \Big) \Big( \cos\Big( \frac{\pi x_1'}{2} \Big) + 1 \Big) \Big( \cos\Big( \frac{\pi x_2'}{2} \Big) + 1 \Big), \nonumber  
\end{eqnarray}
where $l = 0.1$, $s = 0,1$.
The periodic structure induced by cosine functions is deliberately chosen to be commensurate with the separation between modes in the original dataset, meaning that considerable prior information is provided in the GP covariance model.
Our benchmark here was a BNN equipped with with $N = 2000$ hidden units and an i.i.d. prior, constructed in an analogous manner as before with parameters $\sigma_{w^0} = 2$, $\sigma_{b^0} = 12$ and $\sigma_{w^1} = 0.1/\sqrt{N}$.
The ridgelet prior was based on a BNN of the same size, with $\sigma_{\bm{w}} = 2$ and $\sigma_b = 12$ to ensure a fair comparison with the i.i.d. prior.

\Cref{fig: exp_inprint_c} and \Cref{fig: exp_inprint_d} show posterior means estimates for the missing region in the in-painting task, respectively for the i.i.d. prior and the ridgelet prior.
It is interesting to observe that the i.i.d. prior gives rise to a posterior mean that is approximately constant, whereas the ridgelet prior gives rise to a posterior mean that somewhat resembles the original image.
This suggests that the specific structure of the covariance model in \eqref{eq: exp_3_u_k} has been faithfully encoded into the ridgelet prior, leading to superior performance on this stylised in-painting task.

\subsection{Limitations of the Ridgelet Prior for Deep Networks} \label{sec: exp_03}

Finally, in this section we examine the performance of the ridgelet prior in a setting where the depth of the network is increased.
For this experiment, we fixed the hidden unit number $N_l$ of each hidden layer $l$ to $N_l = 10,000$ and increased the number of hidden layers from $L = 1$ to $L = 5$.
The approximation error between a BNN with the ridgelet prior and the target GP was quantified using MMD, aforementioned in \Cref{sec: exp_01}: see \Cref{app: MMD_error} for detail.
Results are displayed in \Cref{fig: deep_prior}.
Interestingly, the lowest value of MMD was observed at $L = 1$, suggesting that the performance of the ridgelet prior deteriorates when the depth of the network is increased.
This makes intuitive sense, since with each additional layer in the network an additional discretisation of the ridgelet transform is required.
This discretisation error then has an opportunity to accumulate and propagate through the network.
It would be interesting to explore strategies to mitigate this degradation of performance, but our focus in this work was limited to providing theoretical analysis and an empirical proof-of-concept.

\begin{figure}[t!]
	\centering
	\begin{minipage}{0.675\textwidth}
		\subcaptionbox{BNN ($L = 1$)}{\includegraphics[width=0.325\textwidth]{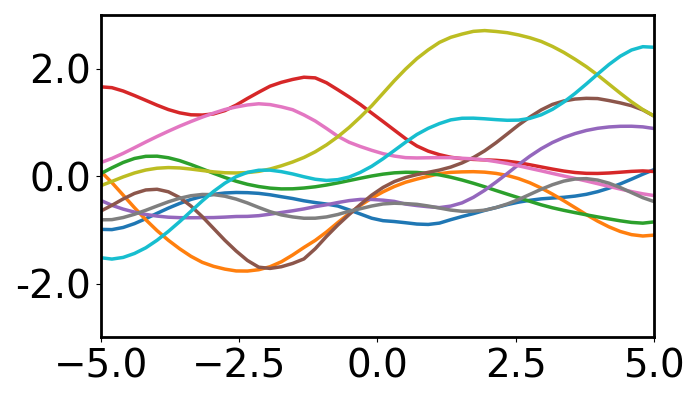}}
		\subcaptionbox{BNN ($L = 2$)}{\includegraphics[width=0.325\textwidth]{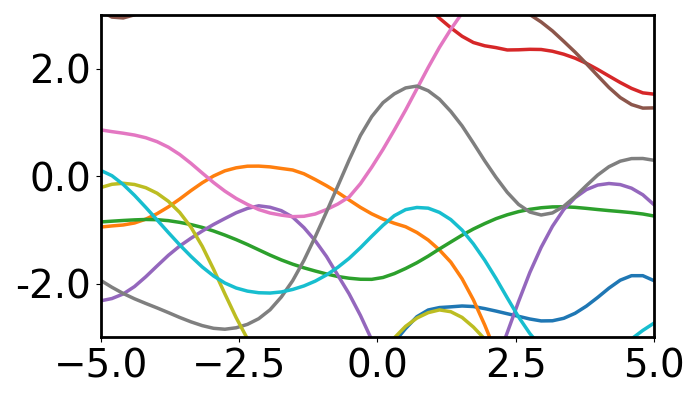}}
		\subcaptionbox{BNN ($L = 3$)}{\includegraphics[width=0.325\textwidth]{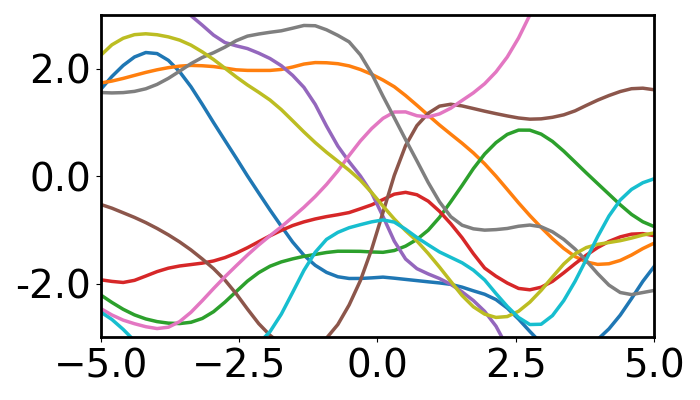}}
		
		\subcaptionbox{BNN ($L = 4$)}{\includegraphics[width=0.325\textwidth]{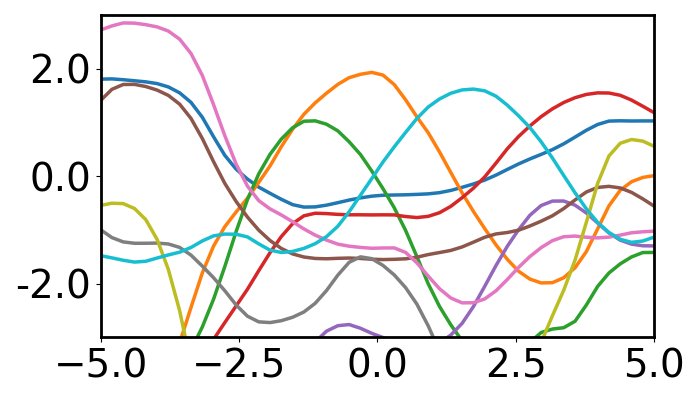}}
		\subcaptionbox{BNN ($L = 5$)}{\includegraphics[width=0.325\textwidth]{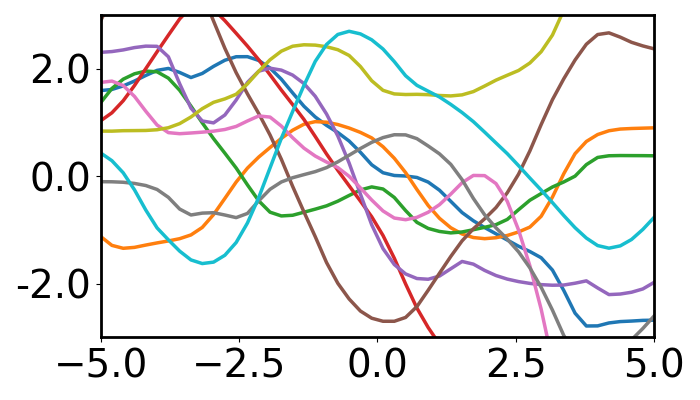}}
		\subcaptionbox{GP}{\includegraphics[width=0.325\textwidth]{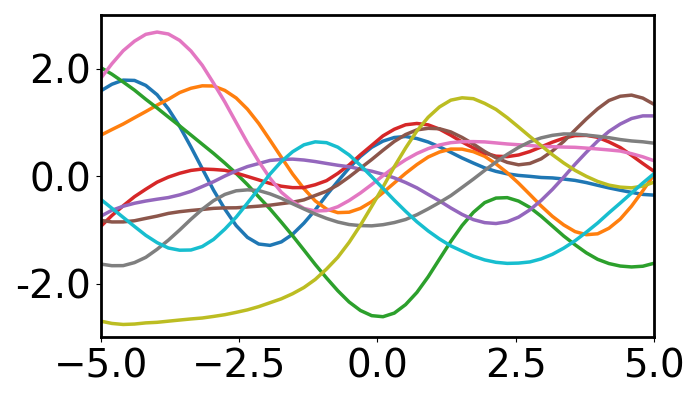}}
	\end{minipage}
	\hfill
	\begin{minipage}{0.3\textwidth}
		\subcaptionbox{MMD$^2$ ($L=1, \dots, 5$)}{\includegraphics[width=\textwidth]{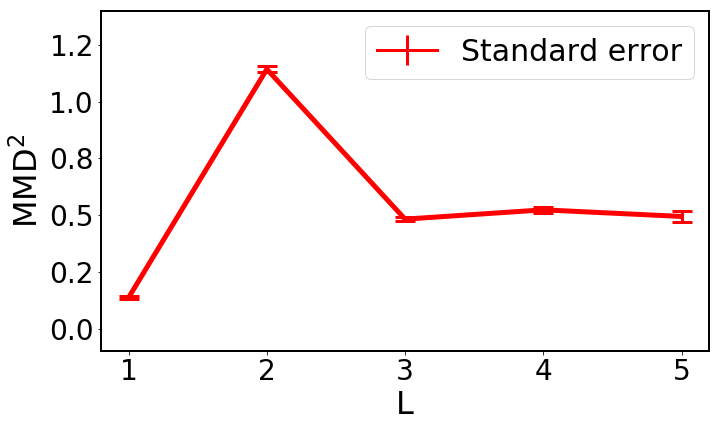}}
	\end{minipage}
	\caption{
		Sample paths from Bayesian neural networks (BNNs) equipped with the ridgelet prior as the number of hidden layers $L$ is increased. 
		The number of hidden units was fixed to $N_l=10,000$ for each hidden layer $l = 1, \dots, L$.
		In (f) sample paths are shown from the target Gaussian process (GP) prior.
		In (g) the MMD$^2$ between the BNN and the GP are displayed.
		The standard error is calculated by 10 independent computations.
	}
	\label{fig: deep_prior}
\end{figure}


\section{Conclusion}
\label{sec:conclusion}

One of the main barriers to the wide-spread adoption of BNN is the identification of prior distributions that are meaningful when lifted to the output space of the network.
In this paper it was shown that the ridgelet transform facilitates the consistent approximation of a GP using a BNN.
This has the potential to bring the powerful framework of covariance modelling for GPs to bear on the task of prior specification for BNN. 
In contrast to earlier work in this direction \citep{Flam-Shepherd2017,Hafner2019,Pearce2019,Sun2019}, our construction is accompanied by theoretical analysis that establishes the approximation is consistent.
Moreover, we are able to provide a finite-sample-size error bound that requires only weak assumptions on the GP covariance model (i.e. that the mean and covariance function is continuous and differentiable).

This role of this paper was to establish the ridgelet prior as a theoretical proof-of-concept only and there remain several open questions to be addressed:
\begin{itemize}
	\item In real applications, is it necessary to have an accurate approximation of the intended covariance model, in order to deliver improved performance of the BNN, or is a crude approximation sufficient?
	\item What cubature rules $\{\bm{x}_i,u_i\}_{i=1}^D$ are most effective in applications - is a regular grid needed, or do real-world applications exhibit an effective low dimension so that e.g. a sparse grid could be used?
	\item In the case of multiple hidden layers, can the convergence of the ridgelet prior to a \textit{deep GP} \citep{Damianou2012,Dunlop2017} be established? 
\end{itemize}
These questions will likely require a substantial amount of work to address in full, but we hope to pursue some of them in a sequel.


\paragraph*{Acknowledgements}
The authors were supported by the Lloyd's Register Foundation Programme on Data-Centric Engineering and the Alan Turing Institute under the EPSRC grant [EP/N510129/1].


\newpage
\appendix
\label{sec:appendix}

\section{Supplementary Material}
This appendix is structured as follows:

\vspace{5pt}
\etocsettocstyle{}{}
\localtableofcontents

\vspace{10pt}
\noindent
\textit{Notation:} Throughout this appendix we adopt identical notation to the main text, but for brevity we denote $\| \cdot \|_{L^1(\mathcal{X})}$ as $\| \cdot \|_{L^1}$ and $\| \cdot \|_{L^\infty(\mathcal{X})}$ as $\| \cdot \|_{L^\infty}$ whenever $\mathcal{X}$ is the Euclidean space $\mathbb{R}^d$ of any dimension $d \in \mathbb{N}$. 
We also introduce function classes of higher order differentiability.
Denote by $C^{r}(\mathcal{X})$ the set of all functions $f : \mathcal{X} \rightarrow \mathbb{R}$ for which the derivatives $\partial^{\bm{\alpha}}f$ exist and are continuous on $\mathcal{X}$ for all $\bm{\alpha} \in \mathbb{N}_{0}^{d}$ s.t. $|\bm{\alpha}| \leq r$.
Denote by $C^{r \times r}(\mathcal{X} \times \mathcal{X})$ the set of all functions $h : \mathcal{X} \times \mathcal{X} \rightarrow \mathbb{R}$ for which the derivatives $\partial^{\bm{\alpha}, \bm{\beta}}h$ exist and are continuous on $\mathcal{X} \times \mathcal{X}$ for all $\bm{\alpha}, \bm{\beta} \in \mathbb{N}_{0}^{d}$ s.t. $|\bm{\alpha}|, |\bm{\beta}| \leq r$.


\subsection{Convergence of the Finite-Bandwidth Ridgelet Transform}
\label{sec: proof_recon_mod}

In this section we derive a finite-sample-size error bound for the finite-bandwidth ridgelet transform of Definition \ref{def: ridgelet_modified}.
Our analysis is more general than the Gaussian case presented in the main text, and we consider more general probability measures on $\bm{w}$ and $b$ to control convergence of $\lambda_\sigma$ toward the improper uniform limit.

Let $p_{\bm{w}}$ and $p_b$ be probability densities, respectively, on $\mathbb{R}^d$ and $\mathbb{R}$, satisfying regularity requirements that will shortly be described.
For $ 0 < \sigma_{\bm{w}} < \infty, 0 < \sigma_b < \infty $, define scaled densities of $ p_{\bm{w}} $ and $ p_b $ by $ p_{\bm{w}, \sigma}(\bm{w}) := \sigma_{\bm{w}}^{-d} p_{\bm{w}}(\sigma_{\bm{w}}^{-1} \bm{w}) $ and $ p_{b, \sigma}(b) := \sigma_b^{-1} p_b(\sigma_b^{-1} b) $.
The parameters $\sigma_{\bm{w}}$ and $\sigma_b$ will be called \emph{bandwidths}.
To recover the case presented in the main text, if $p_{\bm{w}}$ and $p_b$ are standard Gaussians then $p_{\bm{w}, \sigma}(\bm{w})$ and $p_{b, \sigma}(b)$ are Gaussians with variances $\sigma_{\bm{w}}^2 \bm{I}$ and $\sigma_{b}^2$, i.e. 
\begin{align*}
p_{\bm{w}, \sigma}(\bm{w}) = \frac{1}{(2 \pi \sigma_{\bm{w}}^2)^{d / 2} } \exp\left( - \frac{\| \bm{w} \|^2}{2 \sigma_{\bm{w}}^2} \right), \qquad p_{b, \sigma}(b) = \frac{1}{(2 \pi \sigma_{b}^2)^{1/2}} \exp\left( - \frac{b^2}{2 \sigma_{b}^2} \right). 
\end{align*}
Then let $\lambda_\sigma(\bm{w}, b)$ be a measure whose density is $Z p_{\bm{w}, \sigma}(\bm{w}) p_{b, \sigma}(b)$, where for the analysis that follows it will be convenient to set 
\begin{align*}
Z := \frac{ (2 \pi)^{\frac{1}{2}} \sigma_{\bm{w}}^{d} \sigma_b }{ \| \widehat{p_{\bm{w}}} \|_{L^1(\mathbb{R}^d)} \| \widehat{p_b} \|_{L^1(\mathbb{R})} } , 
\end{align*}
indicating the total measure assigned by $\lambda_{\sigma}$ to $\mathbb{R}^{d+1}$.
For the case of standard Gaussian densities $ p_{\bm{w}}$ and $p_{b}$ presented in the main text, we have $ \| \widehat{p_{\bm{w}}} \|_{L^1(\mathbb{R}^d)} = \| \widehat{p_{b}} \|_{L^1(\mathbb{R})} = 1 $. 

Our analysis covers general $p_{\bm{w}}$ and $p_b$ provided that certain regularity conditions are satisfied.
These will now be described.
First, consider an arbitrary probability density function $p$ defined on $\mathbb{R}^m$ for some $m \in \mathbb{N}$ and consider the following properties that $p$ may satisfy: 
\begin{enumerate}[label=(\roman*)]
	\item (Symmetry at the origin) $p(\bm{x}) = p(-\bm{x})$, $\forall \bm{x} \in \mathbb{R}^m$ ,
	\item (Boundedness) $\|p\|_{L^\infty(\mathbb{R}^m)} < \infty$,
	\item (Finite second moment) $\int_{\mathbb{R}^m} \|\bm{x}\|^2 p(\bm{x}) \mathrm{d}\bm{x} < \infty$ ,
	\item (Positivity of the Fourier transform) $\widehat{p}(\bm{x}) > 0$ , $\forall \bm{x} \in \mathbb{R}^m$ ,
	\item (Integrable Fourier transform) $ \| \widehat{p} \|_{L^1(\mathbb{R}^m)} < \infty $.
\end{enumerate}
Of these properties, only (iv) and (v) are not straight forward.
If $\hat{p}$ can be computed then conditions (iv) and (v) can be directly verified.
Otherwise, sufficient conditions on $p$ for (iv) can be found in \cite{Tuck2006} and sufficient conditions on $p$ for (v) can be found in \cite{Liflyand2012}.
The regularity that we require on $p_{\bm{w}}$ and $p_b$ can now be specified:

\begin{assumption}[Finite bandwidth] \label{asmp: ridgelet}
	The probability density functions $p_{\bm{w}} : \mathbb{R}^d \rightarrow [0,\infty)$ and $p_b : \mathbb{R} \rightarrow [0,\infty)$ satisfy properties \emph{(i)} – \emph{(v)} above and, in addition,
	\begin{enumerate}
		\item $\int_{\mathbb{R}^d} \|\bm{x}\|^2 \widehat{p_{\bm{w}}}(\bm{x}) \mathrm{d}\bm{x} < \infty$,
		\item $\|\partial p_b\|_{L^\infty(\mathbb{R})} < \infty$,
	\end{enumerate}
	where we recall that $\partial p_b$ denotes the first derivative of $p_b$.
\end{assumption}

Now we present the convergence result of interest.
Let $M_1(f) := \max_{|\bm{\alpha}| \leq 1} \sup_{\bm{x} \in \mathbb{R}^d} | \partial^{\bm{\alpha}} f(\bm{x})|$ and $B_1(f) := \int_{\mathbb{R}^d} |f(\bm{x})| (1 + \|\bm{x}\|_2) \mathrm{d}\bm{x}$.

\begin{theorem} \label{thm: recon_mod}
	Let \Cref{assm: ridgelet pair} and \Cref{asmp: ridgelet} hold, and let $ f \in C^1(\mathbb{R}^d) $ satisfy $M_1(f) < \infty$ and $ B_1(f) < \infty $.
	Then 
	\begin{eqnarray*}
		\sup_{\bm{x} \in \mathbb{R}^d} \left| f(\bm{x}) - (R_\sigma^{\ast} R)[ f ](\bm{x}) \right| \le C \max(M_1(f),B_1(f)) \left\{ \frac{1}{\sigma_{\bm{w}}} + \frac{\sigma_{\bm{w}}^{d} (\sigma_{\bm{w}} + 1)}{\sigma_{b}} \right\}
	\end{eqnarray*}
	for some constant $C $ that is independent of $\sigma_{\bm{w}}, \sigma_{b}$ and $f$, but may depend on $ \phi $, $ \psi $, $ p_{\bm{w}} $ and $ p_b $.
\end{theorem}

Our priority here was to provide a simple upper bound on the reconstruction error, that separates the influence of $f$ from the influence of the bandwidths $\sigma_{\bm{w}}$ and $\sigma_b$; the bound is not claimed to be tight. 
The formulation of \Cref{thm: recon_mod} enables us to conveniently conclude that taking $\sigma_{\bm{w}}$ and $\sigma_b$ to infinity in a manner such that $\sigma_{\bm{w}}^{d} (\sigma_{\bm{w}} + 1) / \sigma_b \rightarrow 0$ will lead to consistent recovery of $f$, for any function $f$ for which $M_1(f), B_1(f) < \infty$ is satisfied.
The term $M_1(f)$ reflects the general difficulty of approximating $f$ using a finite-bandwidth ridgelet transform, while the term $B_1(f)$ serves to deal with the tails of $f$ and the fact that a supremum is taken over the unbounded domain $\mathbb{R}^d$.
\Cref{lem: alt_1} in \Cref{sec: proof_lemma_unbounded} presents a variation on \Cref{thm: recon_mod} that relaxes the boundedness assumption on the activation function from $\phi \in C_0^*(\mathbb{R})$ to $\phi \in C_1^*(\mathbb{R})$ at the expense of changing the tail condition from $B_1(f) < \infty$ to $B_2(f) := \int_{\mathbb{R}^d} |f(\bm{x})| (1 + \|\bm{x}\|_2)^2 \mathrm{d}\bm{x} < \infty$ and restricting the supremum to a bounded subset of $\mathbb{R}^d$.

The proof of \Cref{thm: recon_mod} makes use of techniques from the theory of \emph{approximate identity} integral operators, which we recall next in \Cref{sec: appdx_apprximate_identity}.
The proof of \Cref{thm: recon_mod} is then presented in \Cref{sec: appdx_main_proof_recon_mod}.
The proof relies on an auxiliary technical lemma regarding the Fourier transform, whose statement and proof we defer to \Cref{subsec: aux lem}.

\subsubsection{Approximate Identity Operators} \label{sec: appdx_apprximate_identity}

An \emph{approximate identity} operator is an integral operator which converges to an identity operator in an appropriate limit.
The discussion in this section follows Gin\`{e} and Nickl \cite[ch.~4.1.3, 4.3.6]{Gine2015b}.
For $ h > 0 $, define an operator $ K_h $ by
\begin{eqnarray}
K_h[f](\bm{x}) := \int_{\mathbb{R}^{d}} f(\bm{x}') \frac{1}{h^d} K\left( \frac{\bm{x}}{h}, \frac{\bm{x}'}{h} \right) \mathrm{d} \bm{x}' \label{eq: apx_id_eq}
\end{eqnarray}
where $ K: \mathbb{R}^d \times \mathbb{R}^d \to \mathbb{R} $ is a measurable function and $f:\mathbb{R}^d \to \mathbb{R}$ is suitably regular for the integral to exist and be well-defined.
\Cref{prop: axid_rate} provides sufficient conditions for the approximate identity $ K_h[f] $ to converge to an identity operator when $ h \to 0 $.

\begin{proposition} \cite[Propositions~4.3.31 and 4.3.33, p.368]{Gine2015b} \label{prop: axid_rate}
	Let $ K_h[f] $ be defined as in \eqref{eq: apx_id_eq} with $ K $ a measurable function satisfying, for some $ N \in \mathbb{N}_0 $,
	\begin{enumerate}
		\item $\displaystyle \int_{\mathbb{R}^{d}} \sup_{\bm{v} \in \mathbb{R}^d} \left| K(\bm{v}, \bm{v} - \bm{u}) \right| \left\| \bm{u} \right\|^{N} \mathrm{d} \bm{u} < \infty$,
		\item for all $ \bm{v} \in \mathbb{R}^d $ and all multi-indices $ \bm{\alpha} $ s.t. $ | \bm{\alpha} | \in \{ 1, \dots , N - 1 \} $, 
		\begin{enumerate}
			\item $\displaystyle \int_{\mathbb{R}^{d}} K(\bm{v}, \bm{v} - \bm{u}) \mathrm{d} \bm{u} = 1$,
			\item $\displaystyle \int_{\mathbb{R}^{d}} K(\bm{v}, \bm{v} - \bm{u}) \bm{u}^{\bm{\alpha}} \mathrm{d} \bm{u} = 0$.
		\end{enumerate}
	\end{enumerate}	
	Then for each $ m \le N $ there exists a constant $ C $, depending only on $ m $ and $ K $, such that	
	\begin{center}
		\begin{tabular}{lcc}
			$ f \in C^m(\mathbb{R}^d) $ & $ \implies $ & $\displaystyle \sup_{\bm{x} \in \mathbb{R}^d} \left| K_h[f](\bm{x}) - f(\bm{x}) \right| \le C h^m \max_{|\bm{\alpha}| = m} \sup_{\bm{x} \in \mathbb{R}^d} \left| \partial^{\bm{\alpha}} f (\bm{x}) \right|  $.
		\end{tabular}
	\end{center}
\end{proposition}

The sense in which \Cref{prop: axid_rate} will be used in the proof of \Cref{thm: recon_mod} is captured by the following example:

\begin{example} \label{ex: approx_ker}
	Consider a translation invariant kernel of the form $ K(\bm{x}, \bm{x}') = \varphi(\bm{x} - \bm{x}') $ for some $ \varphi : \mathbb{R}^d \rightarrow (0,\infty) $.
	Further assume $ \int_{\mathbb{R}^d} \varphi(\bm{u}) \mathrm{d}\bm{u} = 1 $ and $\int_{\mathbb{R}^d} \varphi(\bm{u}) \|\bm{u}\|^2 \mathrm{d}\bm{u} < \infty$.
	Then $ K $ satisfies the preconditions of \Cref{prop: axid_rate} for $ N = 2 $ and hence $ K_h[f] $ is an approximate identity operator.
\end{example}

\subsubsection{Proof of \Cref{thm: recon_mod}} \label{sec: appdx_main_proof_recon_mod}

For this section it is convenient to introduce a notational shorthand.
Let $ \epsilon_{\bm{w}} := \sigma_{\bm{w}}^{-1} $ and $ \epsilon_b := \sigma_b^{-1} $, so that $ p_{\bm{w}}(\sigma_{\bm{w}}^{-1} \bm{w}) = p_{\bm{w}}(\epsilon_{\bm{w}} \bm{w}) $ and $ p_b(\sigma_b^{-1} b) = p_b(\epsilon_b b) $.
Further introduce the shorthand $ p_{\bm{w}}^\epsilon(\cdot) := p_{\bm{w}}(\epsilon_{\bm{w}} \cdot) $ and $ p_b^\epsilon(\cdot) := p_b(\epsilon_b \cdot) $.

Now we turn to the proof of \Cref{thm: recon_mod}.
Denote the topological dual space of $ \mathcal{S}(\mathbb{R}^d) $ by $ \mathcal{S}'(\mathbb{R}^d) $.
The elements of $ \mathcal{S}'(\mathbb{R}^d) $ are called \emph{tempered distributions} and the generalized Fourier transform can be considered as the Fourier transform on tempered distributions $ \mathcal{S}'(\mathbb{R}^d) $ \cite[p.123-131]{Grafakos2000b}.
Throughout the proof we exchange the order of integrals; we do so only when the absolute value of the integrand is itself an integrable function, so that from Fubini's theorem the interchange can be justified.

\begin{proof}[Proof of \Cref{thm: recon_mod}]
	The goal is to bound reconstruction error when the Lebesgue measure in the classical ridgelet transform is replaced by the finite measure
	\begin{eqnarray}
	\mathrm{d} \lambda_\sigma(\bm{w}, b) 
	= (2 \pi)^{\frac{1}{2}} \| \widehat{p_{\bm{w}}} \|_{L^1}^{-1} \| \widehat{p_b} \|_{L^1}^{-1} p_{\bm{w}}(\sigma_{\bm{w}}^{-1} \bm{w}) p_b(\sigma_b^{-1} b) \mathrm{d} \bm{w} \mathrm{d} b . \nonumber
	\end{eqnarray}
	The reconstruction $(R_\sigma^\ast R) [f]$ of a function $f$ on $\mathbb{R}^d$ is defined by the following integral:
	\begin{align}
	& (R_\sigma^{\ast} R)[ f](\bm{x}) \nonumber \\
	& = \int_{\mathbb{R}^{d+1}} \int_{\mathbb{R}^d} f(\bm{x}') \psi(\bm{w} \cdot \bm{x}' + b) \mathrm{d}\bm{x}' \phi(\bm{w} \cdot \bm{x} + b) \mathrm{d} \lambda_\sigma(\bm{w}, b) \nonumber \\
	& = \int_{\mathbb{R}^d} f(\bm{x}') \left\{ (2 \pi)^{\frac{1}{2}} \| \widehat{p_{\bm{w}}} \|_{L^1}^{-1} \| \widehat{p_b} \|_{L^1}^{-1} \int_{\mathbb{R}} \int_{\mathbb{R}^d} \psi(\bm{w} \cdot \bm{x}' + b) \phi(\bm{w} \cdot \bm{x} + b) p_{\bm{w}}^\epsilon(\bm{w}) p_b^\epsilon(b) \mathrm{d} \bm{w} \mathrm{d} b \right\} \mathrm{d}\bm{x}'. \nonumber
	\end{align}
	It is thus clear that $(R_\sigma^\ast R )[f]$ is in some sense a smoothed version of $f$, and in what follows we will make use of the theory of approximate identity operators discussed in \Cref{sec: appdx_apprximate_identity}.
	We aim to deal separately with the reconstruction error due to the use of a finite measure on $\bm{w}$ and the reconstruction error due to the use of a finite measure on $b$.
	To this end, it will be convenient to (formally) define a new operator $ R^{\ast}_{\sigma} R^{(w)} $ by 
	\begin{eqnarray*}
		\left(R_\sigma^{\ast} R^{(w)}\right)[ f](\bm{x}) := \int_{\mathbb{R}^{d}} f(\bm{x}') \left( \| \widehat{p_{\bm{w}}} \|_{L^1}^{-1} \int_{\mathbb{R}} \int_{\mathbb{R}^d} \psi(\bm{w} \cdot \bm{x}' + b) \phi(\bm{w} \cdot \bm{x} + b) p_{\bm{w}}^\epsilon(\bm{w}) \mathrm{d} \bm{w} \mathrm{d} b \right) \mathrm{d} \bm{x'} 
	\end{eqnarray*}
	which replaces $p_b^\epsilon$ by the Lebesgue measure on $\mathbb{R}$.
	That is, $(R^{\ast}_{\sigma} R^{(w)})[f]$ can be intuitively considered as an idealised version of $(R^{\ast}_{\sigma} R )[f]$ where the reconstruction error due to the use of a finite measure on $b$ is removed.
	Our analysis will then proceed based on the following triangle inequality:
	\begin{multline}
	\sup_{\bm{x} \in \mathbb{R}^d} | f(\bm{x}) - (R^{\ast}_{\sigma} R)[ f ](\bm{x}) | \\
	\le \underbrace{ \sup_{\bm{x} \in \mathbb{R}^d} | f(\bm{x}) - (R^{\ast}_{\sigma} R^{(w)})[ f](\bm{x}) | }_{(*)} + \underbrace{ \sup_{\bm{x} \in \mathbb{R}^d} | (R^{\ast}_{\sigma} R^{(w)})[ f ](\bm{x}) - (R^{\ast}_{\sigma} R)[ f](\bm{x}) | }_{(**)} . \label{eq: conv_tri}
	\end{multline}
	Different strategies are required to bound $(*)$ and $(**)$ and we address them separately next.
	
	\vspace{5pt}
	\noindent \textbf{Bounding $(*)$:}	
	A bound on $(*)$ uses techniques from approximate identity operators described in \Cref{sec: appdx_apprximate_identity}.
	To this end, we show that $ R^{\ast}_{\sigma} R $ satisfies the preconditions of \Cref{prop: axid_rate} to obtain a bound.
	
	Define $ \widetilde{k}: \mathbb{R}^d \times \mathbb{R}^d \to \mathbb{R} $ by
	\begin{eqnarray}
	\widetilde{k} \left(\bm{x}, \bm{x}' \right) & := & \| \widehat{p_{\bm{w}}} \|_{L^1}^{-1} \int_{\mathbb{R}} \int_{\mathbb{R}^{d}}  \phi(\bm{w} \cdot \bm{x} + b)  \psi(\bm{w} \cdot \bm{x}' + b) p_{\bm{w}}(\bm{w})  \mathrm{d} \bm{w} \mathrm{d} b. \label{eq: k_gamma} 
	\end{eqnarray}
	From a change of variable $ b' = \bm{w} \cdot \bm{x} + b $,
	\begin{eqnarray}
	\widetilde{k} \left(\bm{x}, \bm{x}' \right) & = & \| \widehat{p_{\bm{w}}} \|_{L^1}^{-1} \int_{\mathbb{R}^{d}} \int_{\mathbb{R}} \phi(b') \psi(\bm{w} \cdot (\bm{x}' - \bm{x}) + b')  \mathrm{d} b' p_{\bm{w}}(\bm{w}) \mathrm{d} \bm{w}. \label{eq: ktilde intermediate}
	\end{eqnarray}
	Next we perform some Fourier analysis similar to that in \cite[Appendix C]{Sonoda2017b}.
	From the discussion on generalised Fourier transform in \Cref{sec: def_not}, recall that $ \widehat{t} \in L^2_{\text{loc}}(\mathbb{R}^d \setminus \{ \bm{0} \}) $ is a \emph{generalised Fourier transform} of function $ t $ if there exists an integer $ m \in \mathbb{N}_{0} $ such that $ \int_{\mathbb{R}^{d}} \widehat{t}(\bm{w}) h(\bm{w}) \mathrm{d} \bm{w} = \int_{\mathbb{R}^{d}} t(\bm{x}) \widehat{h}(\bm{x}) \mathrm{d} \bm{x} $ for all $ h \in \mathcal{S}_{2m}(\mathbb{R}^d) $.
	We set $t = \phi$ and $\widehat{h} = \psi$ for our analysis.
	Part (3) of \Cref{assm: ridgelet pair} implies that $\psi \in \mathcal{S}_{2m}(\mathbb{R}^d)$ for some $m$ depending on the generalised Fourier transform $\widehat{\phi}$.
	Let $\psi_{-}(b) := \psi(- b)$.
	From the definition of the generalised Fourier transform, 
	\begin{align}
	\int_{\mathbb{R}} \phi(b') \psi(\bm{w} \cdot (\bm{x}' - \bm{x}) + b')  \mathrm{d} b' & = \int_{\mathbb{R}} \phi(b') \psi_{-}(- b' - \bm{w} \cdot (\bm{x}' - \bm{x}))  \mathrm{d} b' \nonumber \\
	& = \int_{\mathbb{R}} \widehat{\phi}(\xi_1) \overline{\widehat{\psi}(\xi_1)} e^{i \xi_1 \bm{w} \cdot (\bm{x} - \bm{x}')} \mathrm{d} \xi_1 ,	 \label{eq: conv step 2}
	\end{align}
	where we used the fact that $\widehat{\psi_{-}} = \overline{\widehat{\psi}}$ for the real function $\psi_{-}(b)$ for the last equality \cite[p.109,113]{Grafakos2000b}.
	
	Note that $ \xi_1 \mapsto \widehat{\phi}(\xi_1) \overline{\widehat{\psi}(\xi_1)} $ belongs to $ L^1(\mathbb{R}) $ from part (3) of \Cref{assm: ridgelet pair} and therefore (\ref{eq: conv step 2}) exists.
	Also we have that
	\begin{equation}
	\widehat{p_{\bm{w}}}(\xi_1 \bm{x}' - \xi_1 \bm{x})  = (2 \pi)^{-\frac{d}{2}} \int_{\mathbb{R}^{d}} e^{i \xi_1 \bm{w} \cdot (\bm{x} - \bm{x}')} p_{\bm{w}}(\bm{w}) \mathrm{d} \bm{w} . \label{eq: FT of pw}
	\end{equation}
	Substituting \eqref{eq: conv step 2} and \eqref{eq: FT of pw} into \eqref{eq: ktilde intermediate} gives that
	\begin{eqnarray*}
		\widetilde{k} \left(\bm{x}, \bm{x}' \right) & = & \| \widehat{p_{\bm{w}}} \|_{L^1}^{-1} \int_{\mathbb{R}^{d}} \int_{\mathbb{R}} \phi(b') \psi(\bm{w} \cdot (\bm{x}' - \bm{x}) + b')  \mathrm{d} b' p_{\bm{w}}(\bm{w}) \mathrm{d} \bm{w} \\
		& = & \| \widehat{p_{\bm{w}}} \|_{L^1}^{-1}  \int_{\mathbb{R}} \widehat{\phi}(\xi_1) \overline{\widehat{\psi}(\xi_1)} \int_{\mathbb{R}^d}  e^{i \xi_1 \bm{w} \cdot (\bm{x} - \bm{x}')} p_{\bm{w}}(\bm{w}) \mathrm{d}\bm{w}  \mathrm{d} \xi_1 \\
		& = & (2 \pi)^{\frac{d}{2}} \| \widehat{p_{\bm{w}}} \|_{L^1}^{-1} \int_{\mathbb{R}} \widehat{\phi}(\xi_1) \overline{\widehat{\psi}(\xi_1)} \widehat{p_{\bm{w}}}(\xi_1 \bm{x}' - \xi_1 \bm{x}) \mathrm{d} \xi_1 .
	\end{eqnarray*}
	Let $ k^{\bm{w}}(\bm{v}) := \| \widehat{p_{\bm{w}}} \|_{L^1}^{-1} \widehat{p_{\bm{w}}}(\bm{v}) $, so that
	\begin{eqnarray}
	\widetilde{k} \left(\bm{x}, \bm{x}' \right) & = & (2 \pi)^{\frac{d}{2}} \int_{\mathbb{R}} \widehat{\phi}(\xi) \overline{\widehat{\psi}(\xi)}  k^{\bm{w}}(\xi \bm{x} - \xi \bm{x}') \mathrm{d} \xi. \label{eq: a ktilde expression}
	\end{eqnarray}
	Recall that $ p_{\bm{w}}^\epsilon(\bm{w}) = p_{\bm{w}}(\epsilon \bm{w}) $ and let $ k^{\bm{w}}_\epsilon(\bm{v}) := \| \widehat{p_{\bm{w}}} \|_{L^1}^{-1} \widehat{p_{\bm{w}}^{\epsilon}}(\bm{v}) $.
	Since the Fourier transform of $ \bm{w} \mapsto p_{\bm{w}}(\epsilon_{\bm{w}} \bm{w}) $ is to be $ \bm{v} \mapsto \frac{1}{\epsilon_{\bm{w}}^{d}} \widehat{p_{\bm{w}}}\left(\frac{\bm{v}}{\epsilon_{\bm{w}}}\right) $ by standard properties of the Fourier transform \cite[p.109,113]{Grafakos2000b}, $ k^{\bm{w}}_\epsilon(\bm{v}) = \frac{1}{\epsilon_{\bm{w}}^d} k^{\bm{w}}\left(\frac{\bm{v}}{\epsilon_{\bm{w}}}\right) $.
	Define 
	\begin{eqnarray}
	\widetilde{k}_\epsilon \left(\bm{x}, \bm{x}' \right) & := & \| \widehat{p_{\bm{w}}} \|_{L^1}^{-1} \int_{\mathbb{R}} \int_{\mathbb{R}^{d}} \phi(\bm{w} \cdot \bm{x} + b)  \psi(\bm{w} \cdot \bm{x}' + b) p_{\bm{w}}^\epsilon(\bm{w}) \mathrm{d} \bm{w} \mathrm{d} b. \label{eq: k_tilde}
	\end{eqnarray}
	Noting the similarity between \eqref{eq: k_tilde} and \eqref{eq: k_gamma}, an analogous argument to that just presented shows that we may re-express \eqref{eq: k_tilde} as
	\begin{equation*}
	\widetilde{k}_\epsilon \left(\bm{x}, \bm{x}' \right) = (2 \pi)^{\frac{d}{2}} \int_{\mathbb{R}} \widehat{\phi}(\xi) \overline{\widehat{\psi}(\xi)}  k^{\bm{w}}_\epsilon(\xi \bm{x} - \xi \bm{x}') \mathrm{d} \xi = \frac{1}{\epsilon_{\bm{w}}^d} \widetilde{k} \left(\frac{\bm{x}}{\epsilon_{\bm{w}}}, \frac{\bm{x}'}{\epsilon_{\bm{w}}} \right) .
	\end{equation*}
	Then
	\begin{eqnarray}
	R_\sigma^{\ast} R^{(w)} f(\bm{x}) = \int_{\mathbb{R}^{d}} f(\bm{x}') \widetilde{k}_\epsilon \left(\bm{x}, \bm{x}' \right) \mathrm{d} \bm{x}' = \int_{\mathbb{R}^{d}} f(\bm{x}') \frac{1}{\epsilon_{\bm{w}}^d} \widetilde{k} \left(\frac{\bm{x}}{\epsilon_{\bm{w}}}, \frac{\bm{x}'}{\epsilon_{\bm{w}}} \right) \mathrm{d} \bm{x}'. \nonumber
	\end{eqnarray}
	Now we will show that $ \widetilde{k} $ satisfies the pre-conditions of \Cref{prop: axid_rate}.
	That is, setting $ N = 2 $, we will show that:
	For all $ \bm{v} \in \mathbb{R}^d $ and all multi-index $ \bm{\alpha} \in \mathbb{N}_0^{d} $ s.t. $ | \bm{\alpha} | = 1 $,
	\begin{eqnarray}
	&& \int_{\mathbb{R}^{d}} \sup_{\bm{v} \in \mathbb{R}^d} \left| \widetilde{k}(\bm{v}, \bm{v} - \bm{u}) \right| \left\| \bm{u} \right\|^2 \mathrm{d} \bm{u} < \infty, \label{eq: condition 1} \\
	&& \int_{\mathbb{R}^{d}} \widetilde{k}(\bm{v}, \bm{v} - \bm{u}) \mathrm{d} \bm{u} = 1, \label{eq: condition 2} \\
	&& \int_{\mathbb{R}^{d}} \widetilde{k}(\bm{v}, \bm{v} - \bm{u}) \bm{u}^{\bm{\alpha}} \mathrm{d} \bm{u} = 0. \label{eq: condition 3}
	\end{eqnarray}
	First we verify \eqref{eq: condition 1}.
	From \eqref{eq: a ktilde expression} we have that
	\begin{eqnarray}
	\int_{\mathbb{R}} \sup_{\bm{v} \in \mathbb{R}^d} \left| \widetilde{k}(\bm{v}, \bm{v} - \bm{u}) \right| \left\| \bm{u} \right\|^2 \mathrm{d} \bm{u} & = & (2 \pi)^{\frac{d}{2}} \int_{\mathbb{R}^d} \left| \int_{\mathbb{R}} \overline{\widehat{\psi}(\xi)} \widehat{\phi}(\xi_1) k^{\bm{w}}(\xi \bm{u}) \mathrm{d} \xi \right| \left\| \bm{u} \right\|^2 \mathrm{d} \bm{u}. \nonumber \\
	& \le & (2 \pi)^{\frac{d}{2}} \int_{\mathbb{R}} \int_{\mathbb{R}^d} \left| \overline{\widehat{\psi}(\xi)} \widehat{\phi}(\xi) \right| \left| k^{\bm{w}}(\xi \bm{u}) \right| \left\| \bm{u} \right\|^2 \mathrm{d} \bm{u} \mathrm{d} \xi. \nonumber
	\end{eqnarray}
	By the change of variables $ \bm{u}' = \xi \bm{u} $,
	\begin{eqnarray}
	\int_{\mathbb{R}} \sup_{\bm{v} \in \mathbb{R}^d} \left| \widetilde{k}(\bm{v}, \bm{v} - \bm{u}) \right| \left\| \bm{u} \right\|^2 \mathrm{d} \bm{u} & \le & (2 \pi)^{\frac{d}{2}} \int_{\mathbb{R}^d} \int_{\mathbb{R}} | \overline{\widehat{\psi}(\xi)} \widehat{\phi}(\xi) | | k^{\bm{w}}(\bm{u}') | \left\| \frac{\bm{u}'}{\xi} \right\|^2 \frac{1}{| \xi |^d} \mathrm{d} \bm{u}' \mathrm{d} \xi \nonumber \\
	& \le & (2 \pi)^{\frac{d}{2}} \int_{\mathbb{R}} \int_{\mathbb{R}^d} \frac{ | \overline{\widehat{\psi}(\xi)} \widehat{\phi}(\xi) | }{ | \xi |^{d+2} } | k^{\bm{w}}(\bm{u}') | \left\| \bm{u}' \right\|^2 \mathrm{d} \bm{u}' \mathrm{d} \xi \nonumber \\
	& \le & (2 \pi)^{\frac{d}{2}} \int_{\mathbb{R}} \frac{ \left| \overline{\widehat{\psi}(\xi)} \widehat{\phi}(\xi) \right| }{ | \xi |^{d+2} } \mathrm{d} \xi \int_{\mathbb{R}^d} | k^{\bm{w}}(\bm{u}') | \left\| \bm{u}' \right\|^2 \mathrm{d} \bm{u}' \nonumber
	\end{eqnarray}
	That this final bound is finite follows from the requirement that $ \int_{\mathbb{R}} (| \overline{\widehat{\psi}(\xi)} \widehat{\phi}(\xi) | / | \xi |^{d+2} ) \mathrm{d} \xi < \infty $ in \Cref{def: ridgelet_modified}, together with the assumption that $\widehat{p_{\bm{w}}}$ has finite second moment, which ensures the finiteness of $ \int_{\mathbb{R}^d} | k^{\bm{w}}(\bm{u}') | \left\| \bm{u}' \right\|^2 \mathrm{d} \bm{u}' $.
	Next we verify \eqref{eq: condition 2}.
	From \eqref{eq: a ktilde expression} and the change of variables $ \bm{u}' = \xi \bm{u} $,
	\begin{eqnarray}
	\int_{\mathbb{R}^d} \widetilde{k}(\bm{v}, \bm{v} - \bm{u}) \mathrm{d} \bm{u} & = & (2 \pi)^{\frac{d}{2}} \int_{\mathbb{R}^d} \int_{\mathbb{R}} \overline{\widehat{\psi}(\xi)} \widehat{\phi}(\xi_1) k^{\bm{w}}(\xi \bm{u}) \mathrm{d} \xi \mathrm{d} \bm{u} \nonumber \\
	& = & (2 \pi)^{\frac{d}{2}} \int_{\mathbb{R}} \frac{\overline{\widehat{\psi}(\xi)} \widehat{\phi}(\xi)}{ | \xi |^d } \mathrm{d} \xi \int_{\mathbb{R}^d} k^{\bm{w}}(\bm{u}') \mathrm{d} \bm{u}' = 1 , \nonumber
	\end{eqnarray}
	where the final inequality used the facts that $ \int_{\mathbb{R}^d} k^{\bm{w}}(\bm{u}') \mathrm{d} \bm{u}' = \int_{\mathbb{R}^d} | k^{\bm{w}}(\bm{u}') | \mathrm{d} \bm{u}' = 1$ by construction and $(2 \pi)^{\frac{d}{2}} \int_{\mathbb{R}} \frac{\overline{\widehat{\psi}(\xi)} \widehat{\phi}(\xi)}{ | \xi |^d } \mathrm{d} \xi = 1 $ from Definition \ref{def: ridgelet_modified}.
	Finally we verify \eqref{eq: condition 1}.
	From \eqref{eq: a ktilde expression} and the change of variables $ \bm{u}' = \xi \bm{u} $,
	\begin{eqnarray}
	\int_{\mathbb{R}^d} \widetilde{k}(\bm{v}, \bm{v} - \bm{u}) \bm{u}^{\bm{\alpha}} \mathrm{d} \bm{u} & = & (2 \pi)^{\frac{d}{2}} \int_{\mathbb{R}^d} \int_{\mathbb{R}} \overline{\widehat{\psi}(\xi_1)} \widehat{\phi}(\xi_1) k^{\bm{w}}(\xi \bm{u}) \mathrm{d} \xi \bm{u}^{\bm{\alpha}} \mathrm{d} \bm{u} \nonumber \\
	& = & \underbrace{ (2 \pi)^{\frac{d}{2}} \int_{\mathbb{R}} \frac{\overline{\widehat{\psi}(\xi)} \widehat{\phi}(\xi)}{| \xi |^{d} } \frac{1}{\xi^{| \bm{\alpha} |}} \mathrm{d} \xi }_{(1)}  \underbrace{ \int_{\mathbb{R}^d} k^{\bm{w}}(\bm{u}') {\bm{u}'}^{\bm{\alpha}} \mathrm{d} \bm{u}' }_{(2)} . \nonumber
	\end{eqnarray}
	The term $(1)$ is finite as a consequence of the assumption $ \int_{\mathbb{R}} \frac{ \left| \overline{\widehat{\psi}(\xi)} \widehat{\phi}(\xi) \right| }{ | \xi |^{d+2} } \mathrm{d} \xi < \infty $ in Definition \ref{def: ridgelet_modified}.
	For the term $(2)$, note that $\bm{u} \mapsto \bm{u}^{\bm{\alpha}}$ is odd function whenever $ | \bm{\alpha} | = 1 $.
	On the other hand, $ k^{\bm{w}} $ is even function since $ k^{\bm{w}} $ is the Fourier transform of an even probability density $ p_{\bm{w}} $.
	Thus the function $ \bm{u} \mapsto k^{\bm{w}}(\bm{u}) \bm{u}^{\bm{\alpha}} $, which is given by the product of even and odd functions, is odd.
	For any integrable odd function $ h: \mathbb{R}^d \to \mathbb{R} $, $ \int_{\mathbb{R}^d} h(\bm{u}) \mathrm{d} \bm{u} = 0 $ holds as $ \int_{\mathbb{R}^d_{+}} h(\bm{u}) \mathrm{d} \bm{u} = - \int_{\mathbb{R}^d_{-}} h(\bm{u}) \mathrm{d} \bm{u} $ where $ \mathbb{R}^d_{+} $ and $ \mathbb{R}^d_{-} $ are positive and negative half Euclidean space.
	The function $ \bm{u} \mapsto k^{\bm{w}}(\bm{u}) \bm{u}^{\bm{\alpha}} $ is integrable since $ k^{\bm{w}} $ has the finite second moment and $ | \bm{\alpha} | = 1 $.
	This implies $ \int_{\mathbb{R}^d} \widetilde{k}(\bm{v}, \bm{v} - \bm{u}) \bm{u}^{\bm{\alpha}} \mathrm{d} \bm{u} = \bm{0} $.
	
	Thus $ \widetilde{k} $ satisfies the condition in Proposition \ref{prop: axid_rate} and, with $K_h[f] = R_\sigma^* R^{(w)} f$, we obtain for some $ C_1 > 0 $ depending on $\tilde{k}$ but not $f$,
	\begin{eqnarray}
	\sup_{\bm{x} \in \mathbb{R}^d} | f(\bm{x}) - R^{\ast}_{\sigma} R^{(w)} f(\bm{x}) | \le C_1 M_1(f) \epsilon_{\bm{w}} . \label{eq: bound on ast}
	\end{eqnarray}
	
	\vspace{5pt}
	\noindent \textbf{Bounding $(**)$:}	
	A bound on $(**)$ makes use of Auxiliary Lemma \ref{lem: FT of pb}.
	Define $ k_\epsilon: \mathbb{R}^d \times \mathbb{R}^d \to \mathbb{R} $ by
	\begin{eqnarray}
	k_\epsilon(\bm{x}, \bm{x}') := (2 \pi)^{\frac{1}{2}} \| \widehat{p_{\bm{w}}} \|_{L^1}^{-1} \| \widehat{p_b} \|_{L^1}^{-1} \int_{\mathbb{R}} \int_{\mathbb{R}^{d}} \psi(\bm{w} \cdot \bm{x}' + b) \phi(\bm{w} \cdot \bm{x} + b) p_{\bm{w}}^\epsilon(\bm{w}) p_b^\epsilon(b) \mathrm{d} \bm{w} \mathrm{d} b.
	\end{eqnarray} 
	so that $ R^{\ast}_{\sigma} R f(\bm{x}) = \int_{\mathbb{R}^{d}} f(\bm{x}') k_\epsilon(\bm{x}, \bm{x}') \mathrm{d} \bm{x}' $ and recall that $ R^{\ast}_{\sigma} R^{(w)} f(\bm{x}) = \int_{\mathbb{R}^{d}} f(\bm{x}') \widetilde{k}_\epsilon(\bm{x}, \bm{x}') \mathrm{d} \bm{x}' $ where $ \widetilde{k}_\epsilon $ is defined by (\ref{eq: k_tilde}).
	The second error term is 
	\begin{eqnarray}
	\sup_{\bm{x} \in \mathbb{R}^d} | R^{\ast}_{\sigma} R^{(w)} f(\bm{x}) - R^{\ast}_{\sigma} R f(\bm{x}) | & = & \sup_{\bm{x} \in \mathbb{R}^d} \left| \int_{\mathbb{R}^{d}} f(\bm{x}') \widetilde{k}_\epsilon(\bm{x}, \bm{x}') \mathrm{d} \bm{x}' - \int_{\mathbb{R}^{d}} f(\bm{x}') k_\epsilon(\bm{x}, \bm{x}') \mathrm{d} \bm{x}' \right| \nonumber \\
	& = & \sup_{\bm{x} \in \mathbb{R}^d} \left| \int_{\mathbb{R}^{d}} f(\bm{x}') \left( \widetilde{k}_\epsilon(\bm{x}, \bm{x}') - k_\epsilon(\bm{x}, \bm{x}') \right) \mathrm{d} \bm{x}' \right|. \nonumber
	\end{eqnarray}
	Let $ \Delta k(\bm{x}, \bm{x}') := \widetilde{k}_\epsilon(\bm{x}, \bm{x}') - k_\epsilon(\bm{x}, \bm{x}') $ for short-hand.
	By the definition of $ \widetilde{k}_\epsilon $ and $ k_\epsilon $, 
	\begin{eqnarray}
	\Delta k(\bm{x}, \bm{x}') = \| \widehat{p_{\bm{w}}} \|_{L^1}^{-1} \int_{\mathbb{R}} \int_{\mathbb{R}^{d}} \psi(\bm{w} \cdot \bm{x}' + b) \phi(\bm{w} \cdot \bm{x} + b) p_{\bm{w}}^\epsilon(\bm{w}) \left\{ 1 - (2 \pi)^{\frac{1}{2}} \| \widehat{p_b} \|_{L^1}^{-1} p_b^\epsilon(b) \right\} \mathrm{d} \bm{w} \mathrm{d} b. \nonumber
	\end{eqnarray}
	Next we upper bound $ \Delta k(\bm{x}, \bm{x}') $.
	Substituting the identity established in \Cref{lem: FT of pb} yields
	\begin{eqnarray}
	\Delta k(\bm{x}, \bm{x}') & = & \epsilon_b (2 \pi)^{\frac{1}{2}} \| \widehat{p_{\bm{w}}} \|_{L^1}^{-1} \| \widehat{p_b} \|_{L^1}^{-1} \int_{\mathbb{R}} \int_{\mathbb{R}^{d}} \psi(\bm{w} \cdot \bm{x}' + b) \phi(\bm{w} \cdot \bm{x} + b) p_{\bm{w}}^\epsilon(\bm{w}) \int_0^1 b \partial p_b(t \epsilon_b b) \mathrm{d} t \mathrm{d} \bm{w} \mathrm{d} b \nonumber \\
	& = & \epsilon_b (2 \pi)^{\frac{1}{2}} \| \widehat{p_{\bm{w}}} \|_{L^1}^{-1} \| \widehat{p_b} \|_{L^1}^{-1} \int_0^1 \int_{\mathbb{R}^{d}}  \underbrace{ \int_{\mathbb{R}} b \phi(\bm{w} \cdot \bm{x} + b) \psi(\bm{w} \cdot \bm{x}' + b) \partial p_b(t \epsilon_b b) \mathrm{d} b }_{(3)} p_{\bm{w}}^\epsilon(\bm{w}) \mathrm{d} \bm{w} \mathrm{d} t. \nonumber
	\end{eqnarray}
	By Assumption \ref{asmp: ridgelet}, $ p_b $ has bounded first derivative.
	Thus:
	\begin{eqnarray}
	|(3)| \le \|\partial p_b\|_{L^\infty} \int_{\mathbb{R}} | b \phi(\bm{w} \cdot \bm{x} + b) \psi(\bm{w} \cdot \bm{x}' + b) | \mathrm{d} b' \nonumber
	\end{eqnarray}
	By a change of variable $ b' = \bm{w} \cdot \bm{x}' + b $,
	\begin{eqnarray}
	|(3)| & \le & \|\partial p_b\|_{L^\infty} \int_{\mathbb{R}} | b'- \bm{w} \cdot \bm{x}' | | \phi(\bm{w} \cdot (\bm{x} - \bm{x}') + b') \psi(b') | \mathrm{d} b' \nonumber \\
	& \le & \|\partial p_b\|_{L^\infty} \int_{\mathbb{R}} ( | b' | + \| \bm{w} \|_2 \| \bm{x}' \|_2) | \phi(\bm{w} \cdot (\bm{x} - \bm{x}') + b') \psi(b') | \mathrm{d} b' \nonumber
	\end{eqnarray}
	where the triangle inequality and Cauchy-Schwartz inequality have been applied.
	Then it is easy to see
	\begin{eqnarray}
	|(3)| & \le & \|\partial p_b\|_{L^\infty} (1 + \| \bm{x}' \|_2) (1 + \| \bm{w} \|_2) \int_{\mathbb{R}} (1 + | b' |) | \phi(\bm{w} \cdot (\bm{x} - \bm{x}') + b') \psi(b') | \mathrm{d} b' \label{eq: bound_of_3} 
	\end{eqnarray}
	By the assumption $ \phi $ is bounded,
	\begin{eqnarray}
	|(3)| & \le & \|\partial p_b\|_{L^\infty} \|\phi\|_{L^\infty} (1 + \| \bm{x}' \|_2) (1 + \| \bm{w} \|_2) \int_{\mathbb{R}} (1 + | b' |) | \psi(b') | \mathrm{d} b' \label{eq: bound_of_3_2} 
	\end{eqnarray}
	Since $ \psi $ is Schwartz function, the integral is finite. 
	Let $ C_{\psi} := \int_{\mathbb{R}} (1 + | b |) | \psi(b) | \mathrm{d} b $ to see $ |(3)| \le \|\partial p_b\|_{L^\infty} \|\phi\|_{L^\infty} C_{\psi} C_{p_{\bm{w}}} ( 1 + \| \bm{x}' \|_2 ) ( 1 + \| \bm{w} \|_2 ) $.
	From this upper bound of $ |(3)| $ we have
	\begin{eqnarray}
	\Delta k(\bm{x}, \bm{x}') & = & \epsilon_b (2 \pi)^{\frac{1}{2}} \| \widehat{p_{\bm{w}}} \|_{L^1}^{-1} \| \widehat{p_b} \|_{L^1}^{-1} \int_0^1 \int_{\mathbb{R}^{d}} | (3) | p_{\bm{w}}^\epsilon(\bm{w}) \mathrm{d} \bm{w} \mathrm{d} t \nonumber \\
	& \le & \epsilon_b (2 \pi)^{\frac{1}{2}} \| \widehat{p_{\bm{w}}} \|_{L^1}^{-1} \| \widehat{p_b} \|_{L^1}^{-1} \|\partial p_b\|_{L^\infty} \|\phi\|_{L^\infty} C_{\psi} C_{p_{\bm{w}}} (1 + \| \bm{x}' \|_2) \int_{\mathbb{R}^{d}} (1 + \| \bm{w} \|_2) p_{\bm{w}}^\epsilon(\bm{w}) \mathrm{d} \bm{w} \int_0^1 \mathrm{d} t \nonumber
	\end{eqnarray}
	For the integral $ \int_{\mathbb{R}^{d}} (1 + \| \bm{w} \|_2) p_{\bm{w}}^\epsilon(\bm{w}) \mathrm{d} \bm{w} $, by a change of variable $ \bm{w}' = \epsilon_{\bm{w}} \bm{w} $,
	\begin{eqnarray}
	\int_{\mathbb{R}^{d}} (1 + \| \bm{w} \|_2) p_{\bm{w}}^\epsilon(\bm{w}) \mathrm{d} \bm{w} = \int_{\mathbb{R}^{d}} \left( \frac{1}{\epsilon_{\bm{w}}^d} + \frac{\| \bm{w}' \|_2}{\epsilon_{\bm{w}}^{d+1}} \right) p_{\bm{w}}(\bm{w}') \mathrm{d} \bm{w}'. \nonumber 
	\end{eqnarray}
	Recall that $ p_{\bm{w}} $ was assumed to have finite second moment. 
	Let $ C_{p_{\bm{w}}} := \max\left( 1, \int_{\mathbb{R}^{d}} \| \bm{w} \|_2 p_{\bm{w}}(\bm{w}) \mathrm{d} \bm{w} \right) $ to see
	\begin{eqnarray}
	\int_{\mathbb{R}^{d}} \left( \frac{1}{\epsilon_{\bm{w}}^d} + \frac{\| \bm{w}' \|_2}{\epsilon_{\bm{w}}^{d+1}} \right) p_{\bm{w}}(\bm{w}') \mathrm{d} \bm{w}' \le C_{p_{\bm{w}}} \left( \frac{1}{\epsilon_{\bm{w}}^d} + \frac{1}{\epsilon_{\bm{w}}^{d+1}} \right) = C_{p_{\bm{w}}} \frac{1}{\epsilon_{\bm{w}}^d} \left( 1 + \frac{1}{\epsilon_{\bm{w}}} \right). \nonumber 
	\end{eqnarray}
	Plugging this upper bound in $ \Delta k(\bm{x}, \bm{x}') $,
	\begin{eqnarray}
	\Delta k(\bm{x}, \bm{x}') \le (2 \pi)^{\frac{1}{2}} \| \widehat{p_{\bm{w}}} \|_{L^1}^{-1} \| \widehat{p_b} \|_{L^1}^{-1} \|\partial p_b\|_{L^\infty} \|\phi\|_{L^\infty} C_{\psi} C_{p_{\bm{w}}} (1 + \| \bm{x}' \|_2) \frac{1}{\epsilon_{\bm{w}}^d} \left( 1 + \frac{1}{\epsilon_{\bm{w}}} \right) \epsilon_b. \nonumber
	\end{eqnarray}
	The original error term is then bounded as
	\begin{eqnarray}
	(**) & = & \sup_{\bm{x} \in \mathbb{R}^d} \left| \int_{\mathbb{R}^{d}} f(\bm{x}') \Delta k(\bm{x}, \bm{x}') \mathrm{d} \bm{x}' \right| \nonumber \\
	& \le & (2 \pi)^{\frac{1}{2}} \| \widehat{p_{\bm{w}}} \|_{L^1}^{-1} \| \widehat{p_b} \|_{L^1}^{-1} \|\partial p_b\|_{L^\infty} \|\phi\|_{L^\infty} C_{\psi} C_{p_{\bm{w}}} \int_{\mathbb{R}^{d}} (1 + \| \bm{x}' \|_2) | f(\bm{x}') | \mathrm{d} \bm{x}' \frac{1}{\epsilon_{\bm{w}}^d} \left( 1 + \frac{1}{\epsilon_{\bm{w}}} \right) \epsilon_b. \nonumber
	\end{eqnarray}
	From the assumption of $ f $ in Theorem \ref{thm: recon_mod}, $B_1(f) = \int_{\mathbb{R}^{d}} (1 + \| \bm{x} \|_2) | f(\bm{x}) | \mathrm{d} \bm{x} < \infty $.
	Setting $ C_2 := (2 \pi)^{\frac{1}{2}} \| \widehat{p_{\bm{w}}} \|_{L^1}^{-1} \| \widehat{p_b} \|_{L^1}^{-1} \|\partial p_b\|_{L^\infty} \|\phi\|_{L^\infty} C_{\psi} C_{p_{\bm{w}}} $ gives
	\begin{eqnarray}
	(**) \le C_2 B_1(f) \frac{1}{\epsilon_{\bm{w}}^d} \left( 1 + \frac{1}{\epsilon_{\bm{w}}} \right) \epsilon_b. \label{eq: final bound on astast}
	\end{eqnarray}
	
	\vspace{5pt}
	\noindent	
	\textbf{Overall Bound:}	
	Substituting the results of \eqref{eq: bound on ast} and \eqref{eq: final bound on astast} into (\ref{eq: conv_tri}), for some $ C > 0 $,
	\begin{eqnarray}
	\sup_{\bm{x} \in \mathbb{R}^d} | f(\bm{x}) - R^{\ast}_{\sigma} R f(\bm{x}) | \le C \max(M_1(f),B_1(f)) \left\{ \epsilon_{\bm{w}} + \frac{1}{\epsilon_{\bm{w}}^d} \left( 1 + \frac{1}{\epsilon_{\bm{w}}} \right) \epsilon_b \right\}
	\end{eqnarray}
	where $ C $ only depends on $ \phi $, $ \psi $, $ p_{\bm{w}} $, $ p_b $.
	Setting $ \sigma_{\bm{w}} = \epsilon_{\bm{w}}^{-1} $ and $ \sigma_b = \epsilon_b^{-1} $, the main convergence result is obtained.
\end{proof}

\subsubsection{Auxiliary \Cref{lem: FT of pb}}
\label{subsec: aux lem}

The following technical lemma was exploited in the proof of \Cref{thm: recon_mod}:

\begin{lemma} \label{lem: FT of pb}
	For $p_b$ in the setting of \Cref{asmp: ridgelet}, we have that
	\begin{eqnarray*}
		1 - (2 \pi)^{\frac{1}{2}} \| \widehat{p_b} \|_{L^1}^{-1} p_b^\epsilon(b) = \epsilon_b (2 \pi)^{\frac{1}{2}} \| \widehat{p_b} \|_{L^1}^{-1} \int_0^1 b \partial p_b(t \epsilon_b b) \mathrm{d} t .
	\end{eqnarray*}
\end{lemma}
\begin{proof}
	The result will be established by proving (a) $ 1 = (2 \pi)^{\frac{1}{2}} \| \widehat{p_b} \|_{L^1}^{-1} p_b^\epsilon(0) $ and (b) $ p_b(0) - p_b(\epsilon_b b) = \epsilon_b \int_0^1 b \partial p_b(t \epsilon_b b) \mathrm{d} t $, which are algebraically seen to imply the stated result.
	
	\vspace{5pt}
	\noindent 	
	\textbf{Part (a):} 
	Recall that the \emph{Fourier inversion} $ g(\bm{x}) = (2 \pi)^{-\frac{d}{2}} \int_{\mathbb{R}} \widehat{g}(\bm{\xi}) e^{i \bm{\xi} \cdot \bm{x}} \mathrm{d} \bm{\xi} $ holds for any function $ g \in L^1(\mathbb{R}^d) $ s.t. $ \widehat{g} \in L^1(\mathbb{R}^d) $.
	We use the fact $ g(\bm{0}) = (2 \pi)^{-\frac{d}{2}} \| \widehat{g} \|_{L^1} $ for $ g \in L^1(\mathbb{R}^d) $ s.t. $ \widehat{g} \in L^1(\mathbb{R}^d) $ and $ \widehat{g} $ is positive, which is obtained by substituting $ \bm{x} = \bm{0} $ into the Fourier inversion $ (2 \pi)^{-\frac{d}{2}} \int_{\mathbb{R}^d} \widehat{g}(\bm{w}) e^{i \bm{w} \cdot \bm{0}} \mathrm{d} \bm{w} = (2 \pi)^{-\frac{d}{2}} \int_{\mathbb{R}^d} \widehat{g}(\bm{w}) \mathrm{d} \bm{w} $.
	Recall that $ p_b^\epsilon(b) = p_b(\epsilon_b b) $.
	From standard properties of the Fourier transform \cite[p.109,113]{Grafakos2000b}, the Fourier transform of $ b \mapsto p_b(\epsilon_b b) $ is given as $ \xi \mapsto \frac{1}{\epsilon_b} \widehat{p_b}\left(\frac{\xi}{\epsilon_b}\right) $ and $ \widehat{p_b} $ is positive by the assumption. 
	Hence, $ p_b^\epsilon(0) = (2 \pi)^{-\frac{1}{2}} \| \widehat{p_b^\epsilon} \|_{L^1} = (2 \pi)^{-\frac{1}{2}} \left\| \frac{1}{\epsilon_b} \widehat{p_b}\left(\frac{\cdot}{\epsilon_b}\right) \right\|_{L^1} $.
	Since the $ L^1(\mathbb{R}) $ norm is invariant to the scaling of function i.e. $ \left\| \frac{1}{\epsilon_b} \widehat{p_b}\left(\frac{\cdot}{\epsilon_b}\right) \right\|_{L^1} = \| \widehat{p_b} \|_{L^1} $, we obtain $ p_b^\epsilon(0) = (2 \pi)^{-\frac{1}{2}} \| \widehat{p_b} \|_{L^1} $ as required.
	
	\vspace{5pt}
	\noindent 	
	\textbf{Part (b):} 
	We use the fact that the equation $ g(y) - g(x) = \int_0^1 (y - x) \partial g (x - t (y - x)) \mathrm{d} t $ holds for $ g \in C^1(\mathbb{R}) $ \cite[p.302,304]{Gine2015b} to see that
	\begin{eqnarray*}
		p_b(0) - p_b(\epsilon_b b) = \int_0^1 (- \epsilon_b b) \partial p_b((1 - t') \epsilon_b b) \mathrm{d} t' = \epsilon_b \int_0^1 b \partial p_b(t \epsilon_b b) \mathrm{d} t 
	\end{eqnarray*}
	where change of variable $ t = 1 - t' $ is applied.
	This holds since $ p_b \in C^1(\mathbb{R}) $.
\end{proof}


\subsection{Proof of \Cref{thm: recon_rate}}
\label{sec: proof_disc_rate}

This section is dedicated to the proof of \Cref{thm: recon_rate}.
It is divided into three parts; in \Cref{sec: emp} we state technical lemmas that will be useful; in \Cref{subsec: inter res} we state and prove an intermediate result concerning the discretised ridgelet transform, then in \Cref{subsec: proof of GP thm} we present the proof of \Cref{thm: recon_rate}.

\subsubsection{Technical Lemmas} \label{sec: emp}

The following technical lemmas will be useful for the proof of \Cref{thm: recon_rate}.

\begin{lemma} \label{lem: exp_bnd}
	Let $\mathcal{X}$ be a bounded subset of $\mathbb{R}^d$.
	Let $ g: \mathcal{X} \times \mathbb{R}^p \to \mathbb{R} $ be such that $g(\bm{x}, \cdot): \mathbb{R}^p \to \mathbb{R}$ are measurable for all $\bm{x} \in \mathcal{X}$.
	Let $ \bm{\theta}, \bm{\theta}_1, ..., \bm{\theta}_n $ be independent samples from a distribution $ \mathbb{P} $ on $\mathbb{R}^p$.
	Assume that there exists a measurable function $ G: \mathbb{R}^p \to \mathbb{R}  $ such that $ \mathbb{E}[ G(\bm{\theta})^2 ] < \infty $ and
	\begin{eqnarray}
	| g(\bm{x}, \bm{\theta}) - g(\bm{x}', \bm{\theta}) | \le G(\bm{\theta}) \| \bm{x} - \bm{x}' \|_2 \text{ for all } \bm{x}, \bm{x}' \in \mathcal{X}. \label{eq: envelope_cnd}
	\end{eqnarray}
	Then
	\begin{eqnarray}
	\mathbb{E}\left[ \sup_{\bm{x} \in \mathcal{X}} \left| \frac{1}{n} \sum_{i=1}^{n} g(\bm{x}, \bm{\theta}_i) -  \mathbb{E}[ g(\bm{x}, \bm{\theta}) ] \right| \right] \le \frac{C}{\sqrt{n}} \sqrt{ \mathbb{E}\left[ G(\bm{\theta})^2 \right] }. \nonumber
	\end{eqnarray}
	where $ C $ is a constant that depends only on $ \mathcal{X} $.
\end{lemma}
\begin{proof}
	Let $ g_{\bm{x}} := g(\bm{x}, \cdot) $ for shorthand and let $ \mathcal{G} := \{ g_{\bm{x}} \big| \bm{x} \in \mathcal{X} \} $.
	For any $ g_{\bm{x}}, g_{\bm{x}'} \in \mathcal{G} $, define a (random) pseudo metric $ \rho_n(g_{\bm{x}}, g_{\bm{x}'}) := \sqrt{ \frac{1}{n} \sum_{i=1}^{n} \left( g_{\bm{x}}(\bm{\theta}_i) - g_{\bm{x}'}(\bm{\theta}_i) \right)^2 } $ and the diameter $ D_n := \sup_{\bm{x}, \bm{x}' \in \mathcal{X}} \rho_n(g_{\bm{x}}, g_{\bm{x}'}) $.
	Let $ N(\mathcal{G}, \rho_n, \epsilon) $ denotes the \emph{covering number} of the set $ \mathcal{G} $ by $ \epsilon $-ball under the metric $ \rho_n $ \cite[p.41]{Gine2015b}.
	Then by \cite[Theorem~3.5.1, Remark~3.5.2, p.185]{Gine2015b},
	\begin{eqnarray}
	\mathbb{E}\left[ \sup_{\bm{x} \in \mathcal{X}} \left| \frac{1}{n} \sum_{i=1}^{n} g_{\bm{x}}(\bm{\theta}_i) -  \mathbb{E}[ g_{\bm{x}}(\cdot) ] \right| \right] & \le & \frac{8 \sqrt{2}}{\sqrt{n}} \mathbb{E}\bigg[ \underbrace{ \int_{0}^{D_n} \sqrt{ \log 2 N(\mathcal{G}, \rho_n, \epsilon) } \mathrm{d} \epsilon }_{(*)} \bigg] \nonumber 
	\end{eqnarray}
	Here we reduced the assumption $ 0 \in \mathcal{G} $ by the discussion in \cite[p.135]{Wainwright2019}.
	Let $ \| G \|_{\rho_n} = \sqrt{ \frac{1}{n} \sum_{i=1}^{n} G(\bm{\theta}_i)^2 } $.
	By \cite[Lemma~9.18, p.166]{Kosorok2008} and \cite[Example~19.7, p.271]{Vaart1998}, the covering number is bounded as $ N(\mathcal{G}, \rho_n, \epsilon) \le \left( \frac{K \| G \|_{\rho_n}}{\epsilon} \right)^d $ where $ K $ is a constant depending only on $ \mathcal{X} $.
	Hence $ (*) \le \sqrt{d} \int_{0}^{D_n} \sqrt{ \log \frac{2 K \| G \|_{\rho_n}}{\epsilon} } \mathrm{d} \epsilon $.
	By Cauchy-Schwartz inequality,
	\begin{eqnarray}
	\sqrt{d} \int_{0}^{D_n} \sqrt{ \log \frac{2 K \| G \|_{\rho_n}}{\epsilon} } \mathrm{d} \epsilon \le \sqrt{d D_n} \sqrt{ \int_{0}^{D_n} \log \frac{2 K \| G \|_{\rho_n}}{\epsilon} \mathrm{d} \epsilon }. \nonumber
	\end{eqnarray}
	By calculating the integral,
	\begin{eqnarray}
	(*) \le \sqrt{d D_n} \sqrt{ D_n \left( 1 + \log \frac{2 K \| G \|_{\rho_n}}{D_n} \right) } = \sqrt{d} D_n \sqrt{ 1 + \log \frac{2 K \| G \|_{\rho_n}}{D_n} }. \nonumber
	\end{eqnarray}
	By the preceding assumption (\ref{eq: envelope_cnd}),
	\begin{eqnarray}
	D_n = \sup_{\bm{x}, \bm{x}' \in \mathcal{X}} \sqrt{ \frac{1}{n} \sum_{i=1}^{n} \left( g_{\bm{x}}(\bm{\theta}_i) - g_{\bm{x}'}(\bm{\theta}_i) \right)^2 } \le \sup_{\bm{x}, \bm{x}' \in \mathcal{X}} \sqrt{ \frac{1}{n} \sum_{i=1}^{n} G(\bm{\theta}_i)^2 \| \bm{x} - \bm{x}' \|_2^2 } = R \| G \|_{\rho_n}. \nonumber
	\end{eqnarray}
	where $ R = \sup_{\bm{x}, \bm{x}' \in \mathcal{X}} \| \bm{x} - \bm{x}' \|_2$.
	We can set $ K $ so that $ R \le 2 K $ without loss of generality, then we have $ \frac{2 K \| G \|_{\rho_n}}{D_n} \ge 1 \Leftrightarrow \log \frac{2 K \| G \|_{\rho_n}}{D_n} \ge 0 $ which implies
	\begin{eqnarray}
	\sqrt{d} D_n \sqrt{ 1 + \log \frac{2 K \| G \|_{\rho_n}}{D_n} } \le \sqrt{d} D_n \left( 1 + \log \frac{2 K \| G \|_{\rho_n}}{D_n} \right). \nonumber
	\end{eqnarray}
	By the inequality $ 1 + \log z \le z $ for all $ z > 0 $, we have $ 1 + \log \frac{2 K \| G \|_{\rho_n}}{D_n} \le \frac{2 K \| G \|_{\rho_n}}{D_n} $ and 
	\begin{eqnarray}
	(*) \le 2 \sqrt{d} K \| G \|_{\rho_n}. \nonumber
	\end{eqnarray}
	Then by Jensen' inequality,
	\begin{eqnarray}
	\mathbb{E}[ (*) ] = 2 \sqrt{d} K \mathbb{E}\left[ \sqrt{ \frac{1}{n} \sum_{i=1}^{n} G(\bm{\theta}_i)^2 } \right] \le 2 \sqrt{d} K \sqrt{ \mathbb{E}\left[ \frac{1}{n} \sum_{i=1}^{n} G(\bm{\theta}_i)^2 \right] } = 2 \sqrt{d} K \sqrt{ \mathbb{E}\left[ G(\bm{\theta})^2 \right] }. \nonumber
	\end{eqnarray}
	This completes the proof.
\end{proof}

\begin{lemma} \label{lem: bd_ineq}
	For any $ b $-uniformly bounded class of function $ \mathcal{F} $ and any integer $ n \ge 1 $, we have, with probability at least $ 1- \delta $,
	\begin{eqnarray*}
		\sup_{f \in \mathcal{F}} \left| \frac{1}{n} \sum_{i=1}^{n} f(X_i) - \mathbb{E}[ f(X) ] \right| \le \mathbb{E}\left[ \sup_{f \in \mathcal{F}} \left| \frac{1}{n} \sum_{i=1}^{n} f(X_i) - \mathbb{E}[ f(X) ] \right| \right] + \frac{b \sqrt{2 \log \delta^{-1}}}{\sqrt{n}}. 
	\end{eqnarray*}
\end{lemma}
\begin{proof}
	From equation (4.16) in \cite{Wainwright2019}, with probability at least $ 1 - \exp\left( - \frac{n \delta'^2}{2 b^2} \right) $
	\begin{eqnarray*}
		\sup_{f \in \mathcal{F}} \left| \frac{1}{n} \sum_{i=1}^{n} f(X_i) - \mathbb{E}[ f(X) ] \right| \le \mathbb{E}\left[ \sup_{f \in \mathcal{F}} \left| \frac{1}{n} \sum_{i=1}^{n} f(X_i) - \mathbb{E}[ f(X) ] \right| \right] + \delta'. 
	\end{eqnarray*}
	Setting $ \delta' = \frac{b \sqrt{2 \log \delta^{-1}}}{\sqrt{n}} $ yields the result.
\end{proof}

\begin{lemma} \label{lem: QMC}
	Let $S$ be a positive constant and $h \in C^1([-S,S]^d)$.
	For grid points $( \bm{x}_i )_{i=1}^{D}$ on $[-S, S]^d$ corresponding to a Cartesian product of left endpoint rules, we have that
	\begin{eqnarray}
	\left| \int_{[-S,S]^d} h(\bm{x}) \mathrm{d} \bm{x} - \frac{(2S)^d}{D} \sum_{i=1}^{D} h(\bm{x}_i) \right| \le \frac{C}{D^{\frac{1}{d}}} \max_{|\bm{\alpha}| = 1} \sup_{\bm{x} \in [-S,S]^d} | \partial^{\bm{\alpha}} h (\bm{x}) |. \label{eq: q_error}
	\end{eqnarray}
	where $C$ is a constant independent of $h$.
\end{lemma}
\begin{proof}
	Let $r := 2S / \sqrt[d]{D}$.
	For $\bm{x}_i = (x_{i,1}, \dots, x_{i,d})$, let $R_i := [x_{i,1}, x_{i,1} + r] \times \dots \times [x_{i,d}, x_{i,d} + r] $ for $i = 1, \dots, D$.
	Since $( \bm{x}_i )_{i=1}^{D}$ are grid points on $[-S,S]^d$ corresponding to a Cartesian product of left endpoint rules, the domain $[-S,S]^d$ can be decomposed as $[-S,S]^d = R_i \oplus \dots \oplus R_D$, meaning
	\begin{align*}
	\int_{[-S,S]^d} h(\bm{x}) \mathrm{d} \bm{x} = \sum_{i=1}^{D} \int_{R_i} h(\bm{x}) \mathrm{d} \bm{x} .
	\end{align*}
	Denote the original error in \eqref{eq: q_error} by $(*)$.
	Then noting $r^d = \int_{R_i} \mathrm{d} \bm{x}$,
	\begin{align*}
	(*) & = \left| \sum_{i=1}^{D} \int_{R_i} h(\bm{x}) \mathrm{d} \bm{x} - r^d \sum_{i=1}^{D} h(\bm{x}_i) \right| = \left| \sum_{i=1}^{D} \int_{R_i} \left( h(\bm{x}) - h(\bm{x}_i) \right) \mathrm{d} \bm{x} \right|
	\end{align*}
	By the mean value theorem, there exists $\bm{x}_i^*$ for $i = 1, \dots, D$ such that
	\begin{align*}
	(*) & = \left| \sum_{i=1}^{D} \int_{R_i} \nabla h (\bm{x}_i^*) \cdot \left( \bm{x} - \bm{x}_i \right) \mathrm{d} \bm{x} \right|.
	\end{align*}
	where $\nabla h$ is the gradient vector of $h$.
	Calculating the integral and taking supremum of $\nabla h$
	\begin{align*}
	(*) & \le \sum_{i=1}^{D} \frac{d}{2} r^{d+1} \max_{|\bm{\alpha}| = 1} \sup_{\bm{x} \in [-S,S]^d} | \partial^{\bm{\alpha}} h (\bm{x}) | = \frac{d}{2} D r^{d+1} \max_{|\bm{\alpha}| = 1} \sup_{\bm{x} \in [-S,S]^d} | \partial^{\bm{\alpha}} h (\bm{x}) |  .
	\end{align*}
	Substituting $r = 2S / \sqrt[d]{D}$ and setting $C := d (2S)^{d+1} / 2$ concludes the proof.
\end{proof}

\begin{lemma} \label{lem: markov_ineq}
	For a non-negative random variable $ X $ such that $\mathbb{E}[X] < \infty$, it holds with probability at least  $ 1 - \delta $ that $X \le \mathbb{E}[ X ]  / \delta$.
\end{lemma}
\begin{proof}
	From the \emph{Markov inequality} we have that $ \mathbb{P}[ X \ge t ] \le \frac{\mathbb{E}[ X ]}{ t } $. 
	Taking a complement of the probability and setting $ \delta = \frac{\mathbb{E}[ X ]}{ t } $, we obtain the result. 
\end{proof}

\subsubsection{An Intermediate Ridgelet Reconstruction Result} \label{subsec: inter res}

The aim of this section is to establish an analogue of \Cref{thm: recon_mod} that holds when the ridgelet operator $R_\sigma^\ast R$ is discretised using cubature rules, as in \eqref{eq: I_sdn}.
The purpose of \Cref{thm: disc_mod_rp} is to guarantee an accurate reconstruction with high probability.
This will be central to the proof of \Cref{thm: recon_rate}.
We introduce the generalisation of \Cref{asmp: disc_mod_2}.
\begin{assumption} \label{asmp: disc_mod_3}
	For densities $p_{\bm{w}}$ on $\mathbb{R}^d$ and $p_b$ on $\mathbb{R}$ satisfying \Cref{asmp: ridgelet}, define scaled densities $p_{\bm{w}, \sigma}(\bm{w}) := \sigma_{\bm{w}}^{-d} p_{\bm{w}}(\sigma_{\bm{w}}^{-1} \bm{w}) $ and $ p_{b, \sigma}(b) := \sigma_b^{-1} p_b(\sigma_b^{-1} b) $ with scaling constants $0 < \sigma_{\bm{w}} < \infty$ and $0 < \sigma_b < \infty$.
	The cubature nodes $ \{ \bm{w}_i, b_i \}_{i=1}^{N} $ are independently sampled from $p_{\bm{w}, \sigma} \times p_{b, \sigma}$ and $v_i := Z / N$ for all $ i = 1, \dots , N $, where $Z := (2 \pi)^{\frac{1}{2}} \sigma_{\bm{w}}^{d} \sigma_b \| \widehat{p_{\bm{w}}} \|_{L^1}^{-1} \| \widehat{p_b} \|_{L^1}^{-1}$.
\end{assumption}
Note that \Cref{asmp: disc_mod_2} is recovered by setting Gaussians with variances $\sigma_{\bm{w}}^2 \bm{I}_{d \times d}$ and $\sigma_{b}^2$ to the scaled densities $p_{\bm{w}, \sigma}$ and $p_{b, \sigma}$.
Recall $M_1(f) = \max_{|\alpha| \leq 1} \sup_{\bm{x} \in \mathbb{R}^d} |\partial^{\bm{\alpha}}f(\bm{x})|$ and $M_1^{*}(f) = \max_{|\alpha| \leq 1} \sup_{\bm{x} \in [-S,S]^d} |\partial^{\bm{\alpha}}f(\bm{x})|$.

\begin{theorem} \label{thm: disc_mod_rp}
	Let $ I_{\sigma, D, N} $ is given by (\ref{eq: I_sdn}) under Assumption \ref{asmp: disc_mod} and \ref{asmp: disc_mod_3}.
	Further, assume $ \phi $ is $ L_{\phi} $-Lipschitz continuous.
	For any $ f \in C^1(\mathbb{R}^d) $ with $M_{1}^{*}(f) < \infty$, with probability at least $ 1 - \delta $, 
	\begin{multline}
	\sup_{\bm{x} \in \mathcal{X}} \left| f(\bm{x}) - I_{\sigma, D, N} f(\bm{x}) \right| \\
	\le C M_{1}^{*}(f) \left\{ \frac{1}{\sigma_{\bm{w}}} + \frac{\sigma_{\bm{w}}^{d} (\sigma_{\bm{w}} + 1)}{\sigma_{b}} + \frac{\sigma_b \sigma_{\bm{w}}^{d} (\sigma_{\bm{w}} + 1)}{D^{\frac{1}{d}}} + \frac{\sigma_b \sigma_{\bm{w}}^{d} (\sigma_{\bm{w}} + \sqrt{\log \delta^{-1}})}{\sqrt{N}} \right\} \nonumber
	\end{multline}
	where $ C $ is a constant that may depend on $ \mathcal{X}, \mathbbm{1}, \phi, \psi, p_{\bm{w}}, p_b$ but does not depend on $f$, $\sigma_{\bm{w}}, \sigma_b, \delta$.
\end{theorem}

\begin{proof}
	Under \Cref{asmp: disc_mod}, specifically $u_j = (2S)^d D^{-1} \mathbbm{1}(\bm{x}_j)$ and $v_i = Z/N$, we have that
	\begin{eqnarray}
	I_{\sigma, D, N} f (\bm{x}) = \sum_{i=1}^{N} \frac{Z}{N} \left( \frac{(2S)^d}{D} \sum_{j=1}^{D} \mathbbm{1}(\bm{x}_j) f(\bm{x}_j) \psi(\bm{w}_i \cdot \bm{x}_j + b_i) \right) \phi(\bm{w}_i \cdot \bm{x} + b_i). \nonumber
	\end{eqnarray}
	and we formally define
	\begin{eqnarray}
	I_{\sigma} f (\bm{x}) & := & \int_{\mathbb{R}} \int_{\mathbb{R}^d} \left( \int_{\mathbb{R}^d} \mathbbm{1}(\bm{x}') f(\bm{x}') \psi(\bm{w} \cdot \bm{x}' + b) \mathrm{d} \bm{x}' \right) \phi(\bm{w} \cdot \bm{x} + b) Z p_{\bm{w}, \sigma}(\bm{w}) p_{b, \sigma}(b) \mathrm{d} \bm{w} \mathrm{d} b, \nonumber \\
	I_{\sigma, D} f (\bm{x}) & := & \int_{\mathbb{R}} \int_{\mathbb{R}^d} \left( \frac{(2S)^d}{D} \sum_{j=1}^{D} \mathbbm{1}(\bm{x}'_j) f(\bm{x}'_j) \psi(\bm{w} \cdot \bm{x}'_j + b) \right) \phi(\bm{w} \cdot \bm{x} + b) Z p_{\bm{w}, \sigma}(\bm{w}) p_{b, \sigma}(b) \mathrm{d} \bm{w} \mathrm{d} b . \nonumber 
	\end{eqnarray}
	The error will be decomposed by the triangle inequality,
	\begin{align}
	& \hspace{-20pt} \sup_{\bm{x} \in \mathcal{X}} \left| f(\bm{x}) - I_{\sigma, D, N} f(\bm{x}) \right| \nonumber \\
	& \le \underbrace{ \sup_{\bm{x} \in \mathcal{X}} \left| f(\bm{x}) - I_{\sigma} f(\bm{x}) \right| }_{(*)} + \underbrace{ \sup_{\bm{x} \in \mathcal{X}} \left| I_{\sigma} f(\bm{x}) - I_{\sigma, D} f (\bm{x}) \right| }_{(**)} + \underbrace{ \sup_{\bm{x} \in \mathcal{X}} \left| I_{\sigma, D} f (\bm{x}) - I_{\sigma, D, N} f (\bm{x}) \right| }_{(***)}. \label{eq: errors to be bounded}
	\end{align}
	
	\vspace{5pt}
	\noindent \textbf{Bounding $(*)$:} 
	Let $ f_{\text{clip}}(\bm{x}) := \mathbbm{1}(\bm{x}) f(\bm{x}) $.
	Since $f_{\text{clip}}(\bm{x}) = f(\bm{x})$ for all $\bm{x} \in \mathcal{X}$ and our construction ensures that $ I_{\sigma} f (\bm{x}) = R_{\sigma}^{*} R f_{\text{clip}}(\bm{x}) $, we have that
	\begin{eqnarray}
	(*) = \sup_{\bm{x} \in \mathcal{X}} \left| f(\bm{x}) - I_{\sigma} f(\bm{x}) \right| = \sup_{\bm{x} \in \mathcal{X}} \left| f_{\text{clip}}(\bm{x}) - R_{\sigma}^{*} R f_{\text{clip}}(\bm{x}) \right|. \nonumber
	\end{eqnarray}
	In order to apply \Cref{thm: recon_mod} for $(*)$, the two quantities $M_1(f_{\text{clip}})$ and $B(f_{\text{clip}})$ must be shown to be finite.
	For the first quantity, we have
	\begin{align}
	M_1(f_{\text{clip}}) & \le \sup_{\bm{x} \in \mathbb{R}^d} \left| (\mathbbm{1} f)(\bm{x}) \right| + \max_{|\bm{\alpha}| = 1} \sup_{\bm{x} \in \mathbb{R}^d} \left| \partial^{\bm{\alpha}} (\mathbbm{1} f) (\bm{x}) \right|  \nonumber \\
	& = \sup_{\bm{x} \in \mathbb{R}^d} \left| \mathbbm{1}(\bm{x}) f(\bm{x}) \right| + \max_{|\bm{\alpha}|=1} \sup_{\bm{x} \in \mathbb{R}^d} \left| \mathbbm{1}(\bm{x}) \partial^{\bm{\alpha}} f(\bm{x}) + f(\bm{x}) \partial^{\bm{\alpha}} \mathbbm{1}(\bm{x}) \right|  \nonumber \\
	& \leq \sup_{\bm{x} \in \mathbb{R}^d} \left| \mathbbm{1}(\bm{x}) f(\bm{x}) \right| + \max_{|\bm{\alpha}|=1} \sup_{\bm{x} \in \mathbb{R}^d} \left| \mathbbm{1}(\bm{x}) \partial^{\bm{\alpha}} f(\bm{x}) \right| + \max_{|\bm{\alpha}|=1} \sup_{\bm{x} \in \mathbb{R}^d} \left| f(\bm{x}) \partial^{\bm{\alpha}} \mathbbm{1}(\bm{x}) \right|. \nonumber
	\end{align}
	Recall that the infinitely differentiable function $\mathbbm{1}$ has the property $\mathbbm{1}(\bm{x}) = 1$ if $\bm{x} \in \mathcal{X}$ and $\mathbbm{1}(\bm{x}) = 0$ if $\bm{x} \notin [-S,S]^d$, meaning that $\mathbbm{1}(\bm{x}) = 0$ and $\partial^{\bm{\alpha}} \mathbbm{1}(\bm{x}) = 0$ for all $\bm{x} \notin [-S,S]^d$.
	By the assumption $f \in C^1(\mathbb{R}^d)$, the all terms $\mathbbm{1}(\bm{x}) f(\bm{x})$, $\mathbbm{1}(\bm{x}) \partial^{\bm{\alpha}} f(\bm{x})$, and $f(\bm{x}) \partial^{\bm{\alpha}} \mathbbm{1}(\bm{x})$ vanish outside of $\bm{x} \in [-S,S]^d$.
	Therefore the following inequality holds:
	\begin{align}
	M_1(f_{\text{clip}}) & \leq \sup_{\bm{x} \in [-S,S]^d} \left| \mathbbm{1}(\bm{x}) f(\bm{x}) \right| + \max_{|\bm{\alpha}|=1} \sup_{\bm{x} \in [-S,S]^d} \left| \mathbbm{1}(\bm{x}) \partial^{\bm{\alpha}} f(\bm{x}) \right| + \max_{|\bm{\alpha}|=1} \sup_{\bm{x} \in [-S,S]^d} \left| f(\bm{x}) \partial^{\bm{\alpha}} \mathbbm{1}(\bm{x}) \right|. \nonumber \\
	& \leq \sup_{\bm{x} \in [-S,S]^d} \left| \mathbbm{1}(\bm{x}) \right| \times \sup_{\bm{x} \in [-S,S]^d} \left| f(\bm{x}) \right| + \nonumber \\
	& \hspace{20pt} \sup_{\bm{x} \in [-S,S]^d} \left| \mathbbm{1}(\bm{x}) \right| \times \max_{|\bm{\alpha}|=1} \sup_{\bm{x} \in [-S,S]^d} \left| \partial^{\bm{\alpha}} f(\bm{x}) \right| + \sup_{\bm{x} \in [-S,S]^d} \left| f(\bm{x}) \right| \times \max_{|\bm{\alpha}|=1} \sup_{\bm{x} \in [-S,S]^d} \left| \partial^{\bm{\alpha}} \mathbbm{1}(\bm{x}) \right| \nonumber \\
	& = 3 M_{1}^{*}(\mathbbm{1}) M_{1}^{*}(f) < \infty.  \label{eq: Mf bound}
	\end{align}
	The quantity $B_1(f_{\text{clip}})$ is clearly bounded since $f_{\text{clip}}(\bm{x})$ is compactly support on $[-S,S]^d$:
	\begin{align}
	B_1(f_{\text{clip}}) & = \int_{[-S,S]^d} f_{\text{clip}}(\bm{x}) (1 + \| \bm{x} \|_2) \mathrm{d} \bm{x} \nonumber \\
	&\leq \int_{[-S,S]^d} M_1(f_{\text{clip}}) \cdot (1 + S) \mathrm{d} \bm{x} = (1 + S) (2S)^d M_1(f_{\text{clip}}) < \infty. \label{eq: bound_b_1}
	\end{align}
	Thus we may apply \Cref{thm: recon_mod} to obtain
	\begin{align}
	(*) & \leq C_1' \max(M_1(f_{\text{clip}}), B_1(f_{\text{clip}})) \left( \frac{1}{\sigma_{w}} + \frac{\sigma_{\bm{w}}^{d} (\sigma_{\bm{w}} + 1)}{\sigma_{b}} \right) \nonumber
	\end{align}
	for some constant $C_1' > 0$ depending only on $\phi, \psi, p_{\bm{w}}, p_b $.
	Let $C_1 := (1 + S) (2S)^d C_1'$, so that from \eqref{eq: bound_b_1} we have
	\begin{align}
	(*) & \leq C_1 M_1(f_{\text{clip}}) \left( \frac{1}{\sigma_{w}} + \frac{\sigma_{\bm{w}}^{d} (\sigma_{\bm{w}} + 1)}{\sigma_{b}} \right) . \nonumber
	\end{align}
	
	\vspace{5pt}
	\noindent \textbf{Bounding $(**)$:} Let 
	$$
	h_{\bm{w}, b}(\bm{x}) := f_{\text{clip}}(\bm{x}) \frac{\psi(\bm{w} \cdot \bm{x} + b)}{1 + \| \bm{w} \|_2} .
	$$
	Since $f_{\text{clip}}$ vanishes outside $[-S,S]^d$, so does $h_{\bm{w}, b}$ and therefore
	\begin{multline*}
	I_{\sigma} f(\bm{x}) - I_{\sigma, D} f (\bm{x}) = \int_{\mathbb{R}} \int_{\mathbb{R}^d} \bigg\{ \underbrace{ \int_{[-S,S]^d} h_{\bm{w}, b}(\bm{x}') \mathrm{d} \bm{x}' - \frac{2^d}{D} \sum_{j=1}^{D} h_{\bm{w}, b}(\bm{x}_j) }_{\text{(a)}} \bigg\} \times \\
	(1 + \| \bm{w} \|_2) \phi(\bm{w} \cdot \bm{x} + b) Z p_{\bm{w}, \sigma}(\bm{w}) p_{b, \sigma}(b) \mathrm{d} \bm{w} \mathrm{d} b. \nonumber
	\end{multline*}
	
	Our aim here to show that the collection of functions $ h_{\bm{w}, b} $ indexed by $\bm{w}$ and $b$ has derivatives that are uniformly bounded on $[-S,S]^d$, in order that \Cref{lem: QMC} can be applied to (a).
	Let $\bm{\alpha}$ be a multi-index such that $| \bm{\alpha} | = 1$.
	By the chain rule of differentiation, 
	\begin{align*}
	\partial^{\bm{\alpha}} h_{\bm{w}, b}(\bm{x}) & = \Big( \partial^{\bm{\alpha}} f_{\text{clip}}(\bm{x}) \psi(\bm{w} \cdot \bm{x} + b) + f_{\text{clip}}(\bm{x}) \partial^{\bm{\alpha}} \psi(\bm{w} \cdot \bm{x} + b) \Big) \frac{1}{1 + \| \bm{w} \|_2}\\
	& = \Big( \partial^{\bm{\alpha}} f_{\text{clip}}(\bm{x}) \psi(\bm{w} \cdot \bm{x} + b) + f_{\text{clip}}(\bm{x}) \bm{w}^{\bm{\alpha}} ( \partial \psi )(\bm{w} \cdot \bm{x} + b) \Big) \frac{1}{1 + \| \bm{w} \|_2} 
	\end{align*}
	where we recall that $\bm{w}^{\bm{\alpha}} = w_i$ for the non zero element index $i$ of $\bm{\alpha}$ and that $\partial \psi$ is the first derivative of $\psi: \mathbb{R} \to \mathbb{R}$.
	Since $| \bm{\alpha} | = 1$, we have $\bm{w}^{\bm{\alpha}} \le \| \bm{w} \|_2$.
	In addition, $\partial^{\bm{\alpha}} f_{\text{clip}}(\bm{x}) \le M_1(f_{\text{clip}})$ and $f_{\text{clip}}(\bm{x}) \le M_1(f_{\text{clip}})$ for all $\bm{x} \in [-S,S]^d$ by definition.
	Therefore
	\begin{align*}
	\partial^{\bm{\alpha}} h_{\bm{w}, b}(\bm{x}) & \le M_1(f_{\text{clip}}) \left( \frac{1}{1 + \| \bm{w} \|_2} \psi(\bm{w} \cdot \bm{x} + b) + \frac{ \| \bm{w} \|_2 }{1 + \| \bm{w} \|_2} ( \partial \psi )(\bm{w} \cdot \bm{x} + b) \right) .
	\end{align*}
	Since $\psi \in \mathcal{S}(\mathbb{R})$, $\frac{ 1 }{1 + \| \bm{w} \|_2} \le 1$, and $\frac{ \| \bm{w} \|_2 }{1 + \| \bm{w} \|_2} \le 1$ for all $\bm{w} \in \mathbb{R}^d$, we further have
	\begin{align*}
	\partial^{\bm{\alpha}} h_{\bm{w}, b}(\bm{x}) & \le 2 M_1(f_{\text{clip}}) M_1(\psi), 
	\end{align*}
	which is a uniform bound, independent of $\bm{x} \in [-S,S]^d$, $\bm{w} \in \mathbb{R}^d$, $b \in \mathbb{R}$, and $\bm{\alpha}$ such that $| \bm{\alpha} | = 1$.
	Applying \Cref{lem: QMC} for the term (a) with the grid points $ ( \bm{x}_i )_{i=1}^{D} $ on $[-S,S]^d$,
	\begin{align*}
	\left| \int_{[-S,S]^d} h_{\bm{w}, b}(\bm{x}') \mathrm{d} \bm{x}' - \frac{2^d}{D} \sum_{j=1}^{D} h_{\bm{w}, b}(\bm{x}_j) \right| & \le \frac{C_Q'}{D^{\frac{1}{d}}} \max_{|\bm{\alpha}| = 1} \sup_{\bm{x}' \in [-S,S]^d} | \partial^{\bm{\alpha}} h_{\bm{w}, b}(\bm{x}') | \\
	& \le \frac{2 C_Q' M_1(f_{\text{clip}}) M_1(\psi)}{D^{\frac{1}{d}}} .
	\end{align*}
	where $C_Q'$ is a constant independent of $h_{\bm{w}, b}$.
	Therefore setting $C_Q = 2 C_Q'$ we have that
	\begin{multline}
	\left| I_{\sigma} f(\bm{x}) - I_{\sigma, D} f (\bm{x}) \right| \\
	\le \frac{C_{\text{Q}} M_1(f_{\text{clip}}) M_1(\psi) }{D^{\frac{1}{d}}} \underbrace{ \int_{\mathbb{R}} \int_{\mathbb{R}^{d}} \Big| ( 1 + \| \bm{w} \|_2 ) \phi(\bm{w} \cdot \bm{x} + b) Z p_{\bm{w}, \sigma}(\bm{w}) p_{b, \sigma}(b) \Big| \mathrm{d} \bm{w} \mathrm{d} b }_{\text{(b)}} , \label{eq: I_quadrature_bnd}
	\end{multline}
	where we move $C_{\text{Q}} M_1(\psi) M_1(f_{\text{clip}}) $ outside the integral as they are $\bm{w}$ and $b$-independent.
	
	It remains to bound (b).
	By the assumption $\phi \in C_{0}^{*}(\mathbb{R})$,
	\begin{eqnarray}
	\text{(b)} \le Z \| \phi \|_{L^\infty} \int_{\mathbb{R}^d} (1 + \| \bm{w} \|_2) p_{\bm{w}, \sigma}(\bm{w}) \mathrm{d} \bm{w} \int_{\mathbb{R}} p_{b, \sigma}(b) \mathrm{d} b = Z \| \phi \|_{L^\infty} \int_{\mathbb{R}^d} (1 + \| \bm{w} \|_2) p_{\bm{w}, \sigma}(\bm{w}) \mathrm{d} \bm{w}, \nonumber
	\end{eqnarray}
	where $ \int_{\mathbb{R}} p_{b, \sigma}(b) \mathrm{d} b = 1 $ since $ p_{b. \sigma} $ is probability density.
	Since $ p_{\bm{w}, \sigma}(\bm{w}) = \frac{1}{\sigma_{\bm{w}}^d} p_{\bm{w}}\left( \frac{\bm{w}}{\sigma_{\bm{w}}} \right) $, by a change of variable $ \bm{w}' = \sigma_{\bm{w}}^{-1} \bm{w} $,
	\begin{eqnarray}
	\text{(b)} \le Z \| \phi \|_{L^\infty} \int_{\mathbb{R}^d} (1 + \sigma_{\bm{w}} \| \bm{w}' \|_2) p_{\bm{w}}(\bm{w}') \mathrm{d} \bm{w}'. \nonumber 
	\end{eqnarray}
	Since $ p_{\bm{w}} $ has the finite second moment, let $ C_{p_{\bm{w}}} := \max \left(1,  \int_{\mathbb{R}^d} \| \bm{w} \|_2^2 p_{\bm{w}}(\bm{w}) \mathrm{d} \bm{w} \right) $ to see
	\begin{eqnarray}
	\text{(b)} \le Z \| \phi \|_{L^\infty} C_{p_{\bm{w}}} (1 + \sigma_{\bm{w}}). \nonumber 
	\end{eqnarray}
	Recalling $ Z = (2 \pi)^{\frac{1}{2}} \| \widehat{p_{\bm{w}}} \|_{L^1}^{-1} \| \widehat{p_b} \|_{L^1}^{-1} \sigma_{\bm{w}}^{d} \sigma_b $,
	\begin{eqnarray}
	\text{(b)} \le (2 \pi)^{\frac{1}{2}} \| \widehat{p_{\bm{w}}} \|_{L^1}^{-1} \| \widehat{p_b} \|_{L^1}^{-1} \| \phi \|_{L^\infty} C_{p_{\bm{w}}} \sigma_b \sigma_{\bm{w}}^{d} (\sigma_{\bm{w}} + 1) . \nonumber 
	\end{eqnarray}
	By plugging the upper bound of (b) in and setting $ C_2 := (2 \pi)^{\frac{1}{2}} \| \widehat{p_{\bm{w}}} \|_{L^1}^{-1} \| \widehat{p_b} \|_{L^1}^{-1} \| \phi \|_{L^\infty} M_1(\psi) C_{p_{\bm{w}}} C_{\text{Q}} $, we arrive at the overall bound
	\begin{eqnarray}
	\sup_{\bm{x} \in \mathcal{X}} | I_{\sigma} f(\bm{x}) - I_{\sigma, D} f (\bm{x}) | \le C_2 M_1(f_{\text{clip}}) \frac{\sigma_b \sigma_{\bm{w}}^{d} (\sigma_{\bm{w}} + 1)}{D^{\frac{1}{d}} (\log D)^{-1}}. \nonumber
	\end{eqnarray}

	\vspace{5pt}
	\noindent \textbf{Bounding $(***)$:}
	Define $ \tau_{\phi}(\bm{x}, \bm{w}, b) := \left( \frac{(2S)^d}{D} \sum_{j=1}^{D} f_{\text{clip}}(\bm{x}'_j) \psi(\bm{w} \cdot \bm{x}'_j + b) \right) \phi(\bm{w} \cdot \bm{x} + b) $.
	Then the term $(***)$ can be written as
	\begin{eqnarray}
	\sup_{\bm{x} \in \mathcal{X}} \left| I_{\sigma, D} f (\bm{x}) - I_{\sigma, D, N} f (\bm{x}) \right| = Z \sup_{\bm{x} \in \mathcal{X}} \left| \mathop{\mathbb{E}}_{(\bm{w}, b)}[ \tau_{\phi}(\bm{x}, \bm{w}, b) ] - \frac{1}{N} \sum_{i=1}^{N} \tau_{\phi}(\bm{x}, \bm{w}_i, b_i) \right|. \nonumber
	\end{eqnarray}
	We apply \Cref{lem: exp_bnd} and \Cref{lem: bd_ineq} to obtain the upper bound of $(***)$.
	In order to apply Lemma \ref{lem: exp_bnd}, it is to be verified that there exists $ G: \mathbb{R}^{d+1} \to \mathbb{R} $ such that $ \mathop{\mathbb{E}}_{(\bm{w}, b)}[ G(\bm{w}, b)^2 ] < \infty $ and $ | \tau_{\phi}(\bm{x}, \bm{w}, b) - \tau_{\phi}(\bm{x}', \bm{w}, b) | \le G(\bm{w}, b) \| \bm{x} - \bm{x}' \|_2 $.
	
	Recalling that $\| f_{\text{clip}} \|_{L^\infty} \leq M_1(f_{\text{clip}})$ and that $\phi$ was assumed to be Lipschitz with constant denoted $L_\phi$, the difference of $ \tau_{\phi} $ is given by
	\begin{eqnarray*}
		| \tau_{\phi}(\bm{x}, \bm{w}, b) - \tau_{\phi}(\bm{x}', \bm{w}, b) | & = & \left| \left( \frac{(2S)^d}{D} \sum_{j=1}^{D} f_{\text{clip}}(\bm{x}'_j) \psi(\bm{w} \cdot \bm{x}'_j + b) \right) \left( \phi(\bm{w} \cdot \bm{x} + b) - \phi(\bm{w} \cdot \bm{x}' + b) \right) \right|  \\
		& \leq & (2S)^d M_1(f_{\text{clip}}) \| \psi \|_{L^\infty} L_\phi  \left| \bm{w} \cdot (\bm{x} - \bm{x}') \right| \\
		& \leq & (2S)^d M_1(f_{\text{clip}}) \| \psi \|_{L^\infty} L_{\phi} \| \bm{w} \|_2 \| \bm{x} - \bm{x}' \|_2
	\end{eqnarray*}
	where the final inequality used Cauchy-Schwartz.
	Let $ G(\bm{w}, b) := M_1(f_{\text{clip}}) \| \psi \|_{L^\infty} L_{\phi} \| \bm{w} \|_2 $, so that
	\begin{eqnarray}
	\mathop{\mathbb{E}}_{(\bm{w}, b)}[ G(\bm{w}, b)^2 ] & = & (2S)^{2d} M_1(f_{\text{clip}})^2 \| \psi \|_{L^\infty}^2  L_{\phi}^2 \int_{\mathbb{R}^d} \| \bm{w} \|_2^2 p_{\bm{w}, \sigma}(\bm{w}) \mathrm{d} \bm{w} \int_{\mathbb{R}} p_{b, \sigma}(b) \mathrm{d} b \nonumber \\
	& = & (2S)^{2d} M_1(f_{\text{clip}})^2 \| \psi \|_{L^\infty}^2  L_{\phi}^2 \int_{\mathbb{R}^d} \| \bm{w} \|_2^2 p_{\bm{w}, \sigma}(\bm{w}) \mathrm{d} \bm{w}. \nonumber
	\end{eqnarray}
	Since $ p_{\bm{w}, \sigma}(\bm{w}) = \frac{1}{\sigma_{\bm{w}}^d} p_{\bm{w}}\left( \frac{\bm{w}}{\sigma_{\bm{w}}} \right) $, by a change of variable $ \bm{w}' = \sigma_{\bm{w}}^{-1} \bm{w} $, we have $ \int_{\mathbb{R}^d} \| \bm{w} \|_2^2 p_{\bm{w}, \sigma}(\bm{w}) \mathrm{d} \bm{w} = \sigma_{\bm{w}}^2 \int_{\mathbb{R}^d} \| \bm{w}' \|_2^2 p_{\bm{w}}(\bm{w}') \mathrm{d} \bm{w}' $.
	By the assumption that $ p_{\bm{w}} $ has the finite second moment, let $ V_{p_{\bm{w}}}^2 := \int_{\mathbb{R}^d} \| \bm{w} \|_2^2 p_{\bm{w}}(\bm{w}) \mathrm{d} \bm{w} $ to see
	\begin{eqnarray}
	\mathop{\mathbb{E}}_{(\bm{w}, b)}[ G(\bm{w}, b)^2 ] \le \left\{ (2S)^d M_1(f_{\text{clip}}) \| \psi \|_{L^\infty} L_{\phi} V_{p_{\bm{w}}} \sigma_{\bm{w}} \right\}^2. \nonumber
	\end{eqnarray}
	By Lemma \ref{lem: exp_bnd}, for some constant $ C_{\mathcal{X}} $ only depending on $ \mathcal{X} $,
	\begin{align}
	\mathop{\mathbb{E}}\left[ \sup_{\bm{x} \in \mathcal{X}} \left| \mathop{\mathbb{E}}_{(\bm{w}, b)}[ \tau_{\phi}(\bm{x}, \bm{w}, b) ] - \frac{1}{N} \sum_{i=1}^{N} \tau_{\phi}(\bm{x}, \bm{w}_i, b_i) \right| \right] & \le \frac{ C_{\mathcal{X}} \sqrt{\mathop{\mathbb{E}}_{(\bm{w}, b)}[ G(\bm{w}, b)^2 ]} }{\sqrt{N}} \nonumber \\
	& \le \frac{ C_{\mathcal{X}} (2S)^d M_1(f_{\text{clip}}) \| \psi \|_{L^\infty} L_{\phi} V_{p_{\bm{w}}} \sigma_{\bm{w}} }{\sqrt{N}}. \label{eq: exp_I_bnd}
	\end{align}
	
	By the upper bound $ \left( \frac{(2S)^d}{D} \sum_{j=1}^{D} f_{\text{clip}}(\bm{x}'_j) \psi(\bm{w} \cdot \bm{x}'_j + b) \right) \le (2S)^d M_1(f_{\text{clip}}) \| \psi \|_{L^\infty} $ and the assumption $\phi \in C_0^*(\mathbb{R})$, we have $ \tau_{\phi}(\bm{x}, \bm{w}, b) \le (2S)^d M_1(f_{\text{clip}}) \| \psi \|_{L^\infty} \| \phi \|_{L^\infty} $ for all $(\bm{x}, \bm{w}, b) \in \mathbb{R}^d \times \mathbb{R}^d \times \mathbb{R}$.
	From \Cref{lem: bd_ineq}, we have, with probability at least $ 1- \delta $, 
	\begin{eqnarray}
	\sup_{\bm{x} \in \mathcal{X}} \left| I_{\sigma, D} f (\bm{x}) - I_{\sigma, D, N} f (\bm{x}) \right| \le Z (2S)^d M_1(f_{\text{clip}}) \| \psi \|_{L^\infty} \left( \frac{ C_{\mathcal{X}} L_{\phi} V_{p_{\bm{w}}} \sigma_{\bm{w}} }{\sqrt{N}} + \frac{ \sqrt{2} \| \phi \|_{L^\infty} \sqrt{\log \delta^{-1}} }{\sqrt{N}} \right). \nonumber
	\end{eqnarray}
	Let $ C_3 := (2 \pi)^{\frac{1}{2}} (2S)^d \| \widehat{p_{\bm{w}}} \|_{L^1}^{-1} \| \widehat{p_b} \|_{L^1}^{-1} \| \psi \|_{L^\infty} \max\left( C_{\mathcal{X}} L_{\phi} V_{p_{\bm{w}}}, \sqrt{2} \| \phi \|_{L^\infty} \right) $ where we recall that $ Z = (2 \pi)^{\frac{1}{2}} \| \widehat{p_{\bm{w}}} \|_{L^1}^{-1} \| \widehat{p_b} \|_{L^1}^{-1} \sigma_{\bm{w}}^{d} \sigma_b $. 
	Then we have
	\begin{align}
	\sup_{\bm{x} \in \mathcal{X}} \left| I_{\sigma, D} f (\bm{x}) - I_{\sigma, D, N} f (\bm{x}) \right| & \leq C_3 M_1(f_{\text{clip}}) \left( \frac{ \sigma_{b} \sigma_{\bm{w}}^{d} ( \sigma_{\bm{w}} + \sqrt{\log \delta^{-1}} ) }{\sqrt{N}} \right). \nonumber
	\end{align}
	
	For all bounds of $(*)$, $(**)$ and $(***)$, recall that $ M(f_{\text{clip}}) \le 3 M_{1}^{*}(\mathbbm{1}) M_{1}^{*}(f) $ from (\ref{eq: Mf bound}).
	Then combining $(*)$, $(**)$ and $(***)$ and setting $C := 3 M_1^*(\mathbbm{1}) ( C_1 + C_2 + C_3 )$ completes the proof.
\end{proof}

\subsubsection{Proof of Theorem \ref{thm: recon_rate}} \label{subsec: proof of GP thm}

\begin{proof}
	For a bivariate function $g(\bm{x},\bm{y})$ we let $I_{\sigma,D,N}^{\bm{x}} g(\bm{x},\bm{y})$ denote the action of $I_{\sigma,D,N}$ on the first argument of $g$ and we let $I_{\sigma,D,N}^{\bm{y}} g(\bm{x},\bm{y})$ denote the action of $I_{\sigma,D,N}$ on the second argument of $g$.
	To reduce notation, in this proof we denote $\mathbb{E}_{f | \bm{w},b}[ \cdot ] := \mathbb{E}[ \cdot \mid \{ \bm{w}_i, b_i \}_{i=1}^{N} ]$.
	For fixed $\bm{x},\bm{y}$ we let
	\begin{eqnarray*}
		(a) & := & \mathbb{E}_{f | \bm{w},b}\left[ \left( f(\bm{x}) - I_{\sigma, D, N}f(\bm{x})  \right) \left( f(\bm{y}) - I_{\sigma, D, N}f(\bm{y})  \right) \right] \\
		& = & \mathbb{E}_{f | \bm{w},b}[ f(\bm{x}) f(\bm{y})  ] - \mathbb{E}_{f | \bm{w},b}[ I_{\sigma, D, N} f(\bm{x}) f(\bm{y}) ] - \mathbb{E}_{f | \bm{w},b}[ f(\bm{x}) I_{\sigma, D, N}f(\bm{y}) ] \\
		& & \qquad + \mathbb{E}_{f | \bm{w},b}[ I_{\sigma, D, N}f(\bm{x}) I_{\sigma, D, N}f(\bm{y}) ] .
	\end{eqnarray*}
	Recall that $ \mathbb{E}_{f | \bm{w},b}[ f(\bm{x}) f(\bm{y}) ] = k(\bm{x}, \bm{y}) + m(\bm{x}) m(\bm{y}) $ for $ f \sim \mathcal{GP}(m, k) $.
	Then
	\begin{eqnarray*}
		(a) & = & \mathbb{E}_{f | \bm{w},b}[ f(\bm{x}) f(\bm{y}) ] - I_{\sigma, D, N}^{\bm{x}} \mathbb{E}_{f | \bm{w},b}[ f(\bm{x}) f(\bm{y}) ] - I_{\sigma, D, N}^{\bm{y}} \mathbb{E}_{f | \bm{w},b}[ f(\bm{x}) f(\bm{y}) ] \\
		& & \qquad + I_{\sigma, D, N}^{\bm{x}} I_{\sigma, D, N}^{\bm{y}} \mathbb{E}_{f | \bm{w},b}[ f(\bm{x}) f(\bm{y}) ] \\
		& = & \left( k(\bm{x},\bm{y}) + m(\bm{x}) m(\bm{y}) \right) - I_{\sigma,D,N}^{\bm{x}} \left( k(\bm{x},\bm{y}) + m(\bm{x}) m(\bm{y}) \right) \\
		& & \hspace{80pt} - I_{\sigma,D,N}^{\bm{y}} \left( k(\bm{x},\bm{y}) + m(\bm{x}) m(\bm{y}) \right)  + I_{\sigma,D,N}^{\bm{x}} I_{\sigma,D,N}^{\bm{y}} \left( k(\bm{x},\bm{y}) + m(\bm{x}) m(\bm{y}) \right) \\
		& = & \left( m(\bm{x}) m(\bm{y}) - I_{\sigma,D,N} m(\bm{x}) m(\bm{y}) - m(\bm{x}) I_{\sigma,D,N} m(\bm{y}) + I_{\sigma,D,N} m(\bm{x}) I_{\sigma,D,N} m(\bm{y}) \right) \\
		& & \hspace{80pt} + \left( k(\bm{x},\bm{y}) - I_{\sigma,D,N}^{\bm{x}} k(\bm{x},\bm{y}) - I_{\sigma,D,N}^{\bm{y}} k(\bm{x},\bm{y}) + I_{\sigma,D,N}^{\bm{x}} I_{\sigma,D,N}^{\bm{y}} k(\bm{x},\bm{y}) \right) 
	\end{eqnarray*}
	Let $h(\bm{x}, \bm{y}) := k(\bm{x},\bm{y}) - I_{\sigma,D,N}^{\bm{x}} k(\bm{x},\bm{y})$ in order to see
	\begin{eqnarray*}
		(a) & = & ( m(\bm{x}) - I_{\sigma,D,N} m(\bm{x}) ) ( m(\bm{y}) - I_{\sigma,D,N} m(\bm{y}) ) + \left( h(\bm{x}, \bm{y}) - I_{\sigma,D,N}^{\bm{y}} h(\bm{x}, \bm{y}) \right).
	\end{eqnarray*}
	Therefore the error is
	\begin{align*}
	\sup_{\bm{x} \in \mathcal{X}} \sqrt{ \mathbb{E}_{f | \bm{w},b} \left[ \left( f(\bm{x}) - I_{\sigma,D,N} f(\bm{x}) \right)^2 \right] } & \leq \sup_{\bm{x} \in [-S,S]^d} \sqrt{ \mathbb{E}_{f | \bm{w},b} \left[ \left( f(\bm{x}) - I_{\sigma,D,N} f(\bm{x}) \right)^2 \right] } \nonumber \\
	& \hspace{-30pt} \leq \underbrace{ \sup_{\bm{x} \in [-S,S]^d} \left| m(\bm{x}) - I_{\sigma,D,N} m(\bm{x}) \right| }_{(*)} + \underbrace{ \sup_{\bm{x} \in [-S,S]^d} \left| h(\bm{x}, \bm{x}) - I_{\sigma,D,N}^{\bm{y}} h(\bm{x}, \bm{x}) \right|^{\frac{1}{2}} }_{(**)}
	\end{align*}
	In the remainder we bound $(*)$ and $(**)$.
	
	\vspace{5pt}
	\noindent \textbf{Bounding $ (*) $:} 
	Applying Theorem \ref{thm: disc_mod_rp}, we immediately have, with probability at least $ 1 - \delta $ with respect to the random variables $\{\bm{w}_i, b_i\}_{i=1}^N$,
	\begin{eqnarray}
	(*) \le C_1 M_1^*(m) \left( \frac{1}{\sigma_{\bm{w}}} + \frac{\sigma_{\bm{w}}^{d} (\sigma_{\bm{w}} + 1)}{\sigma_{b}} + \frac{\sigma_b \sigma_{\bm{w}}^{d} (\sigma_{\bm{w}} + 1)}{D^{\frac{1}{d}}} + \frac{\sigma_b \sigma_{\bm{w}}^{d} (\sigma_{\bm{w}} + \sqrt{\log \delta^{-1}})}{\sqrt{N}} \right) \nonumber
	\end{eqnarray}
	for some constant $ C_1 $ independent of $\delta, \sigma_{\bm{x}} , \sigma_b, D, N$ and $m$.

	\vspace{5pt}
	\noindent \textbf{Bounding $ (**) $:} 
	It is clear to see
	\begin{eqnarray}
	(**) = \left\{ \sup_{\bm{x} \in [-S,S]^d} \left| h(\bm{x}, \bm{x}) - I_{\sigma,D,N}^{\bm{y}} h(\bm{x}, \bm{x}) \right| \right\}^{\frac{1}{2}} \leq \left\{ \sup_{\bm{x} \in [-S,S]^d} \sup_{\bm{y} \in [-S,S]^d} \left| h(\bm{x}, \bm{y}) - I_{\sigma,D,N}^{\bm{y}} h(\bm{x}, \bm{y}) \right| \right\}^{\frac{1}{2}} \label{eq: I_k_error}
	\end{eqnarray}
	First, with respect to the supremum of $\bm{y}$, from Theorem \ref{thm: disc_mod_rp}, with probability at least $ 1 - \delta $ with respect to the random variables $\{\bm{w}_i, b_i\}_{i=1}^N$,
	\begin{multline}
	\sup_{\bm{y} \in [-S,S]^d} \left| h(\bm{x}, \bm{y}) - I_{\sigma,D,N}^{\bm{y}} h(\bm{x}, \bm{y}) \right| \\
	\leq C_2' M_1^*(h(\bm{x}, \cdot)) \left( \frac{1}{\sigma_{\bm{w}}} + \frac{\sigma_{\bm{w}}^{d} (\sigma_{\bm{w}} + 1)}{\sigma_{b}} + \frac{\sigma_b \sigma_{\bm{w}}^{d} (\sigma_{\bm{w}} + 1)}{D^{\frac{1}{d}}} + \frac{\sigma_b \sigma_{\bm{w}}^{d} (\sigma_{\bm{w}} + \sqrt{\log \delta^{-1}})}{\sqrt{N}} \right) \label{eq: I_k_bound}
	\end{multline}
	where $M_1^*(h(\bm{x}, \cdot))$ is given as
	\begin{align}
	M_1^*(h(\bm{x}, \cdot)) = \max_{|\bm{\beta}| \leq 1} \sup_{\bm{y} \in [-S,S]^d} | \partial^{\bm{0}, \bm{\beta}} h (\bm{x}, \bm{y})|. \label{eq: I_k_bound_M}
	\end{align}
	Second, $\left| M_1^*(h(\bm{x}, \cdot)) \right|$ is to be bounded.
	Recall $h(\bm{x}, \bm{y}) = k(\bm{x},\bm{y}) - I_{\sigma,D,N}^{\bm{x}} k(\bm{x},\bm{y})$ to see $\partial^{\bm{0}, \bm{\beta}}h(\bm{x}, \bm{y}) = \partial^{\bm{0}, \bm{\beta}}k(\bm{x},\bm{y}) - I_{\sigma,D,N}^{\bm{x}} \partial^{\bm{0}, \bm{\beta}}k(\bm{x},\bm{y})$.
	For fixed $\bm{y} \in [-S,S]^d$ and $|\beta| \leq 1$, $\partial^{\bm{0}, \bm{\beta}}h(\bm{x}, \bm{y})$ is upper bounded for all $\bm{x} \in [-S,S]^d$ from Theorem \ref{thm: disc_mod_rp}: with probability at least $1 - \delta$ with respect to the random variables $\{\bm{w}_i, b_i\}_{i=1}^N$,
	\begin{align*}
	\partial^{\bm{0}, \bm{\beta}}h(\bm{x}, \bm{y}) & \leq \sup_{\bm{x} \in [-S,S]^d} \left| \partial^{\bm{0}, \bm{\beta}}h(\bm{x}, \bm{y}) \right| \\
	& = \sup_{\bm{x} \in [-S,S]^d} \left| \partial^{\bm{0}, \bm{\beta}}k(\bm{x},\bm{y}) - I_{\sigma,D,N}^{\bm{x}} \partial^{\bm{0}, \bm{\beta}}k(\bm{x},\bm{y}) \right| \\
	& \leq C_2'' M_1^*(\partial^{\bm{0}, \bm{\beta}}k(\cdot, \bm{y})) \left( \frac{1}{\sigma_{\bm{w}}} + \frac{\sigma_{\bm{w}}^{d} (\sigma_{\bm{w}} + 1)}{\sigma_{b}} + \frac{\sigma_b \sigma_{\bm{w}}^{d} (\sigma_{\bm{w}} + 1)}{D^{\frac{1}{d}}} + \frac{\sigma_b \sigma_{\bm{w}}^{d} (\sigma_{\bm{w}} + \sqrt{\log \delta^{-1}})}{\sqrt{N}} \right).
	\end{align*}
	where $M_1^*(\partial^{\bm{0}, \bm{\beta}}k(\cdot, \bm{y})) = \max_{|\bm{\alpha}| \leq 1} \sup_{\bm{x} \in [-S,S]^d} |\partial^{\bm{\alpha}, \bm{\beta}}k(\bm{x}, \bm{y})|$.
	Plugging this upper bound into \eqref{eq: I_k_bound_M}, with probability at least $1 - \delta$ with respect to the random variables $\{\bm{w}_i, b_i\}_{i=1}^N$,
	\begin{multline*}
	M_1^*(h(\bm{x}, \cdot)) \leq C_2'' \max_{|\bm{\beta}| \leq 1} \sup_{\bm{y} \in [-S,S]^d} \left| \max_{|\bm{\alpha}| \leq 1} \sup_{\bm{x} \in [-S,S]^d} \left| \partial^{\bm{\alpha}, \bm{\beta}}k(\bm{x}, \bm{y}) \right| \right| \\
	\left( \frac{1}{\sigma_{\bm{w}}} + \frac{\sigma_{\bm{w}}^{d} (\sigma_{\bm{w}} + 1)}{\sigma_{b}} + \frac{\sigma_b \sigma_{\bm{w}}^{d} (\sigma_{\bm{w}} + 1)}{D^{\frac{1}{d}}} + \frac{\sigma_b \sigma_{\bm{w}}^{d} (\sigma_{\bm{w}} + \sqrt{\log \delta^{-1}})}{\sqrt{N}} \right).
	\end{multline*}
	Recall $M_1^{*}(k) := \max_{|\bm{\alpha}| \leq 1, |\bm{\beta}| \leq 1} \sup_{\bm{x}, \bm{y} \in [-S,S]^d} |\partial^{\bm{\alpha}, \bm{\beta}}k(\bm{x}, \bm{y})|$ to see, for all $\bm{x} \in [-S,S]^d$, and with probability at least $1 - \delta$ with respect to the random variables $\{\bm{w}_i, b_i\}_{i=1}^N$,
	\begin{align*}
	M_1^*(h(\bm{x}, \cdot)) \leq C_2'' M_1^{*}(k) \left( \frac{1}{\sigma_{\bm{w}}} + \frac{\sigma_{\bm{w}}^{d} (\sigma_{\bm{w}} + 1)}{\sigma_{b}} + \frac{\sigma_b \sigma_{\bm{w}}^{d} (\sigma_{\bm{w}} + 1)}{D^{\frac{1}{d}}} + \frac{\sigma_b \sigma_{\bm{w}}^{d} (\sigma_{\bm{w}} + \sqrt{\log \delta^{-1}})}{\sqrt{N}} \right).
	\end{align*}
	Notice that $M_1^{*}(k) < \infty$ by the assumption $k \in C^{1 \times 1}(\mathbb{R}^d \times \mathbb{R}^d)$.
	Combining this upper bound with \eqref{eq: I_k_bound} and taking the supremum over $\bm{x} \in [-S,S]^d$, from \eqref{eq: I_k_error} we have, with probability at least $1 - \delta$ with respect to the random variables $\{\bm{w}_i, b_i\}_{i=1}^N$,
	\begin{align*}
	(**) \le C_2 \sqrt{M_1^{*}(k)} \left( \frac{1}{\sigma_{\bm{w}}} + \frac{\sigma_{\bm{w}}^{d} (\sigma_{\bm{w}} + 1)}{\sigma_{b}} + \frac{\sigma_b \sigma_{\bm{w}}^{d} (\sigma_{\bm{w}} + 1)}{D^{\frac{1}{d}}} + \frac{\sigma_b \sigma_{\bm{w}}^{d} (\sigma_{\bm{w}} + \sqrt{\log \delta^{-1}})}{\sqrt{N}} \right).
	\end{align*}
	where $C_2 := \sqrt{C_2' C_2''}$.
	
	Combining these bounds on $(*)$ and $(**)$ completes the proof.
\end{proof}


\subsection{An Analogous Result to \Cref{thm: recon_rate} for Unbounded $\phi$}
\label{sec: proof_lemma_unbounded}

In this section we state and prove an analogous result of \Cref{thm: recon_rate} that holds under weaker assumptions on the activation function $\phi$, with a correspondingly stronger assumption on the GP.
This ensures that our theory is compatible with activation functions $\phi$ that may be unbounded, including the ReLU activation function $\phi(x) = \max(0,x)$.
An example of the associated pair function $\psi$ is given in \Cref{ex: ridgelet_func}.

First of all, we recall that a positive semi-definite function $k : \mathbb{R}^d \times \mathbb{R}^d \rightarrow \mathbb{R}$ \textit{reproduces} a Hilbert space $\mathcal{H}_k$ whose elements are functions $f : \mathbb{R}^d \rightarrow \mathbb{R}$ such that $k(\cdot,\bm{x}) \in \mathcal{H}_k$ for all $\bm{x} \in \mathbb{R}^d$ and $\langle f , k(\cdot,\bm{x}) \rangle_{\mathcal{H}_k} = f(\bm{x})$ for all $\bm{x} \in \mathbb{R}^d$ and all $f \in \mathcal{H}_k$.
The Hilbert space $\mathcal{H}_k$ is called a \textit{reproducing kernel Hilbert space} (RKHS) .
It is a well-known fact that GP sample paths are not contained in the RKHS induced by the GP covariance kernel i.e. $f \sim \mathcal{GP}(m,k) \implies \mathbb{P}(f \in \mathcal{H}_k) = 0$, unless $\mathcal{H}_k$ is finite dimensional \cite[Theorem 7.5.4]{Hsing2015b}.
However, $\mathbb{P}(f \in \mathcal{H}_R) = 1$ holds whenever $\mathcal{H}_R$ satisfies \emph{nuclear dominance} over $\mathcal{H}_k$; see \cite{Lukic2001}.
The additional assumption that we require on the GP in this appendix is that the GP takes values in a RKHS $\mathcal{H}_R$ where $R$ is continuously differentiable.
Intuitively, this imposes an additional smoothness requirement on the GP compared to \Cref{thm: recon_rate}.
Our analogous result to \Cref{thm: recon_rate} is as follows:

\begin{theorem}[Analogue of \Cref{thm: recon_rate} for Unbounded $\phi$] \label{lem: alt_unbnd}
	In the same setting of \Cref{thm: recon_rate}, replace the assumption $\phi \in C_0^*(\mathbb{R})$ with $\phi \in C_1^*(\mathbb{R})$.
	In addition, assume that $f \sim \mathcal{GP}(m, k)$ is a random variable taking values in $\mathcal{H}_R$ with the reproducing kernel $R \in C^{1 \times 1}(\mathbb{R}^d \times \mathbb{R}^d)$.
	Assume that $m \in \mathcal{H}_R$ and that the covariance operator $\mathcal{K}$ of $f$ is trace class.
	Then, with probability at least $ 1 - \delta $, 
	\begin{multline*}
	\sup_{\bm{x} \in \mathcal{X}} \sqrt{ \mathop{\mathbb{E}} \left[ \left( f(\bm{x}) - I_{\sigma,D,N} f (\bm{x}) \right)^2  | \{\bm{w}_i,b_i\}_{i=1}^N \right] } \\
	\le C \sqrt{ M_1^*(R) } 
	\left( \| m \|_{\mathcal{H}_R} + \sqrt{ \text{\normalfont tr}(\mathcal{K}) } \right)
	\left\{ \frac{1}{\sigma_{\bm{w}}} + \frac{\sigma_{\bm{w}}^{d} (\sigma_{\bm{w}} + 1)^2}{\sigma_{b}} + \frac{\sigma_b (\sigma_{b} + 1) \sigma_{\bm{w}}^{d} (\sigma_{\bm{w}} + 1)^2}{D^{\frac{1}{d}}} + \frac{\sigma_b \sigma_{\bm{w}}^{d+1}}{\delta \sqrt{N}} \right\}.
	\end{multline*}
	where $\text{\normalfont tr}(\mathcal{K})$ is the trace of the operator $\mathcal{K}$ and $ C $ is a constant independent of $m, k, \sigma_{\bm{w}}, \sigma_b, D, N, \delta $.
\end{theorem}
\begin{proof}
	To reduce notation, in this proof we denote $\mathbb{E}_{f | \bm{w},b}[ \cdot ] := \mathbb{E}[ \cdot \mid \{ \bm{w}_i, b_i \}_{i=1}^{N} ]$.
	By Jensen's inequality, we have
	\begin{align}
	\sup_{\bm{x} \in \mathcal{X}} \sqrt{ \mathbb{E}_{f | \bm{w},b}\left[ \left( f(\bm{x}) - I_{\sigma, D, N}f(\bm{x})  \right)^2 \right] } \leq \sqrt{ \mathbb{E}_{f | \bm{w},b}\bigg[ \bigg(  \sup_{\bm{x} \in \mathcal{X}} \Big| f(\bm{x}) - I_{\sigma, D, N}f(\bm{x}) \Big| \bigg)^2 \bigg] }  \label{eq: looser cdn}
	\end{align}
	By \Cref{lem: alt_2} we have, with probability $1 - \delta$ with respect to the random variables $\{\bm{w}_i, b_i\}_{i=1}^N$, the right hand side of \eqref{eq: looser cdn} can be bounded as
	\begin{align*}
	& \leq \sqrt{ \mathbb{E}_{f | \bm{w},b}\left[ C^2 M_{1}^{*}(f)^2 \left\{ \frac{1}{\sigma_{\bm{w}}} + \frac{\sigma_{\bm{w}}^{d} (\sigma_{\bm{w}} + 1)}{\sigma_{b}} + \frac{\sigma_b (\sigma_{b} + 1) \sigma_{\bm{w}}^{d} (\sigma_{\bm{w}} + 1)^2}{D^{\frac{1}{d}}} + \frac{\sigma_b \sigma_{\bm{w}}^{d+1}}{\delta \sqrt{N}} \right\}^2 \right] } \\
	& = C \sqrt{ \mathbb{E}_{f | \bm{w},b}\left[ M_{1}^{*}(f)^2 \right] } \left\{ \frac{1}{\sigma_{\bm{w}}} + \frac{\sigma_{\bm{w}}^{d} (\sigma_{\bm{w}} + 1)}{\sigma_{b}} + \frac{\sigma_b (\sigma_{b} + 1) \sigma_{\bm{w}}^{d} (\sigma_{\bm{w}} + 1)^2}{D^{\frac{1}{d}}} + \frac{\sigma_b \sigma_{\bm{w}}^{d+1}}{\delta \sqrt{N}} \right\} ,
	\end{align*}
	where $ C $ is a constant independent on $f$, $\sigma_{\bm{w}}, \sigma_b, \delta$.
	Next we will upper-bound $\sqrt{ \mathbb{E}_{f | \bm{w},b}\left[ M_{1}^{*}(f)^2 \right] }$.
	Since $f$ is a $\mathcal{H}_R$-valued random variable, by the reproducing property of $\mathcal{H}_R$
	\begin{align*}
	\mathbb{E}_{f | \bm{w},b}\left[ M_{1}^{*}(f)^2 \right] = \mathbb{E}_{f | \bm{w},b}\left[ \left( \max_{|\bm{\alpha}| \leq 1} \sup_{\bm{x} \in [-S,S]^d} \partial^{\bm{\alpha}} f(\bm{x}) \right)^2 \right] = \mathbb{E}_{f | \bm{w},b}\left[ \left( \max_{|\bm{\alpha}| \leq 1} \sup_{\bm{x} \in [-S,S]^d} \langle f, \partial^{\bm{\alpha}, \bm{0}} R(\bm{x}, \cdot) \rangle_{\mathcal{H}_R} \right)^2 \right].
	\end{align*}
	By the Cauchy-Schwartz inequality,
	\begin{align*}
	\mathbb{E}_{f | \bm{w},b}\left[ \left( \max_{|\bm{\alpha}| \leq 1} \sup_{\bm{x} \in [-S,S]^d} \langle f, \partial^{\bm{\alpha}, \bm{0}} R(\bm{x}, \cdot) \rangle_{\mathcal{H}_R} \right)^2 \right] & \le \mathbb{E}_{f | \bm{w},b}\left[ \left( \max_{|\bm{\alpha}| \leq 1} \sup_{\bm{x} \in [-S,S]^d} \| f \|_{\mathcal{H}_R} \| \partial^{\bm{\alpha}, \bm{0}} R(\bm{x}, \cdot) \|_{\mathcal{H}_R} \right)^2 \right] \\
	& = \mathbb{E}_{f | \bm{w},b}\left[ \max_{|\bm{\alpha}| \leq 1} \sup_{\bm{x} \in [-S,S]^d} \| f \|_{\mathcal{H}_R}^2 \partial^{\bm{\alpha}, \bm{\alpha}} R(\bm{x}, \bm{x}) \right] \\
	&\le M_1^*(R) \mathbb{E}_{f | \bm{w},b}\left[ \| f \|_{\mathcal{H}_R}^2 \right].
	\end{align*}
	From \cite[(1.13)]{Prato2006b}, $\mathbb{E}_{f | \bm{w},b}\left[ \| f \|_{\mathcal{H}_R}^2 \right] = \| m \|_{\mathcal{H}_R}^2 + \text{tr}(\mathcal{K})$.
	Therefore, we have
	\begin{align*}
	\sqrt{ \mathbb{E}_{f | \bm{w},b}\left[ M_{1}^{*}(f)^2 \right] } \le \sqrt{ M_1^*(R) \left( \| m \|_{\mathcal{H}_R}^2 + \text{tr}(\mathcal{K}) \right) } \le \sqrt{ M_1^*(R) } \left( \| m \|_{\mathcal{H}_R} + \sqrt{ \text{tr}(\mathcal{K}) } \right)
	\end{align*}
	where the fact that $\sqrt{a + b} \le \sqrt{a} + \sqrt{b}$ for $a, b \in \mathbb{R}$ is applied for the last inequality.
	Plugging in this upper bound concludes the proof.
\end{proof}

In the remaining part of this section, we show \Cref{lem: alt_1} and \Cref{lem: alt_2} which are analogous results to \Cref{thm: recon_mod} and \Cref{thm: disc_mod_rp} for $\phi \in C_1^*(\mathbb{R})$.
The assumption $\phi \in C_1^*(\mathbb{R})$ implies for some $ C_{\phi} < \infty $
\begin{align}
\phi(\bm{w} \cdot \bm{x} + b) & \le C_{\phi} (1 + | \bm{w} \cdot \bm{x} + b | ) \nonumber \\
& \le C_{\phi} (1 + \| \bm{w} \|_2 \| \bm{x} \|_2 + | b | ) \nonumber \\
& \le C_{\phi} (1 + \| \bm{x} \|_2) (1 + \| \bm{w} \|_2) (1 + | b |) \label{eq: phi_bound}
\end{align}
where Cauchy-Schwartz inequality is applied for the second inequality.
The same discussions in the proof of \Cref{thm: recon_mod} and \Cref{thm: disc_mod_rp} holds by replacing all bounds involving $\phi$ with an expression similar to \eqref{eq: phi_bound}.
Let $B_2(f) := \int_{\mathbb{R}^d} |f(\bm{x})| (1 + \|\bm{x}\|_2)^2 \mathrm{d}\bm{x}$.

\begin{theorem}[Analogue of \Cref{thm: recon_mod} for Unbounded $\phi$] \label{lem: alt_1}
	Let $ \mathcal{X} \subset \mathbb{R}^d $ be bounded.
	Let \Cref{assm: ridgelet pair} and \Cref{asmp: ridgelet} hold, but with $\phi \in C_0^*(\mathbb{R})$ replaced with $\phi \in C_1^*(\mathbb{R})$, and let $ f \in C^1(\mathbb{R}^d) $ satisfy $M_1(f) < \infty$ and $ B_2(f) < \infty $.
	Then 
	\begin{eqnarray}
	\sup_{\bm{x} \in \mathcal{X}} \left| f(\bm{x}) - (R_\sigma^{\ast} R)[ f ](\bm{x}) \right| \le C \max(M_1(f),B_2(f)) \left\{ \frac{1}{\sigma_{\bm{w}}} + \frac{\sigma_{\bm{w}}^{d} (\sigma_{\bm{w}} + 1)^2}{\sigma_{b}} \right\}
	\end{eqnarray}
	for some constant $C $ that is independent of $\sigma_{\bm{w}}, \sigma_{b}$ and $f$, but may depend on $ \phi $, $ \psi $, $ p_{\bm{w}} $ and $ p_b $.
\end{theorem}

\begin{proof}
	We use the same proof as \Cref{thm: recon_mod}, but consider the supremum error over a bounded domain $\mathcal{X} \subset \mathbb{R}^d$ i.e. $\sup_{\bm{x} \in \mathcal{X}} \left| f(\bm{x}) - (R_\sigma^{\ast} R)[ f ](\bm{x}) \right|$, instead of error over $\mathbb{R}^d$.
	Recall the overall structure of the proof is \eqref{eq: conv_tri} and in particular there are two quantities, $(*)$ and $(**)$ to be bounded. The argument used to bound $(*)$ remains valid, so our attention below turns to the argument used to bound $(**)$.
	
	In order to establish a bound on $(**)$, we replace the upper bound involving $\phi$ subsequent to (\ref{eq: bound_of_3}).
	From (\ref{eq: bound_of_3}), we have
	\begin{eqnarray}
	|(3)| & \le & \| \partial p_b \|_{L^\infty} (1 + \| \bm{x}' \|_2) (1 + \| \bm{w} \|_2) \int_{\mathbb{R}} (1 + | b' |) | \phi(\bm{w} \cdot (\bm{x} - \bm{x}') + b') \psi(b') | \mathrm{d} b'. \label{eq: updated_3}
	\end{eqnarray}
	From (\ref{eq: phi_bound}), for some $ C_{\phi} < \infty $,
	\begin{eqnarray}
	\phi(\bm{w} \cdot (\bm{x} - \bm{x}') + b') & \le & C_{\phi} (1 + \| \bm{x} - \bm{x}' \|_2) (1 + \| \bm{w} \|_2) (1 + | b |) \nonumber \\
	& \le & C_{\phi} (1 + \| \bm{x} \|_2) (1 + \| \bm{x}' \|_2)  (1 + \| \bm{w} \|_2) (1 + | b |). \nonumber
	\end{eqnarray}
	Plugging this upper bound in \eqref{eq: updated_3},
	\begin{eqnarray}
	|(3)| & \le & C_{\phi} \| \partial p_b \|_{L^\infty} (1 + \| \bm{x} \|_2) (1 + \| \bm{x}' \|_2)^2 (1 + \| \bm{w} \|_2)^2 \int_{\mathbb{R}} (1 + | b' |)^2 | \psi(b') | \mathrm{d} b'. \label{eq: updated_3_2} 
	\end{eqnarray}
	Let $ C_{\psi} := \int_{\mathbb{R}} (1 + | b |)^2 | \psi(b) | \mathrm{d} b $ and $C_{p_{\bm{w}}} := \max\left( 1, 2 \int_{\mathbb{R}^{d}} \| \bm{w} \|_2 p_{\bm{w}}(\bm{w}) \mathrm{d} \bm{w}, \int_{\mathbb{R}^{d}} \| \bm{w} \|_2^2 p_{\bm{w}}(\bm{w}) \mathrm{d} \bm{w} \right)$.
	By the same discussion in the proof of \Cref{thm: recon_mod} subsequent to (\ref{eq: bound_of_3}), considering the difference between \eqref{eq: bound_of_3_2} and \eqref{eq: updated_3_2},
	\begin{eqnarray}
	\Delta k(\bm{x}, \bm{x}') \le (2 \pi)^{\frac{1}{2}} \| \widehat{p_{\bm{w}}} \|_{L^1}^{-1} \| \widehat{p_b} \|_{L^1}^{-1} \|\partial p_b\|_{L^\infty} \|\phi\|_{L^\infty} C_{\psi} C_{p_{\bm{w}}} (1 + \| \bm{x} \|_2) (1 + \| \bm{x}' \|_2)^2 \frac{1}{\epsilon_{\bm{w}}^d} \left( 1 + \frac{1}{\epsilon_{\bm{w}}} \right)^2 \epsilon_b. \nonumber
	\end{eqnarray}
	Set $ C_2' := (2 \pi)^{\frac{1}{2}} \| \widehat{p_{\bm{w}}} \|_{L^1}^{-1} \| \widehat{p_b} \|_{L^1}^{-1} \|\partial p_b\|_{L^\infty} \|\phi\|_{L^\infty} C_{\psi} C_{p_{\bm{w}}} $.
	The error term $(**)$ is then bounded as
	\begin{eqnarray}
	(**) & = & \sup_{\bm{x} \in \mathcal{X}} \left| \int_{\mathbb{R}^{d}} f(\bm{x}') \Delta k(\bm{x}, \bm{x}') \mathrm{d} \bm{x}' \right| \nonumber \\
	& \le & C_2' \sup_{\bm{x} \in \mathcal{X}} \big| 1 + \| \bm{x} \|_2 \big| \int_{\mathbb{R}^{d}} (1 + \| \bm{x}' \|_2)^2 | f(\bm{x}') | \mathrm{d} \bm{x}' \frac{1}{\epsilon_{\bm{w}}^d} \left( 1 + \frac{1}{\epsilon_{\bm{w}}} \right)^2 \epsilon_b. \nonumber
	\end{eqnarray}
	Setting $C_2 := C_2' \sup_{\bm{x} \in \mathcal{X}} \big| 1 + \| \bm{x} \|_2 \big|$, we have
	\begin{eqnarray}
	(**) \le C_2 B_2(f) \frac{1}{\epsilon_{\bm{w}}^d} \left( 1 + \frac{1}{\epsilon_{\bm{w}}} \right)^2 \epsilon_b.
	\end{eqnarray}
	
	Combining  the bound for $(*)$ and $(**)$ and setting $C := \max(C_1,C_2)$,
	\begin{eqnarray}
	\sup_{\bm{x} \in \mathcal{X}} | f(\bm{x}) - R^{\ast}_{\sigma} R f(\bm{x}) | \le C \max(M_1(f), B_2(f)) \left\{ \epsilon_{\bm{w}} + \frac{1}{\epsilon_{\bm{w}}^d} \left( 1 + \frac{1}{\epsilon_{\bm{w}}} \right)^2 \epsilon_b \right\} \nonumber
	\end{eqnarray}
	where $ C $ only depends on $ \mathcal{X} $, $ \phi $, $ \psi $, $ p_{\bm{w}} $, $ p_b $ but not on $f, \epsilon_{\bm{w}}, \epsilon_{b}$.
	Setting $ \sigma_{\bm{w}} = \epsilon_{\bm{w}}^{-1} $ and $ \sigma_b = \epsilon_b^{-1} $ completes the proof.
\end{proof}

\begin{theorem}[Analogue of \Cref{thm: disc_mod_rp} for Unbounded $\phi$] \label{lem: alt_2}
	In the same setting of \Cref{thm: disc_mod_rp}, but with the assumption $\phi \in C_0^*(\mathbb{R})$ replaced with $\phi \in C_1^*(\mathbb{R})$, for any $ f \in C^1(\mathbb{R}^d) $ with $M_{1}^{*}(f) < \infty$, with probability at least $ 1 - \delta $, 
	\begin{multline}
	\sup_{\bm{x} \in \mathcal{X}} \left| f(\bm{x}) - I_{\sigma, D, N} f(\bm{x}) \right| \le \\
	C M_{1}^{*}(f) \left\{ \frac{1}{\sigma_{\bm{w}}} + \frac{\sigma_{\bm{w}}^{d} (\sigma_{\bm{w}} + 1)^2}{\sigma_{b}} + \frac{\sigma_b (\sigma_{b} + 1) \sigma_{\bm{w}}^{d} (\sigma_{\bm{w}} + 1)^2}{D^{\frac{1}{d}}} + \frac{\sigma_{b} \sigma_{\bm{w}}^{d+1}}{\delta \sqrt{N}} \right\} \nonumber
	\end{multline}
	where $ C $ is a constant that may depend on $ \mathcal{X}, \mathbbm{1}, \phi, \psi, p_{\bm{w}}, p_b$ but does not depend on $f$, $\sigma_{\bm{w}}, \sigma_b, \delta$.
\end{theorem}

\begin{proof}
	This result follows from a modification of the proof of \Cref{thm: disc_mod_rp}.
	Recall from \eqref{eq: errors to be bounded} there are three terms, $(*)$, $(**)$, and $(***)$, to be bounded.
	In what follows we indicate how the arguments used to establish \Cref{thm: disc_mod_rp} should be modified.
	
	\vspace{5pt}
	\noindent \textbf{Bounding $(*)$:} 
	Replace \eqref{eq: bound_b_1} in the proof of \Cref{thm: disc_mod_rp} with the following inequality
	\begin{align}
	B_2(f_{\text{clip}}) & = \int_{[-S,S]^d} f_{\text{clip}}(\bm{x}) (1 + \| \bm{x} \|_2)^2 \mathrm{d} \bm{x} \nonumber \\
	& \leq \int_{[-S,S]^d} M_1(f_{\text{clip}}) \cdot (1 + S)^2 \mathrm{d} \bm{x} = (1 + S)^2 (2S)^d M_1(f_{\text{clip}}) < \infty. \nonumber
	\end{align}
	Now apply \Cref{lem: alt_1} in place of \Cref{thm: recon_mod} to obtain
	\begin{align}
	(*) & \leq C_1' \max(M_1(f_{\text{clip}}), B_2(f_{\text{clip}})) \left( \frac{1}{\sigma_{w}} + \frac{\sigma_{\bm{w}}^{d} (\sigma_{\bm{w}} + 1)^2}{\sigma_{b}} \right) \nonumber 
	\end{align}
	for some constant $ C_1' > 0 $ depending only on $\phi, \psi, p_{\bm{w}}, p_b$.
	Let $C_1 = (1 + S)^2 (2S)^d C'$ to see
	\begin{align}
	(*) & \leq C_1 M_1(f_{\text{clip}}) \left( \frac{1}{\sigma_{w}} + \frac{\sigma_{\bm{w}}^{d} (\sigma_{\bm{w}} + 1)^2}{\sigma_{b}} \right) . \nonumber
	\end{align}
	
	\vspace{5pt}
	\noindent \textbf{Bounding $(**)$:} 
	Here we replace the upper bound of (b) in the proof of \Cref{thm: disc_mod_rp}.
	From (\ref{eq: phi_bound}),
	\begin{eqnarray}
	\text{(b)} & = & Z \int_{\mathbb{R}} \int_{\mathbb{R}^{d}} \Big| ( 1 + \| \bm{w} \|_2 ) \phi(\bm{w} \cdot \bm{x} + b) p_{\bm{w}, \sigma}(\bm{w}) p_{b, \sigma}(b) \Big| \mathrm{d} \bm{w} \mathrm{d} b \nonumber \\
	& \le & Z C_{\phi} (1 + \| \bm{x} \|_2) \underbrace{ \int_{\mathbb{R}^d} (1 + \| \bm{w} \|_2)^2 p_{\bm{w}, \sigma}(\bm{w}) \mathrm{d} \bm{w} }_{\text{(c)}} \underbrace{ \int_{\mathbb{R}} (1 + | b |) p_{b, \sigma}(b) \mathrm{d} b }_{\text{(d)}}. \nonumber
	\end{eqnarray}
	Let $ C_{p_{\bm{w}}} := \max\left( 1, 2 \int_{\mathbb{R}^{d}} \| \bm{w} \|_2 p_{\bm{w}}(\bm{w}) \mathrm{d} \bm{w}, \int_{\mathbb{R}^{d}} \| \bm{w} \|_2^2 p_{\bm{w}}(\bm{w}) \mathrm{d} \bm{w} \right) $ and $C_{p_b} := \max(1, \int_{\mathbb{R}} | b | p_{b}(b) \mathrm{d} b)$.
	By the same discussion on the upper bound of (b) in the proof of \Cref{thm: disc_mod_rp},
	\begin{eqnarray}
	\text{(c)} \le C_{p_{\bm{w}}} \left( 1 + \sigma_{\bm{w}} \right)^2, \qquad \text{(d)} \le C_{p_b} (1 + \sigma_{b}) . \nonumber 
	\end{eqnarray}
	Recall $ Z = (2 \pi)^{\frac{1}{2}} \| \widehat{p_{\bm{w}}} \|_{L^1}^{-1} \| \widehat{p_b} \|_{L^1}^{-1} \sigma_{\bm{w}}^{d} \sigma_b $.
	By the upper bound of (c) and (d),
	\begin{eqnarray}
	\text{(b)} \le (1 + \| \bm{x} \|_2) (2 \pi)^{\frac{1}{2}} \| \widehat{p_{\bm{w}}} \|_{L^1}^{-1} \| \widehat{p_b} \|_{L^1}^{-1} C_{p_{\bm{w}}} C_{p_b} C_{\phi} \sigma_b (1 + \sigma_{b}) \sigma_{\bm{w}}^{d} \left( 1 + \sigma_{\bm{w}} \right)^2. \nonumber
	\end{eqnarray}
	Let $C_2' := \sup_{\bm{x} \in [-S,S]^d} | (1 + \| \bm{x} \|_2) | (2 \pi)^{\frac{1}{2}} \| \widehat{p_{\bm{w}}} \|_{L^1}^{-1} \| \widehat{p_b} \|_{L^1}^{-1} C_{p_{\bm{w}}} C_{p_b} C_{\phi}$.
	By plugging the upper bound of (b) in \eqref{eq: I_quadrature_bnd} and setting $ C_2 := C_Q M_1(\psi) C_2'$, we have
	\begin{eqnarray}
	(**) \le C_2 M_1(f_{\text{clip}}) \frac{\sigma_b (\sigma_{b} + 1) \sigma_{\bm{w}}^{d} (\sigma_{\bm{w}} + 1)^2}{D^{\frac{1}{d}}}. \nonumber
	\end{eqnarray}
	
	\vspace{5pt}
	\noindent \textbf{Bounding $(***)$:} 
	Finally we apply \Cref{lem: markov_ineq} in place of \Cref{lem: bd_ineq}.
	From \eqref{eq: exp_I_bnd}, $\mathop{\mathbb{E}}_{\{ \bm{w}_i, b_i \}_{i=1}^{N}}\left[ (***) \right]$ is upper bounded by
	\begin{eqnarray}
	\mathop{\mathbb{E}}_{\{ \bm{w}_i, b_i \}_{i=1}^{N}}\left[ (***) \right] & \le & \frac{ C_{\mathcal{X}} (2S)^d M_1(f_{\text{clip}}) \| \psi \|_{L^\infty} L_{\phi} V_{p_{\bm{w}}} \sigma_{\bm{w}} }{\sqrt{N}} \nonumber
	\end{eqnarray}
	Set $C_3 := C_{\mathcal{X}} (2S)^d \| \psi \|_{L^\infty} L_{\phi} V_{p_{\bm{w}}} \sigma_{\bm{w}}$.
	By \Cref{lem: markov_ineq}, with probability at least $1 - \delta$,
	\begin{eqnarray}
	(***) \le C_3 M_1(f_{\text{clip}}) \frac{ \sigma_{b} \sigma_{\bm{w}}^{d+1} }{\delta \sqrt{N}}. \nonumber
	\end{eqnarray}
	
	For all bounds of $(*)$, $(**)$ and $(***)$, recall that $ M(f_{\text{clip}}) \le 3 M_{1}^{*}(\mathbbm{1}) M_{1}^{*}(f) $ from (\ref{eq: Mf bound}).
	Then combining $(*)$, $(**)$ and $(***)$ and setting $C := 3 M_1^*(\mathbbm{1}) (C_1 + C_2 + C_3)$ completes the proof.
\end{proof}


\subsection{Experiments}
\label{sec: experiment_apdx}

This section expands on the experiments that were reported in the main text.
\Cref{app: experimental settings} reports the cubature rules that were used.
\Cref{app: MMD_error} discusses how to measure the similarity between BNNs and GPs by the maximum mean discrepancy (MMD) as an alternative error criteria to the MRMSE and shows the MMD of the same experiment as \Cref{sec: exp_01}.
\Cref{sec: comp_exp} explores the effect of using alternative settings in the ridgelet prior, compared to those used to produce the figures in the main text.

In some instances of the experiments we report, poor numerical conditioning was encountered.
The results that we present employed a crude form of numerical regularisation in order that such issues -- which arise from the posterior approximation approach used and are not intrinsic to the ridgelet prior itself -- were obviated.
Specifically, we employed the Moore--Penrose pseudo-inverse whenever the action of an inverse matrix was required, and we employed Tikhonov regularisation (with the extent of the regularisation manually selected) whenever a matrix square root was required.

\subsubsection{Experimental Setting} \label{app: experimental settings}

Here, for completeness, we report the cubature rules $\{(u_i,\bm{x}_i)\}_{i=1}^D$ and the bandwidth parameters $\sigma_{\bm{w}}$, $\sigma_b$ that were used in our experiments.
The sensitivity of reported results to these choices is investigated in \Cref{sec: comp_exp}.

Recall that $\{(u_i,\bm{x}_i)\}_{i=1}^D$ is a cubature rule on $[-S,S]^d$ and that, in all experiments, we aim for accurate approximation on $\mathcal{X} = (-5,5)^d$.
For $S > 5$ a mollifier $\mathbbm{1}$ is required and in this case \eqref{eq: mollifier} we used.
The settings that we used for the results in the main text were as follows.
As aforementioned in \Cref{sec:experiments}, the functions $\phi$ and $\psi$ were set as \Cref{ex: ridgelet_func} and the densities $p_{\bm{w}}$ and $p_{b}$ were set as standard Gaussians for all of the results in \Cref{tbl: quad_list}.

\begin{table}[h!]
	\footnotesize
	\begin{center}
		\begin{tabular}{|l|l|l|}
			\hline
			\textbf{Experiment} & \textbf{Cubature points, weights, total number} & \textbf{Bandwidth of $p_{\bm{w}}, p_{b}$} \\
			\hline
			\Cref{sec: exp_01} & $x_j$: grid point on $[-6, 6]$, $u_j = 12 D^{-1} \mathbbm{1}(x_j)$, $D = 200$ & $\sigma_{\bm{w}} = 3$, $\sigma_{b} = 12$ \\
			\Cref{sec: exp_02} CO$_2$ & $x_j$: grid point on $[-5, 5]$, $u_j = 10 D^{-1}$, $D = 200$ & $\sigma_{\bm{w}} = 3$, $\sigma_{b} = 12$ \\
			\Cref{sec: exp_02} Airline & $x_j$: grid point on $[-5, 5]$, $u_j = 10 D^{-1}$, $D = 200$ & $\sigma_{\bm{w}} = 3$, $\sigma_{b} = 12$ \\
			\Cref{sec: exp_02} In-painting & $\bm{x}_j$: grid point on $[-5, 5]^2$, $u_j = 10^2 D^{-1}$, $D = 30^2$ & $\sigma_{\bm{w}} = 2$, $\sigma_{b} = 12$ \\
			\Cref{sec: exp_03} Deep BNN & $x_j$: grid point on $[-6, 6]$, $u_j = 12 D^{-1} \mathbbm{1}(x_j)$, $D = 50$ & $\sigma_{\bm{w}} = 3$, $\sigma_{b} = 12$ \\
			\hline
		\end{tabular}
	\end{center}
	\caption{Cubature rules and bandwidth parameters used for each experiment reported in the main text.}
	\label{tbl: quad_list}
\end{table}

\subsubsection{Alternative Error Measurement by MMD} \label{app: MMD_error}

In the main text we adopted the MRMSE to measure the similarity between a BNN and the target GP.
However, authors such as \cite{Matthews2018} considered instead a two sample test based on MMD.
The purpose of this appendix is to present complementary results to those in the main text, based instead on MMD.
The MMD is a distance between two probability measures $\mathbb{P}$ and $\mathbb{Q}$, defined as
\begin{align*}
\operatorname{MMD}^2(\mathbb{P}, \mathbb{Q}) := \sup_{\|h\|_{\mathcal{H}} \leq 1} \big| \mathbb{E}_{Y \sim \mathbb{P}}[ f(Y) ] - \mathbb{E}_{Y \sim \mathbb{Q}}[ f(Y) ] \big|
\end{align*}
where $\mathcal{H}$ is a reproducing kernel Hilbert space uniquely associated with a kernel $K$; see e.g. \cite{Muandet2017}.
The MMD has the closed form 
\begin{align*}
\operatorname{MMD}^2(\mathbb{P}, \mathbb{Q}) = \mathbb{E}_{Y, Y' \sim \mathbb{P}}[ K(Y, Y') ] -2 \mathbb{E}_{Y \sim \mathbb{P}, Y' \sim \mathbb{Q}}[ K(Y, Y') ] + \mathbb{E}_{Y, Y' \sim \mathbb{Q}}[ K(Y, Y') ] ,
\end{align*}
which can be approximated using samples from $\mathbb{P}$ and $\mathbb{Q}$.
To this end, we consider the BNN and GP's prior predictives on fixed inputs $\{ \bm{x}_i \}_{i=1}^{M}$ that are $M$-dimensional probability distributions $\mathbb{P}$ and $\mathbb{Q}$ and employ the squared exponential kernel
\begin{align*}
K(\bm{y}, \bm{y}') := \exp\left(- \frac{1}{\alpha^2} \| \bm{y} - \bm{y}' \|_2^2 \right), \qquad \bm{y}, \bm{y}' \in \mathbb{R}^M
\end{align*}
with $1 / \alpha^2 = 0.001$.
For the experiment that we report here, we set $M = 50$ where $\{ x_i \}_{i=1}^{50}$ are a regular grid over $(-5, 5)$, and we use $1,000$ samples from $\mathbb{P}$ and $\mathbb{Q}$ to approximate $\operatorname{MMD}^2$, where $\mathbb{P}$ and $\mathbb{Q}$ are the BNN and GP's prior predictives on $\{ x_i \}_{i=1}^{50}$.

The same experiment as in \Cref{sec: exp_01} was assessed using MMD, with results presented in \Cref{fig: apx_mmd_error}.
It can be observed that MMD decreases when $N$ increases, likewise the MRMSE in \Cref{fig:exp_01}.
The rate of decrease appears somewhat slower than observed in \cite{Matthews2018}, but this is to be expected since we are attempting to approximate an arbitrary GP whereas \cite{Matthews2018} consider approximation of the GP defined as the natural infinite-width limit of the BNN.

\begin{figure}[h!]
	\centering
	\includegraphics[width=0.35\textwidth]{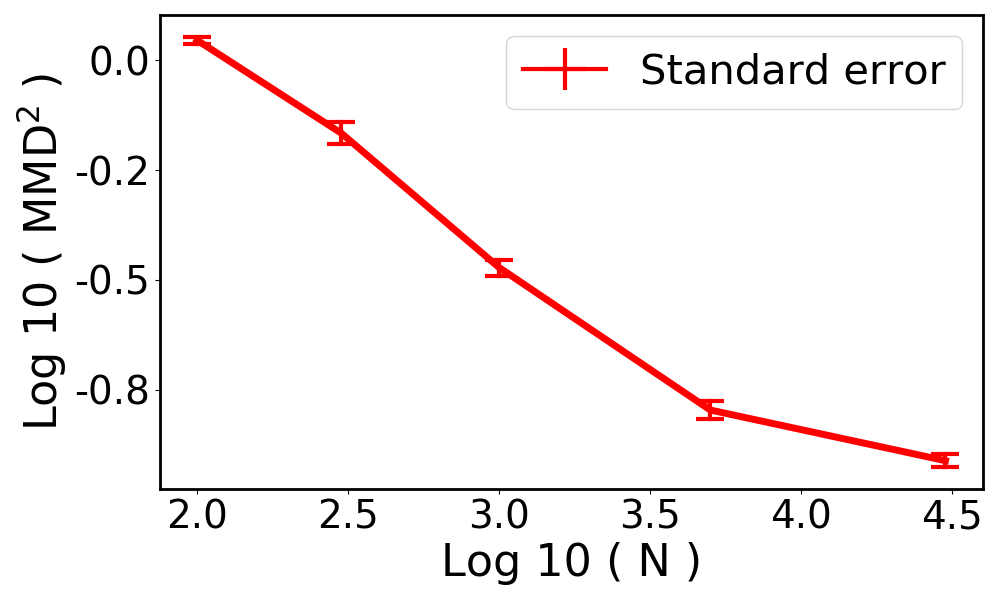}
	\caption{MMD computed from 1000 random samples of GP and BNN outputs at 50 inputs $\{ \bm{x}_i \}_{i=1}^{50}$. The standard error is computed by 10 independent experiments.}
	\label{fig: apx_mmd_error}
\end{figure}

\subsubsection{Comparison of Prior Predictive with Different Settings} \label{sec: comp_exp}

In this final section we investigate the effect of varying the settings of the ridgelet prior.
The default settings in \Cref{ex: ridgelet_func} were taken as a starting point and were then systematically varied.
Initially we consider a squared exponential covariance model for the target GP.
Specifically, we considered:
\begin{enumerate}
	\item Different choices of $\sigma_{\bm{w}}$ and $\sigma_{b}$: $(\sigma_{\bm{w}}, \sigma_{b}) = (1, 2), (2, 6), (3, 12), (4, 20), (5, 30)$
	\item Dynamic setting of $\sigma_{\bm{w}}$ and $\sigma_{b}$: $\sigma_{\bm{w}}$ and $\sigma_{b}$ varies depending on $N$
	\item Different choices of activation function: Gaussian, Tanh, ReLU
	\item Different choices of GP covariance model: squared exponential, rational quadratic, periodic
\end{enumerate}
The findings in each case are summarised next.

\vspace{5pt}
\noindent \textbf{Different choices of $\sigma_{\bm{w}}$ and $\sigma_{b}$:} 
\Cref{fig: comp_exp_01_a} displays the MRMSE and BNN covariance function for each of the choices $(\sigma_{\bm{w}}, \sigma_{b}) = (1, 2), (2, 6), (4, 20), (5, 30)$. Note the choice $(\sigma_{\bm{w}}, \sigma_{b}) = (3, 12)$ is displayed in \Cref{fig:exp_01}.
It can be observed that the BNN covariance function for larger $(\sigma_{\bm{w}}, \sigma_{b})$ has a qualitatively correct shape but is larger overall compared to the GP target when $N$ is small.
On the other hand, the BNN covariance function for smaller $(\sigma_{\bm{w}}, \sigma_{b})$ takes values that are closer to that of the GP, but is visually flatter than the GP and the approximation does not improve as $N$ is increased.
These observation indicates that it may be advantageous to change the values of $(\sigma_{\bm{w}}, \sigma_{b})$ in a manner that increases with $N$.
This leads us to the next experiment.

\begin{figure}[h!]
	\centering
	
	\includegraphics[width=0.24\textwidth]{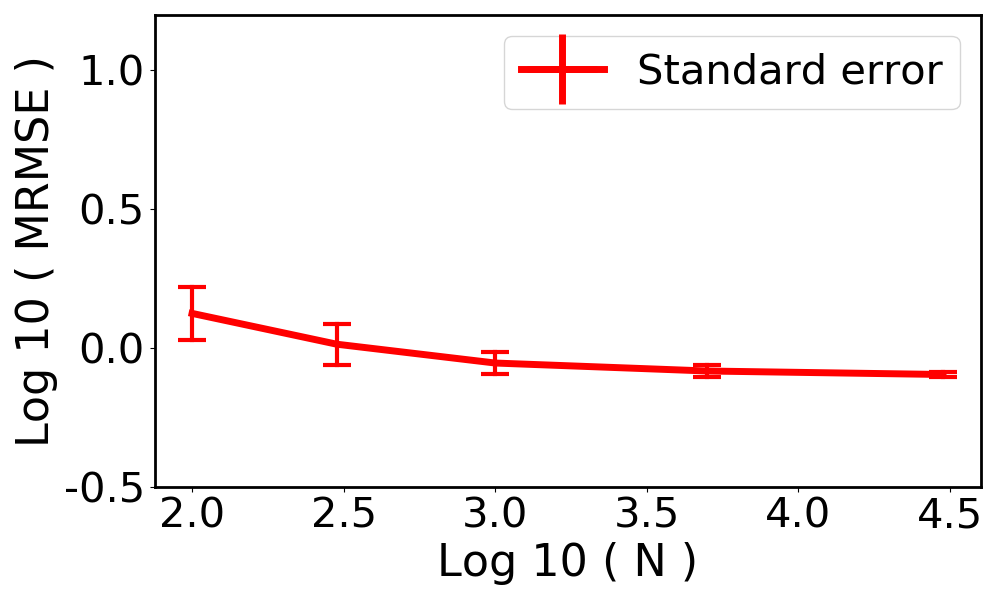}
	\includegraphics[width=0.24\textwidth]{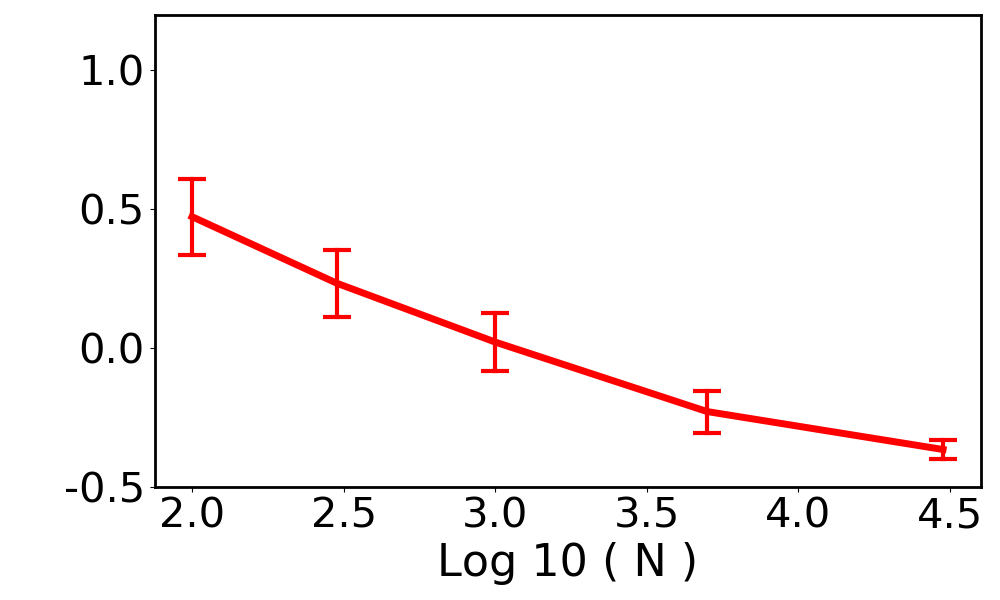}
	\includegraphics[width=0.24\textwidth]{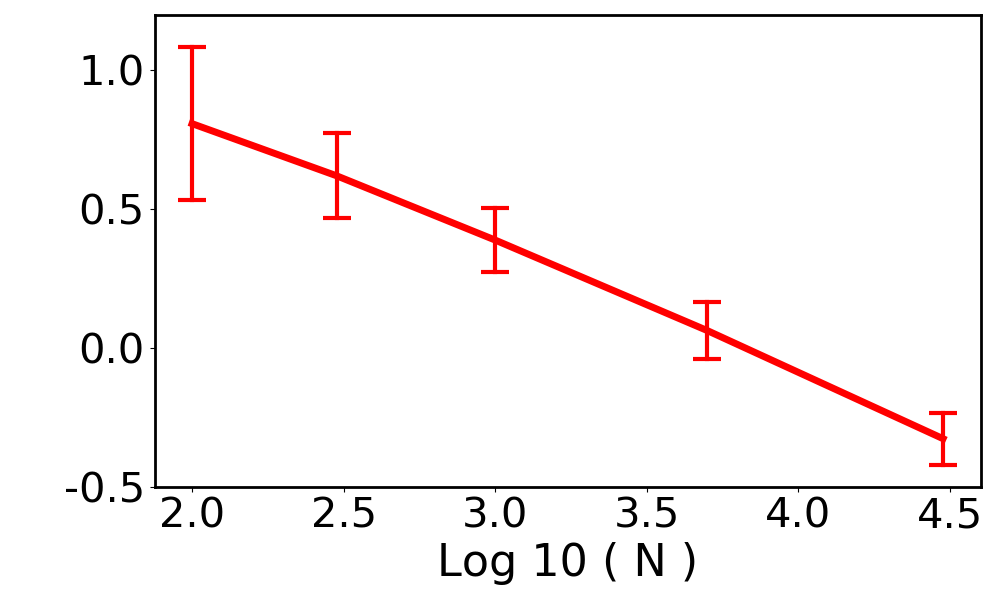}
	\includegraphics[width=0.24\textwidth]{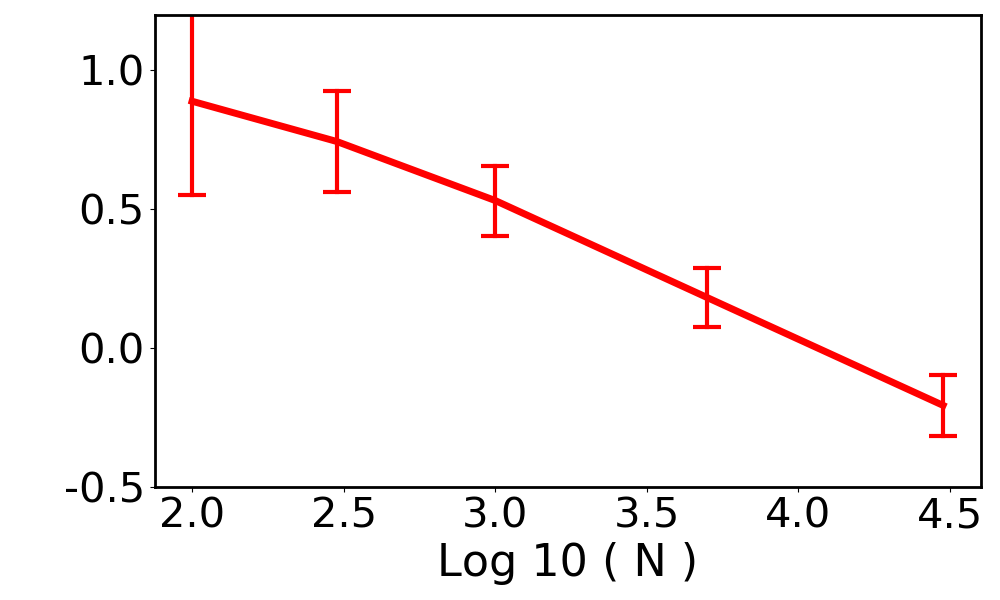}
	
	\subcaptionbox{$(\sigma_{\bm{w}}, \sigma_{b}) = (1,2)$}{\includegraphics[width=0.24\textwidth]{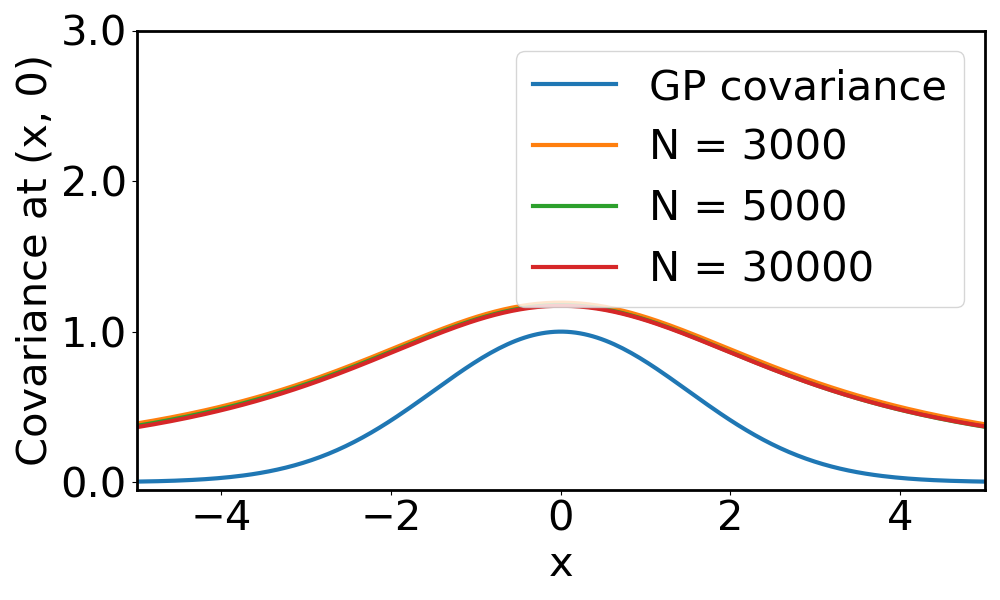}}
	\subcaptionbox{$(\sigma_{\bm{w}}, \sigma_{b}) = (2,6)$}{\includegraphics[width=0.24\textwidth]{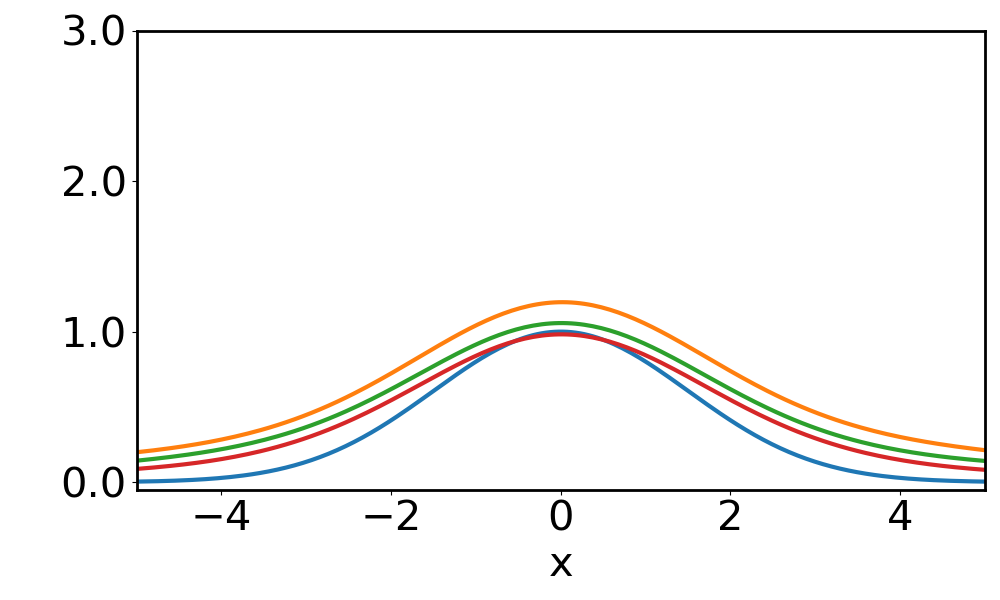}}
	\subcaptionbox{$(\sigma_{\bm{w}}, \sigma_{b}) = (4,20)$}{\includegraphics[width=0.24\textwidth]{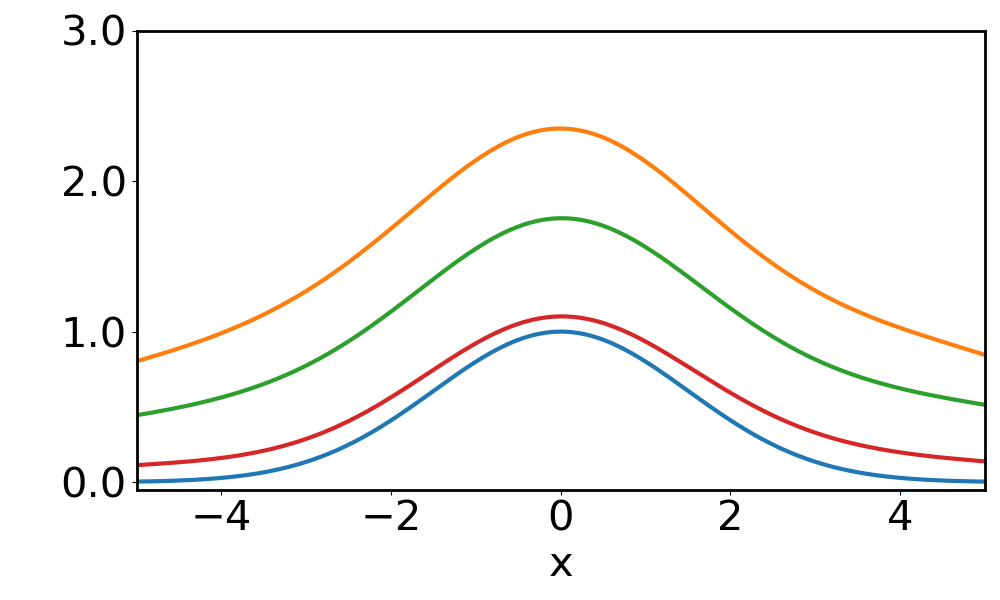}}
	\subcaptionbox{$(\sigma_{\bm{w}}, \sigma_{b}) = (5,30)$}{\includegraphics[width=0.24\textwidth]{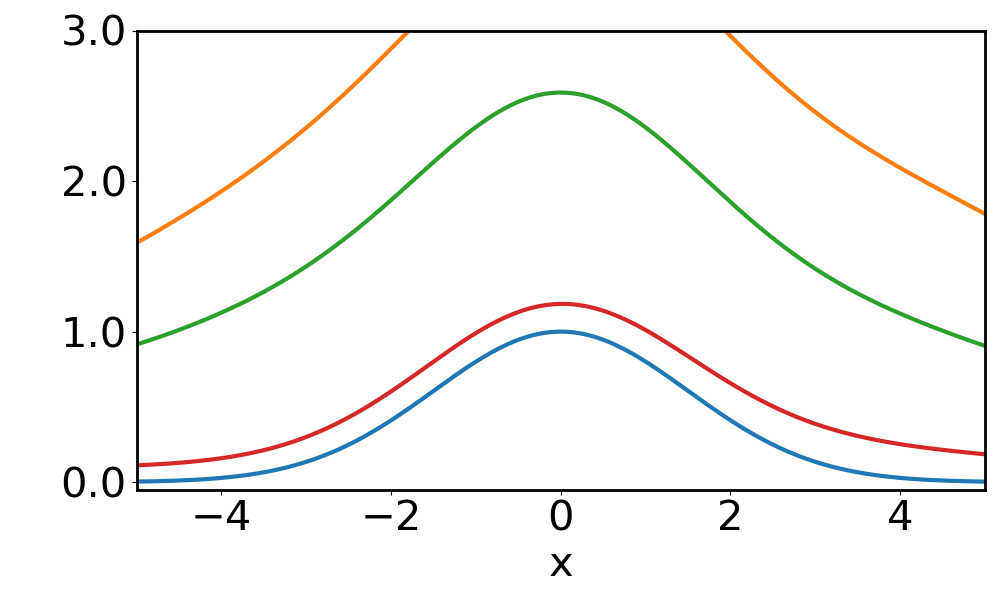}}
	
	\caption{MRMSE and BNN covariance for $(\sigma_{\bm{w}}, \sigma_{b}) = (1, 2), (2, 6), (4, 20), (5, 30)$.}
	\label{fig: comp_exp_01_a}
\end{figure}

\vspace{5pt}
\noindent \textbf{Dynamic setting of $\sigma_{\bm{w}}$ and $\sigma_{b}$:} 
As observed in the previous experiment, it may be advantageous to change $(\sigma_{\bm{w}},\sigma_b)$ in a manner that increases with $N$.
In this experiment, the values of $(\sigma_{\bm{w}}, \sigma_{b}, N)$ were varied as $(\sigma_{\bm{w}}, \sigma_{b}, N) = (1, 2, 300)$, $(2, 6, 1000)$, $(3, 12, 3000)$, $(4, 20, 10000)$, $(5, 30, 30000)$.
\Cref{fig: comp_exp_02_a} shows BNN sample paths and \Cref{fig: comp_exp_02_b} shows MRMSE and BNN covariance.
This dynamic setting of $(\sigma_{\bm{w}}, \sigma_{b}, N)$ appears to constitute a promising compromise between the two extremes of behaviour observed in \Cref{fig: comp_exp_01_a}.

\begin{figure}[h!]
	\centering
	
	\subcaptionbox{$(\sigma_{\bm{w}}, \sigma_{b}, N)=(1, 2, 300)$}{\includegraphics[width=0.325\textwidth]{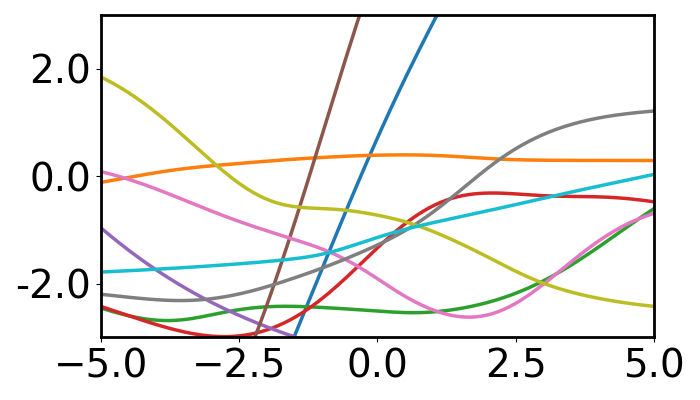}}
	\subcaptionbox{$(\sigma_{\bm{w}}, \sigma_{b}, N)=(2, 6, 1000)$}{\includegraphics[width=0.325\textwidth]{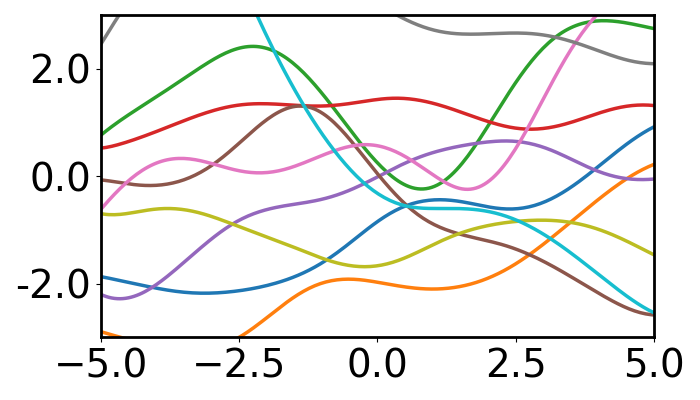}}
	\subcaptionbox{$(\sigma_{\bm{w}}, \sigma_{b}, N)=(3, 12, 3000)$}{\includegraphics[width=0.325\textwidth]{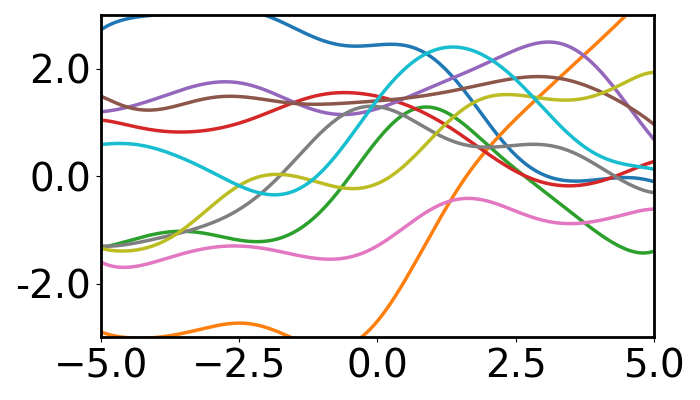}}
	\subcaptionbox{$(\sigma_{\bm{w}}, \sigma_{b}, N)=(4, 20, 10000)$}{\includegraphics[width=0.325\textwidth]{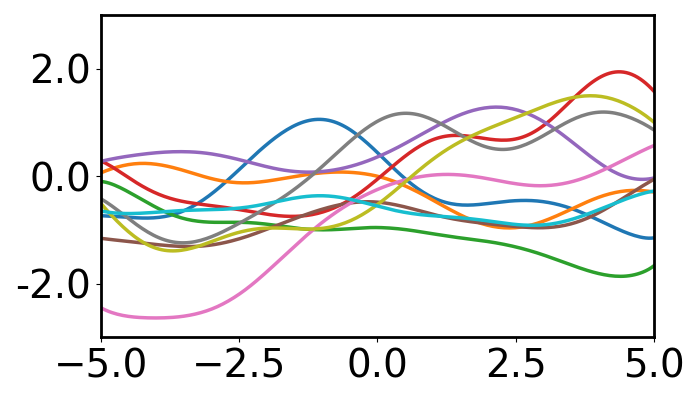}}
	\subcaptionbox{$(\sigma_{\bm{w}}, \sigma_{b}, N)=(5, 30, 30000)$}{\includegraphics[width=0.325\textwidth]{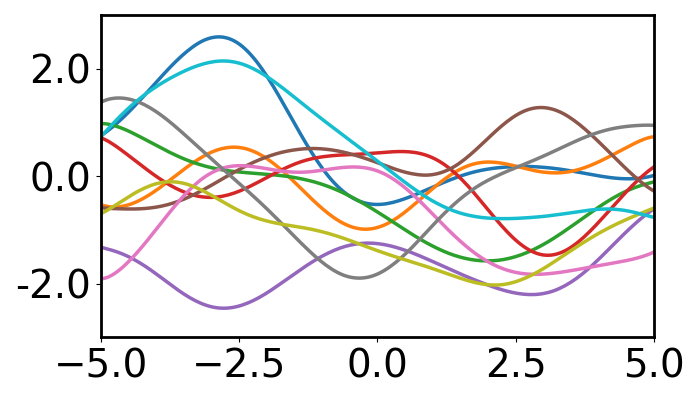}}
	\subcaptionbox{Original GP}{\includegraphics[width=0.325\textwidth]{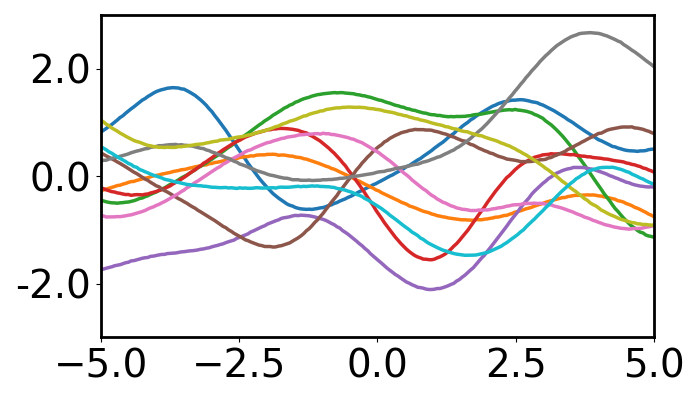}}
	
	\caption{Sample paths of BNN for dynamic setting of $(\sigma_{\bm{w}}, \sigma_{b}, N)$.}
	\label{fig: comp_exp_02_a}
\end{figure}

\begin{figure}[h!]
	\centering
	
	\includegraphics[height=100pt]{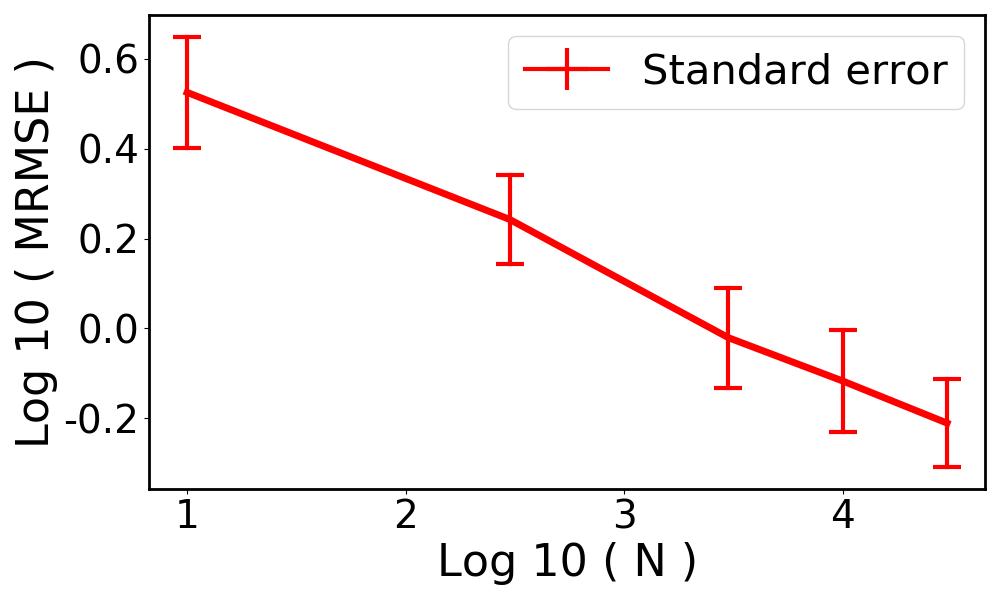}
	\hfill
	\includegraphics[height=100pt]{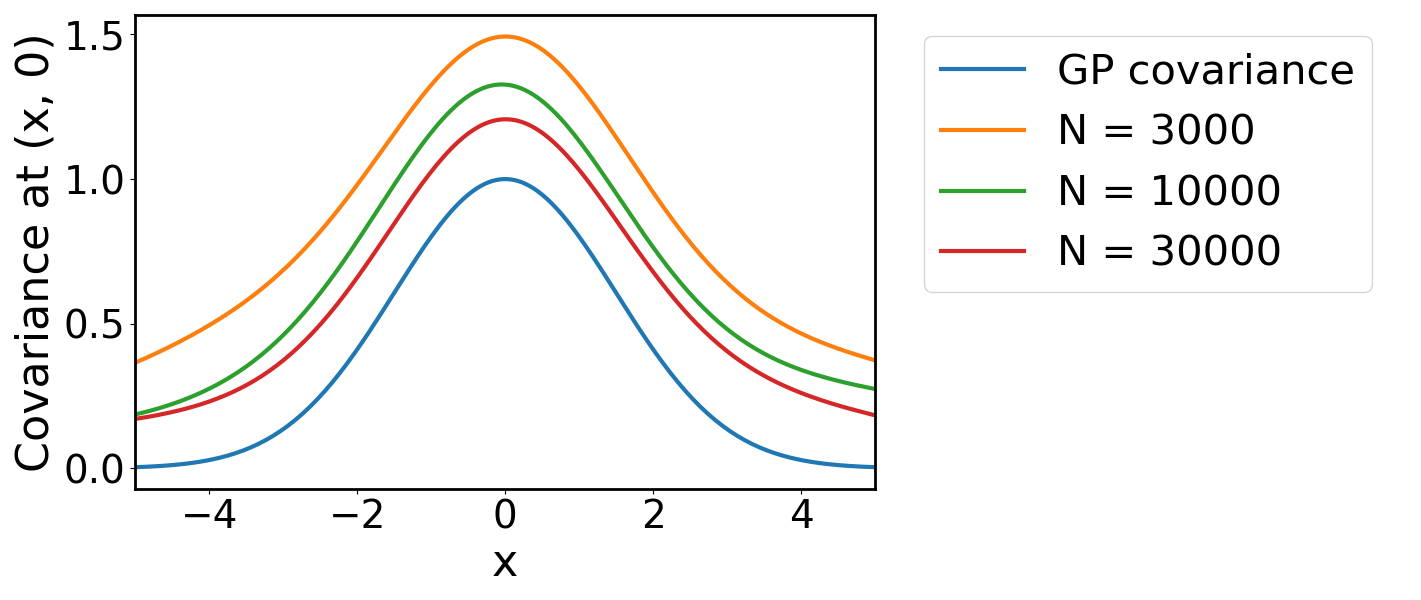}
	
	\caption{MRMSE and BNN covariance for dynamic setting of $(\sigma_{\bm{w}}, \sigma_{b}, N)$.}
	\label{fig: comp_exp_02_b}
\end{figure}

\vspace{5pt}
\noindent \textbf{Different choices of activation function:} 
In this experiment, we fix $(\sigma_{\bm{w}}, \sigma_b) = (3, 12)$ and use 3 different activation functions: Gaussian, hyperbolic tangent, and ReLU.
The settings for the ridgelet prior corresponding to each activation function are given in \Cref{ex: ridgelet_func}.
\Cref{fig: comp_exp_03_a} and \Cref{fig: comp_exp_03_b} indicate that smooth and bounded activation functions, such as the Gaussian activation function, allowed the ridgelet approximation to converge more rapidly to the GP in this experiment. 

\begin{figure}[h!]
	
	\subcaptionbox{Gauss $N=300$}{\includegraphics[width=0.24\textwidth]{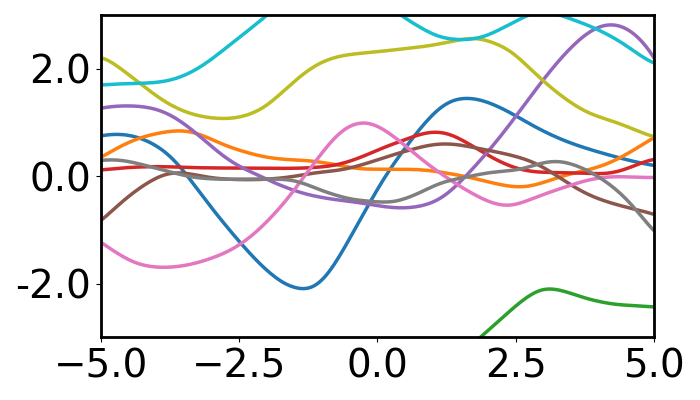}}
	\subcaptionbox{Gauss $N=3000$}{\includegraphics[width=0.24\textwidth]{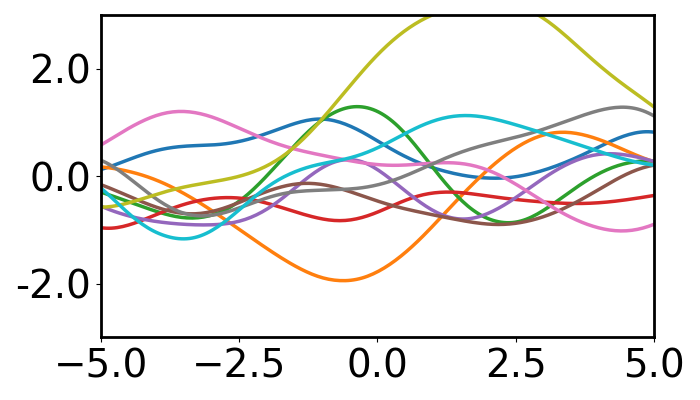}}
	\subcaptionbox{Gauss $N=30000$}{\includegraphics[width=0.24\textwidth]{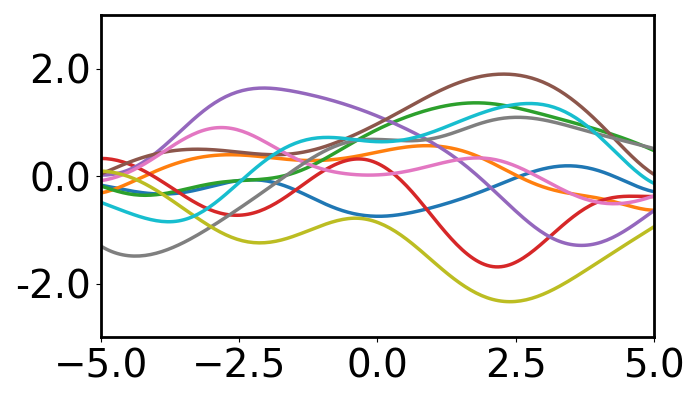}}
	
	\subcaptionbox{Tanh $N=300$}{\includegraphics[width=0.24\textwidth]{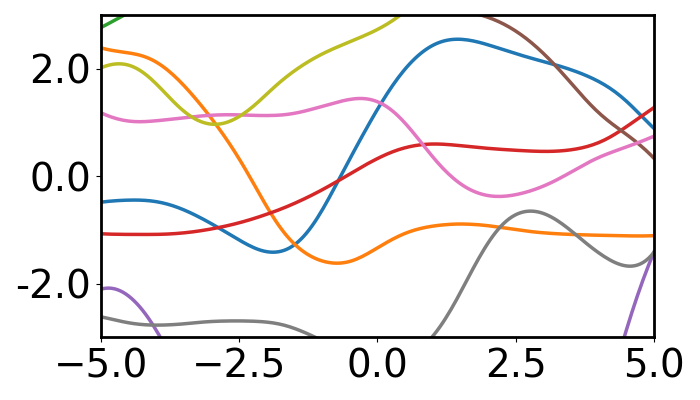}}
	\subcaptionbox{Tanh $N=3000$}{\includegraphics[width=0.24\textwidth]{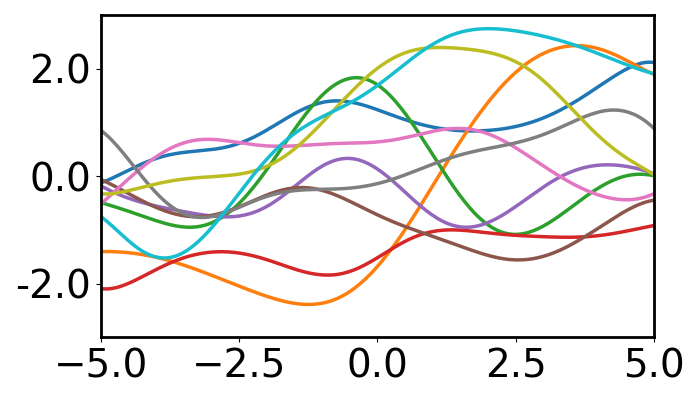}}
	\subcaptionbox{Tanh $N=30000$}{\includegraphics[width=0.24\textwidth]{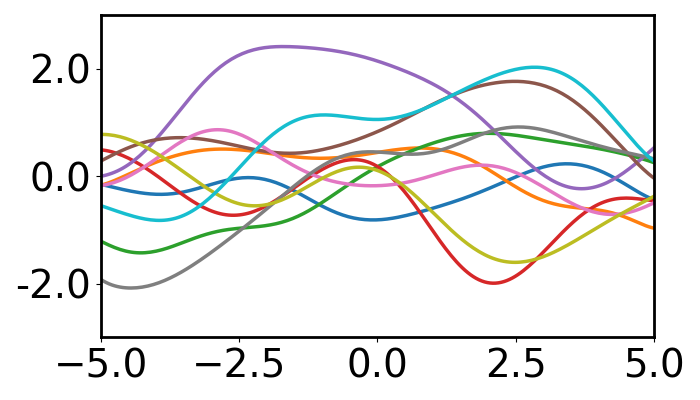}}
	
	\subcaptionbox{ReLU $N=300$}{\includegraphics[width=0.24\textwidth]{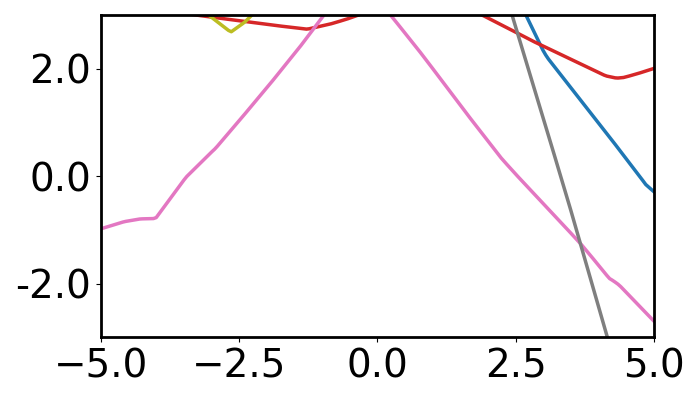}}
	\subcaptionbox{ReLU $N=3000$}{\includegraphics[width=0.24\textwidth]{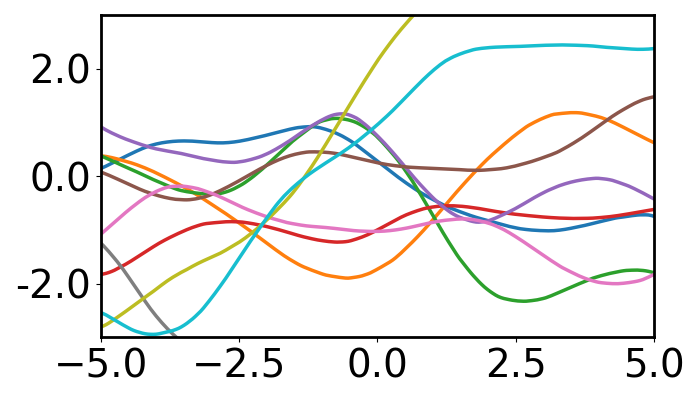}}
	\subcaptionbox{ReLU $N=30000$}{\includegraphics[width=0.24\textwidth]{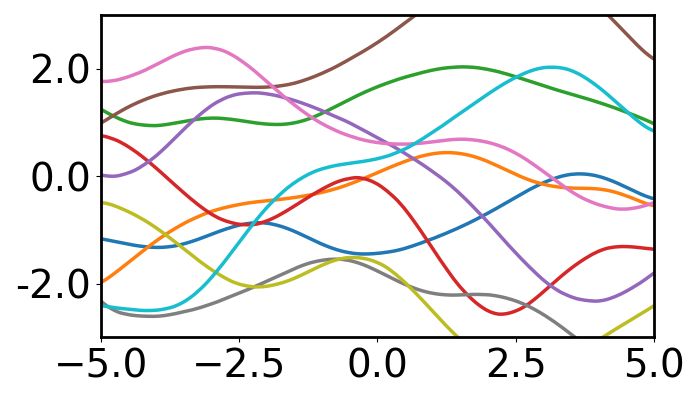}}
	\subcaptionbox{Original GP}{\includegraphics[width=0.24\textwidth]{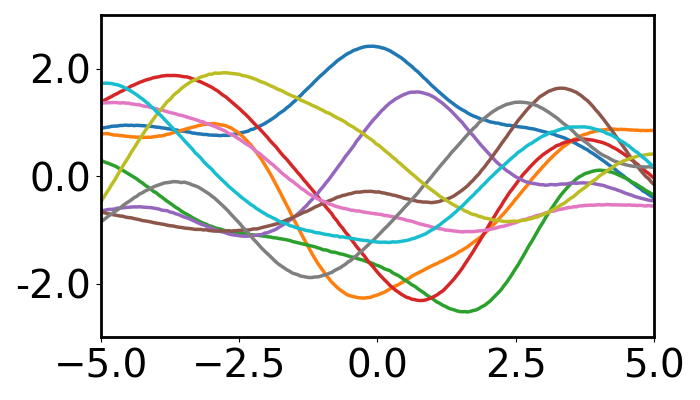}}
	
	\caption{Sample paths of the BNN for different activation functions; Gaussian, hyperbolic tangent, and ReLU.}
	\label{fig: comp_exp_03_a}
\end{figure}

\begin{figure}[h!]
	\centering
	
	\includegraphics[width=0.325\textwidth]{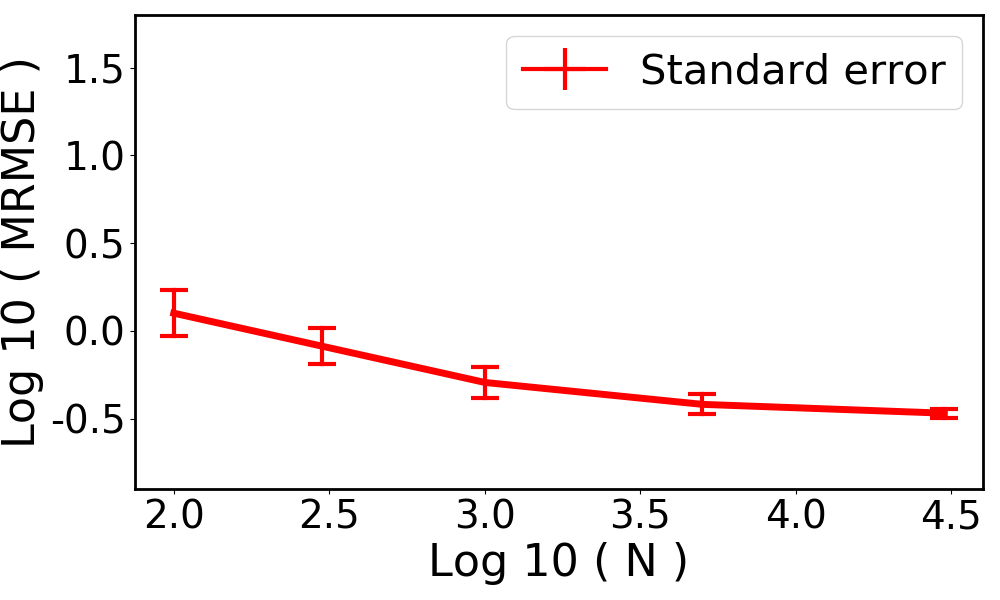}
	\includegraphics[width=0.325\textwidth]{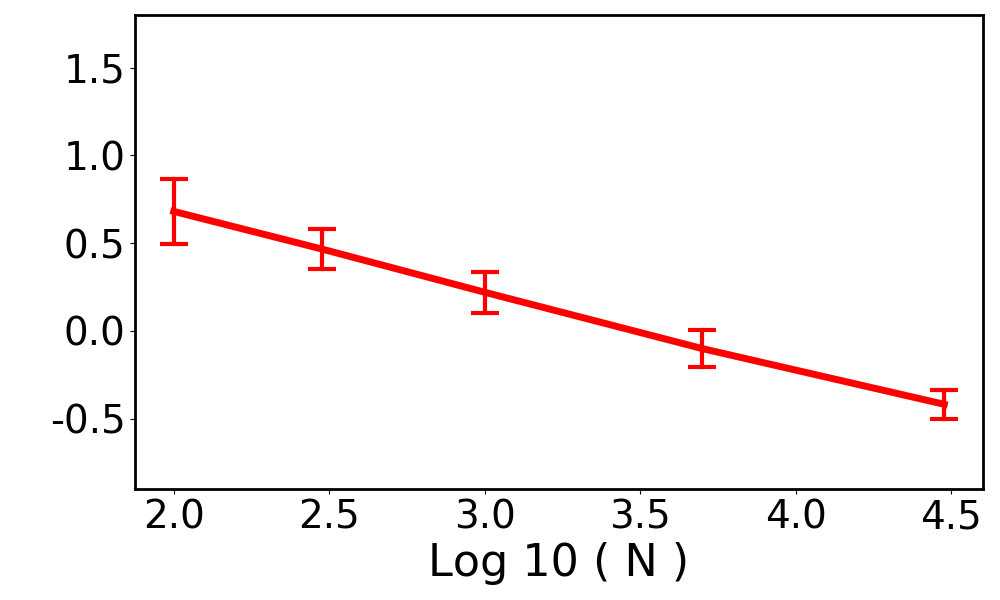}
	\includegraphics[width=0.325\textwidth]{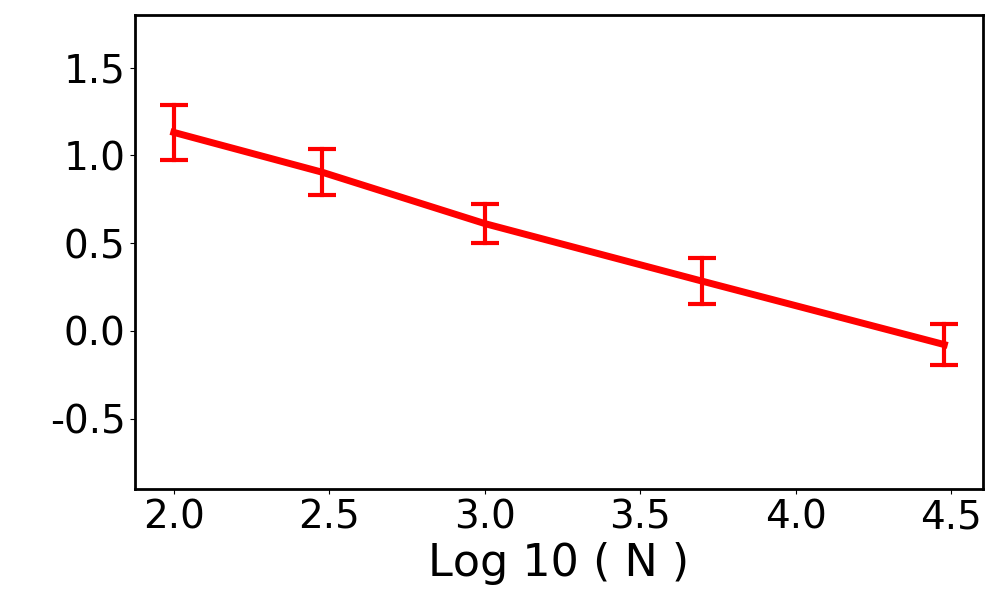}
	
	\subcaptionbox{Gauss}{\includegraphics[width=0.325\textwidth]{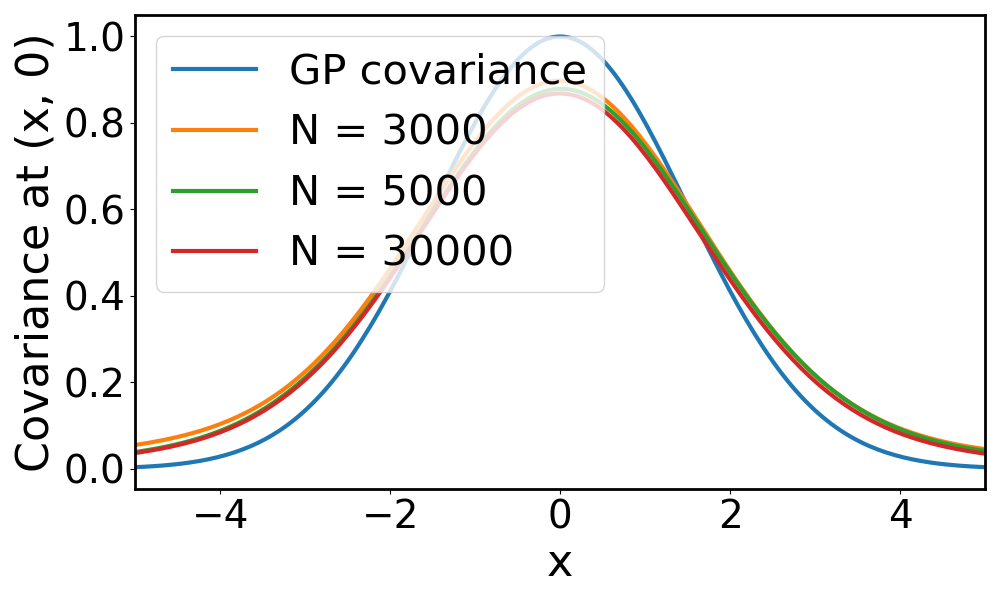}}
	\subcaptionbox{Tanh}{\includegraphics[width=0.325\textwidth]{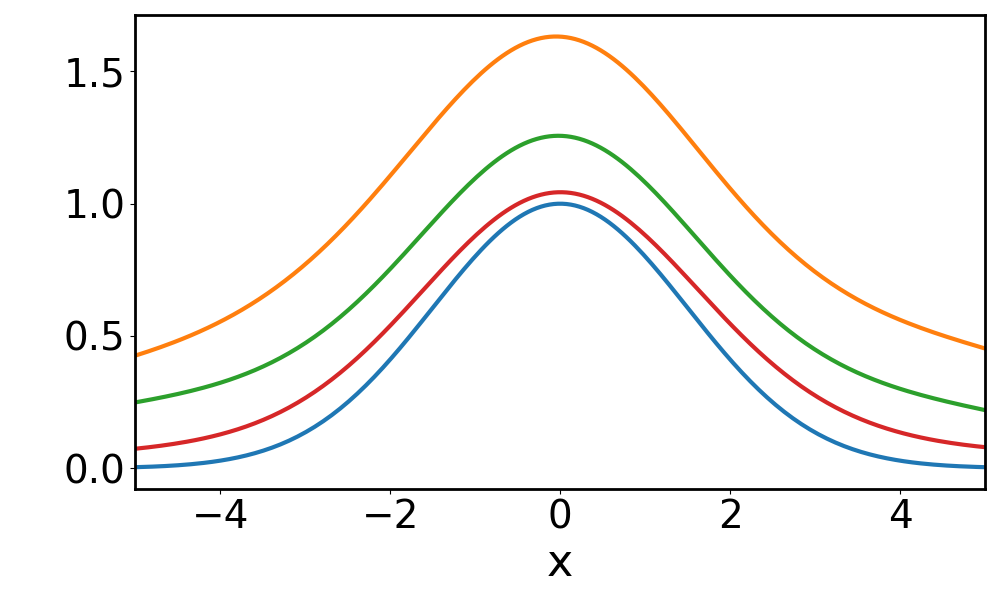}}
	\subcaptionbox{ReLU}{\includegraphics[width=0.325\textwidth]{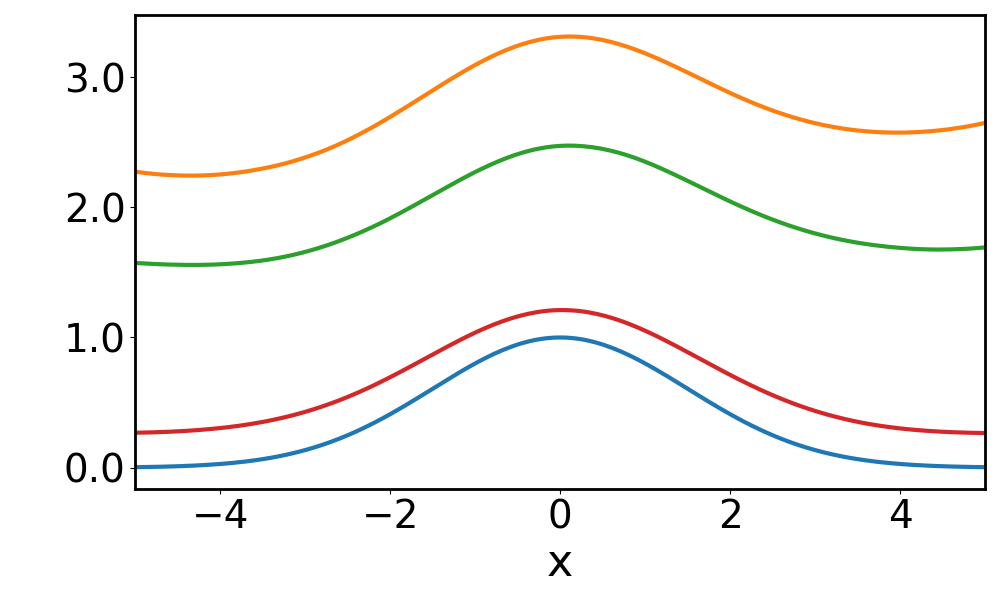}}
	
	\caption{MRMSE and BNN covariance for different activation functions; Gaussian, hyperbolic tangent, and ReLU.}
	\label{fig: comp_exp_03_b}
\end{figure}

\vspace{5pt}
\noindent \textbf{Different choice of GP covariance model:} 
For these experiments we fixed the activation function to the hyperbolic tangent.
Then we considered in turn each of the following covariance models for the target GP: square exponential $k_1$, rational quadratic $k_2$, and periodic $k_3$:
\begin{align*}
k_1(x, x') & := l_1^2 \exp\left( - \frac{1}{2 s_1^2} | x - x' |^2 \right) \\
k_2(x, x') & := l_2^2 \left( 1 + \frac{1}{2 \alpha_2 s_2^2} | x - x' |^2 \right)^{- \alpha_2} \\
k_3(x, x') & := l_3^2 \exp\left( - \frac{2}{s_3^2} \sin\left( \frac{\pi}{p_3^2} | x - x' | \right)^2 \right)
\end{align*}
where $l_1=1.0, s_1=1.5, l_2=1.0, \alpha_2=1.0, s_2=0.75, l_3=1.0, p_3=2.0, s_3=0.75$.
For the choice of $(\sigma_{\bm{w}}, \sigma_{b}, N)$, we used the five combination $(\sigma_{\bm{w}}, \sigma_{b}, N) = (1, 2, 300)$, $(2, 6, 1000)$, $(3, 12, 3000)$, $(4, 20, 10000)$, $(5, 30, 30000)$ examined in the previous experiment.
Such dynamic choice of $(\sigma_{\bm{w}}, \sigma_{b})$ is useful to achieve a better approximation quality when the covariance model is complex.
The sample paths from the BNN with the ridgelet prior are displayed in \Cref{fig: comp_exp_04_a} as a function of $N$, and the associated BNN covariance functions are displayed in \Cref{fig: comp_exp_04_b}.
It is perhaps not surprising that the periodic covariance model, being the most complex, appears to be the most challenging to approximate with a BNN.

\begin{figure}[h!]
	\centering
	
	\subcaptionbox{SE $N=300$}{\includegraphics[width=0.24\textwidth]{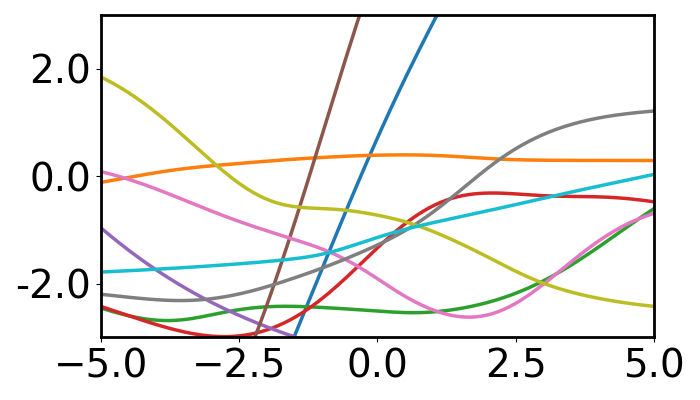}}
	\subcaptionbox{SE $N=3000$}{\includegraphics[width=0.24\textwidth]{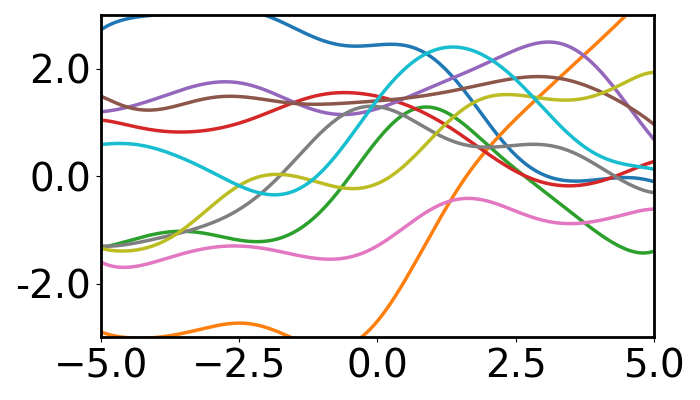}}
	\subcaptionbox{SE $N=30000$}{\includegraphics[width=0.24\textwidth]{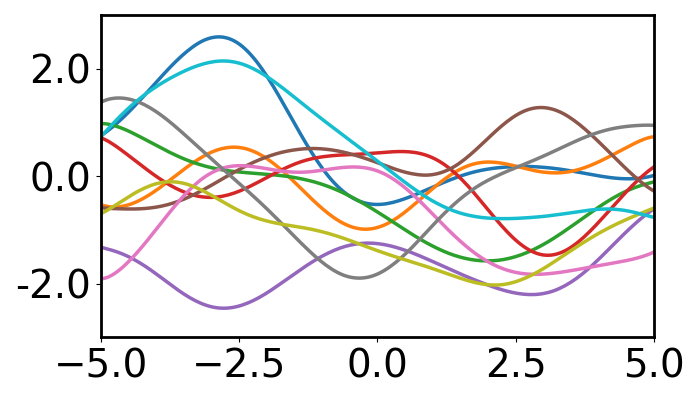}}
	\subcaptionbox{SE Original GP}{\includegraphics[width=0.24\textwidth]{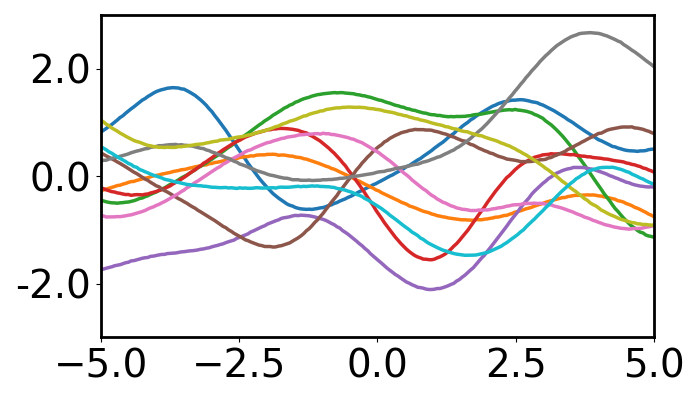}}
	
	\subcaptionbox{RQ $N=300$}{\includegraphics[width=0.24\textwidth]{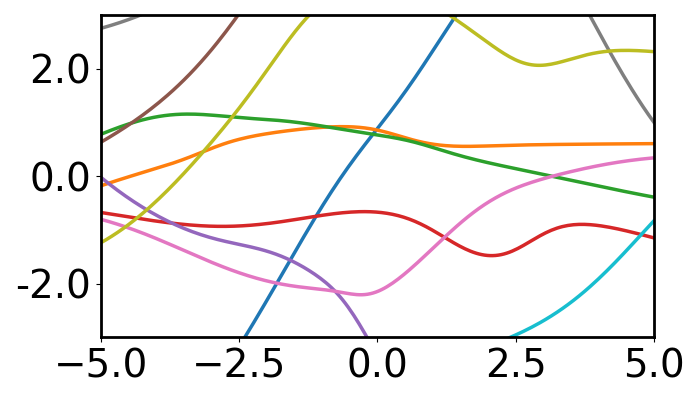}}
	\subcaptionbox{RQ $N=3000$}{\includegraphics[width=0.24\textwidth]{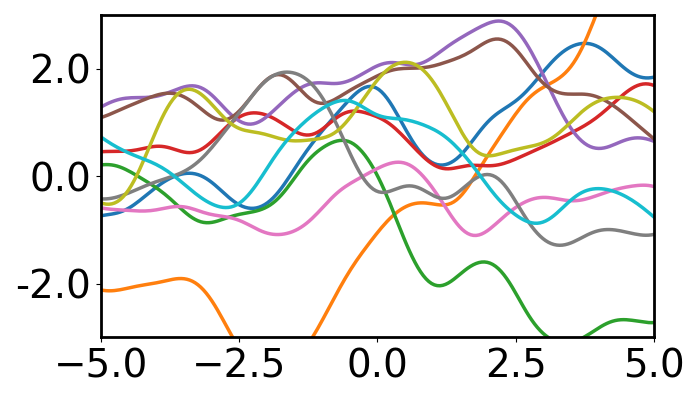}}
	\subcaptionbox{RQ $N=30000$}{\includegraphics[width=0.24\textwidth]{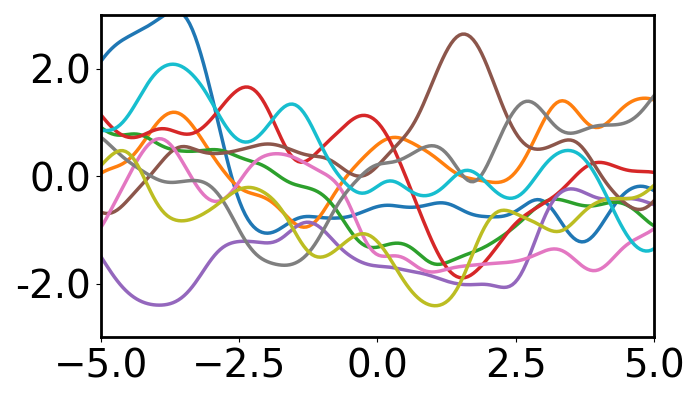}}
	\subcaptionbox{RQ Original GP}{\includegraphics[width=0.24\textwidth]{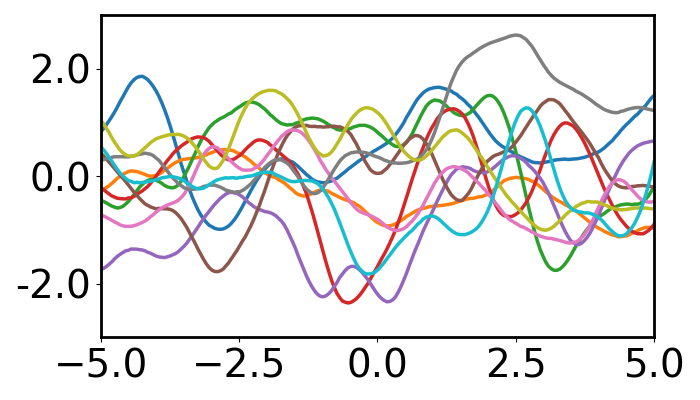}}
	
	\subcaptionbox{Periodic $N=300$}{\includegraphics[width=0.24\textwidth]{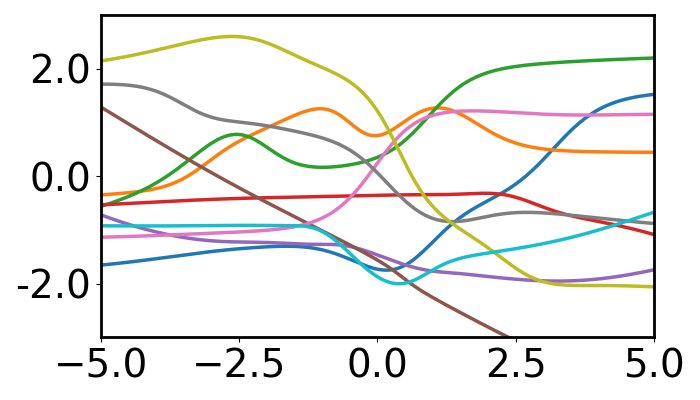}}
	\subcaptionbox{Periodic $N=3000$}{\includegraphics[width=0.24\textwidth]{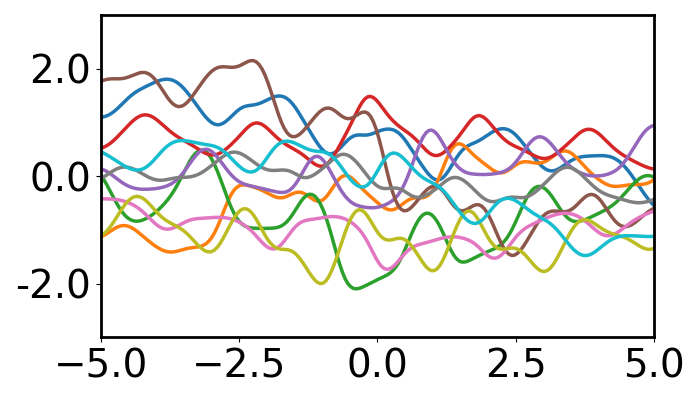}}
	\subcaptionbox{Periodic $N=30000$}{\includegraphics[width=0.24\textwidth]{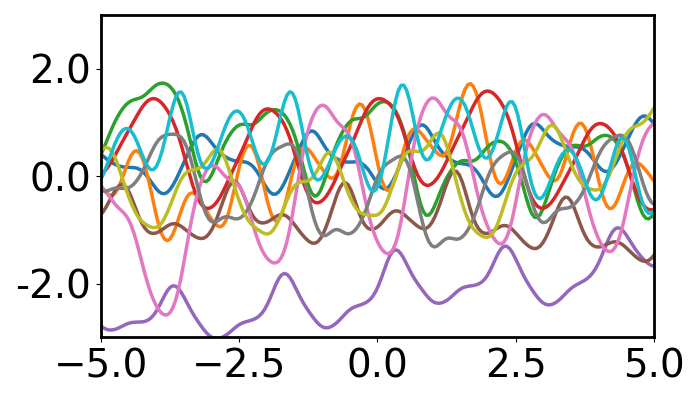}}
	\subcaptionbox{Periodic Original GP}{\includegraphics[width=0.24\textwidth]{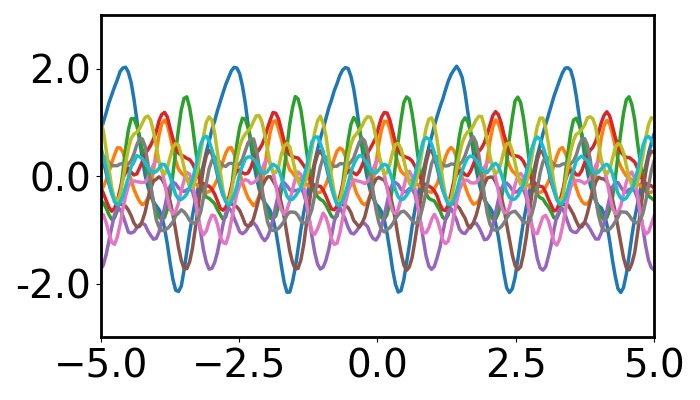}}
	
	\caption{Sample paths of the BNN for different GP covariance models; square exponential, rational quadratic, and periodic.}
	\label{fig: comp_exp_04_a}
\end{figure}

\begin{figure}[h!]
	\centering
	
	\includegraphics[width=0.325\textwidth]{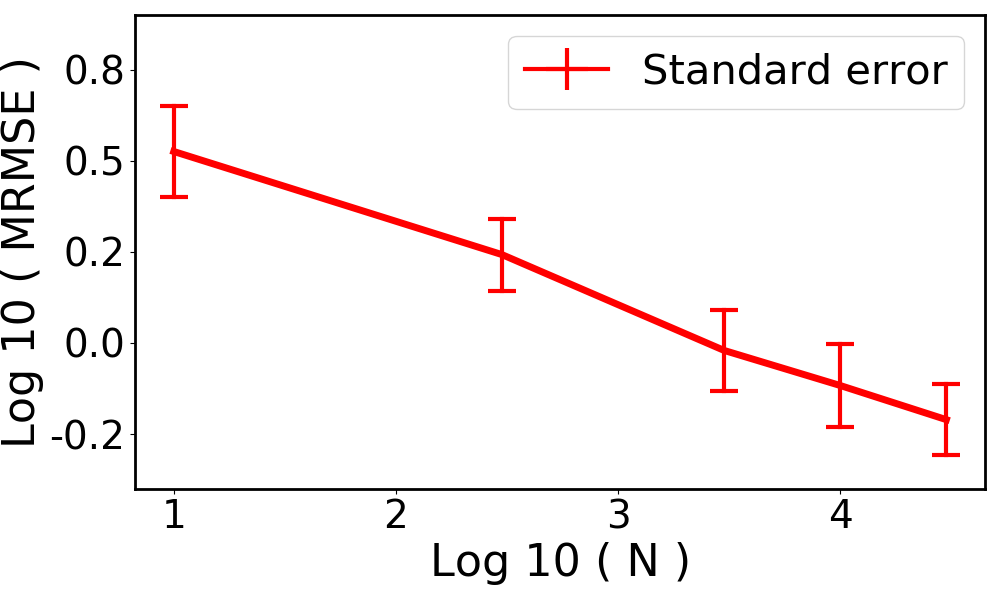}
	\includegraphics[width=0.325\textwidth]{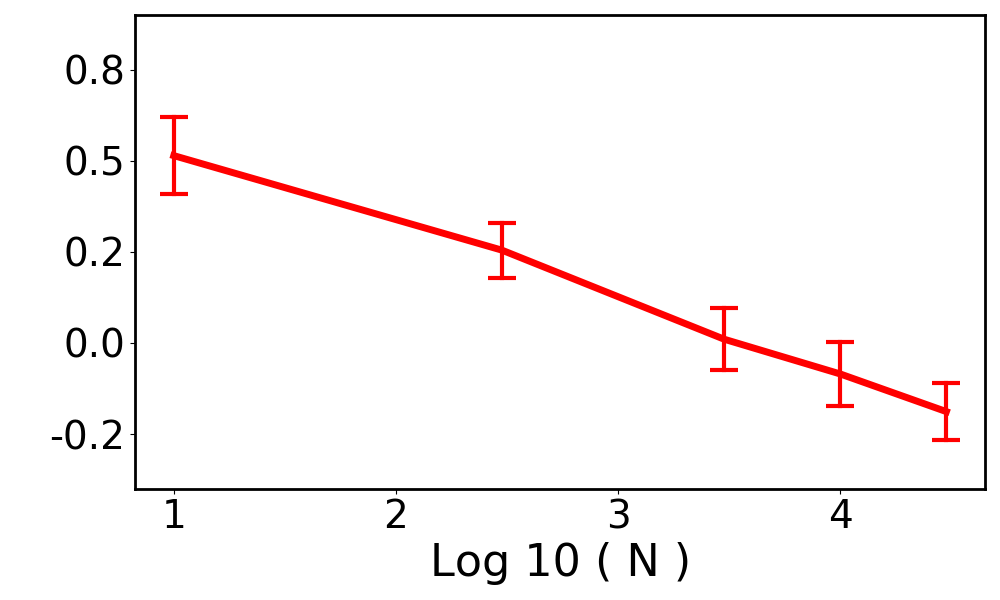}
	\includegraphics[width=0.325\textwidth]{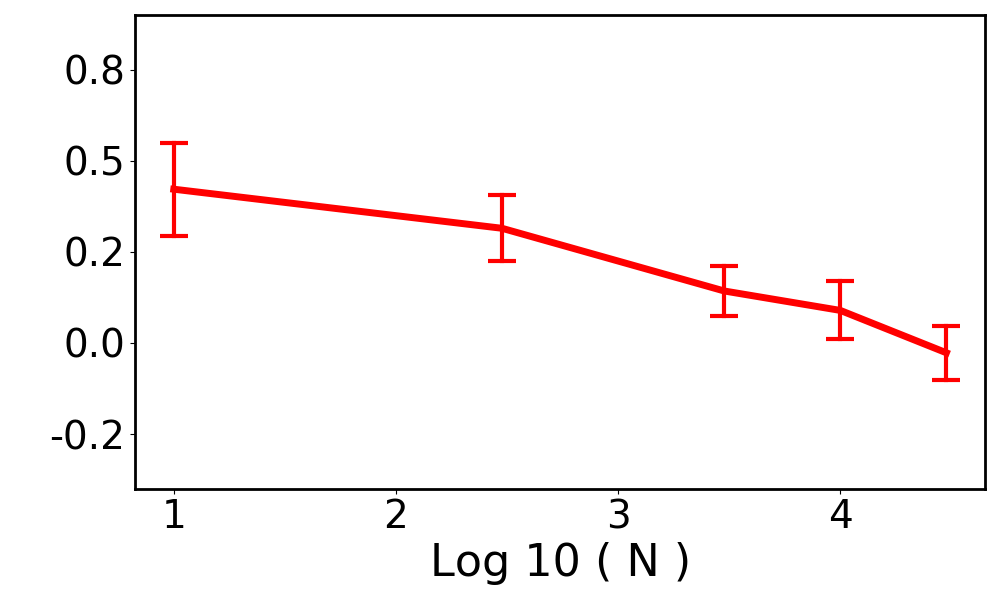}
	
	\subcaptionbox{SE}{\includegraphics[width=0.325\textwidth]{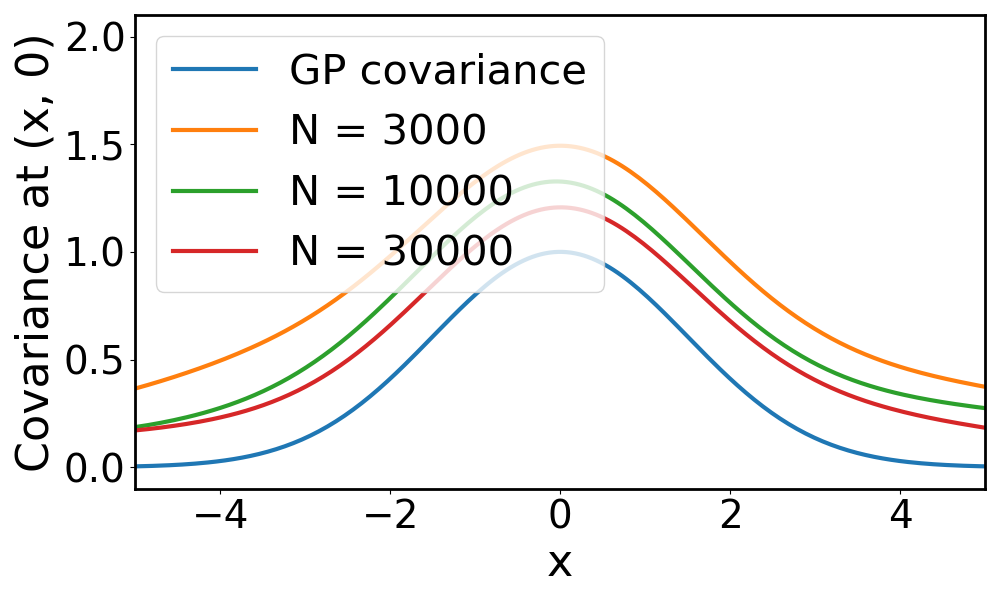}}
	\subcaptionbox{RQ}{\includegraphics[width=0.325\textwidth]{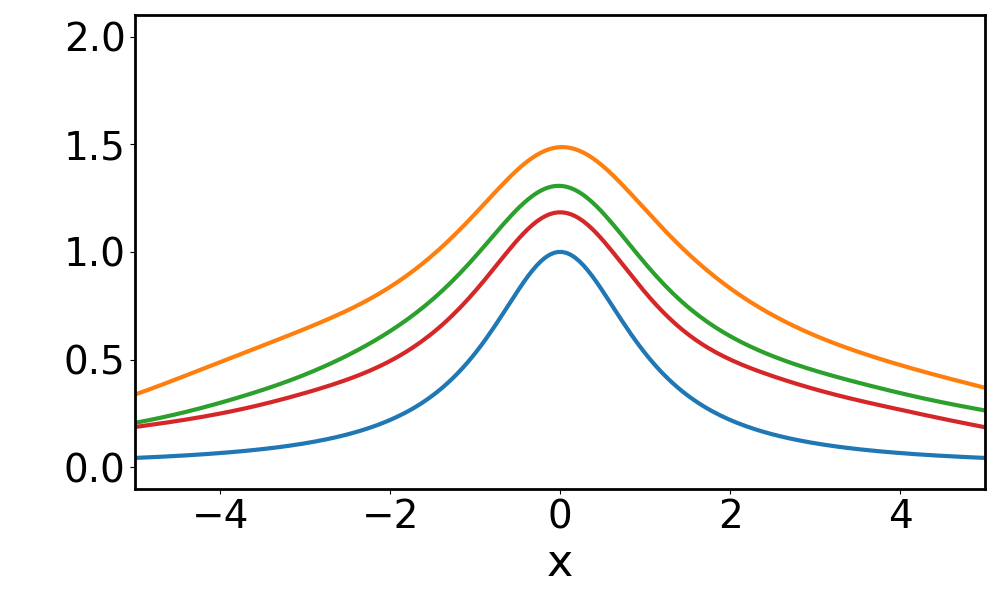}}
	\subcaptionbox{Periodic}{\includegraphics[width=0.325\textwidth]{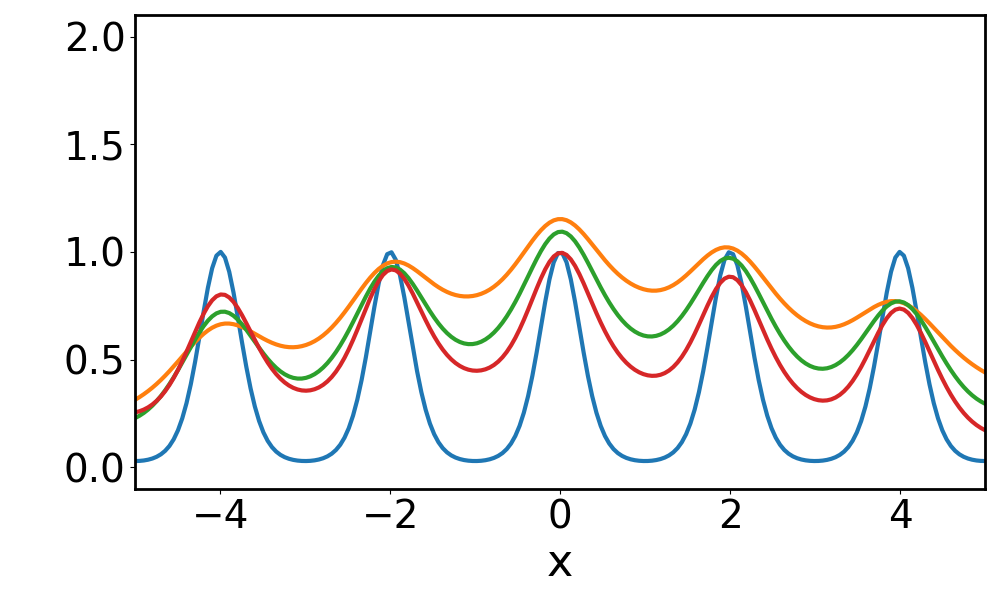}}
	
	\caption{MRMSE and BNN covariance for different GP covariance models; square exponential, rational quadratic, and periodic.}
	\label{fig: comp_exp_04_b}
\end{figure}


\newpage
\bibliography{reference}

\end{document}